\documentclass[conference]{IEEEtran}
\IEEEoverridecommandlockouts
\usepackage{cite}
\usepackage{amsmath,amssymb,amsfonts}
\usepackage{algorithmic}
\usepackage{graphicx}
\usepackage{textcomp}
\usepackage{graphicx}
\usepackage{booktabs} 
\usepackage{times}
\usepackage{epsfig}
\usepackage{graphicx}
\usepackage{amsmath}
\usepackage{amsthm}
\usepackage{amssymb}
\usepackage{subcaption}
\usepackage{arydshln}
\usepackage{caption}
\usepackage[table]{xcolor}
\let\labelindent\relax
\usepackage{enumitem}

\newcommand{\cM}{\mathcal{M}}

\newcommand{\cX}{\mathcal{X}}
\newcommand{\R}{\mathbb{R}}

\newtheorem{thm}{Theorem}[section]
\newtheorem{remark}[thm]{Remark}
\newtheorem{prop}[thm]{Proposition}
\newtheorem{example}[thm]{Example}

\def\BibTeX{{\rm B\kern-.05em{\sc i\kern-.025em b}\kern-.08em
    T\kern-.1667em\lower.7ex\hbox{E}\kern-.125emX}}
\begin{document}

\title{Mapper Based Classifier
\thanks{$^*$ Jacek Cyranka and Alexander Georges contributed equally to this work. 
Most of this work was done when Jacek Cyranka held a postdoctoral position at 
CSE Department, UC San Diego under supervision of prof. Sicun Gao.
JC has been partially supported by NAWA Polish Returns grant PPN/PPO/2018/1/00029.
}
}

\author{\IEEEauthorblockN{Jacek Cyranka$^*$}
\IEEEauthorblockA{\textit{Institute of Informatics} \\
\textit{University of Warsaw}\\
Warszawa, Poland\\
jcyranka@gmail.com}
\and
\IEEEauthorblockN{Alexander Georges$^*$}
\IEEEauthorblockA{\textit{Department of Physics} \\
\textit{University of California, San Diego}\\
San Diego, USA\\
ageorges@ucsd.edu}
\and
\IEEEauthorblockN{David Meyer}
\IEEEauthorblockA{\textit{Department of Mathematics} \\
\textit{University of California, San Diego}\\
San Diego, USA \\
dmeyer@math.ucsd.edu}
}

\maketitle

\begin{abstract}
   Topological data analysis aims to extract topological quantities from data, which tend to focus on the broader global structure of the data rather than
local information.  The Mapper method, specifically, generalizes clustering methods to identify significant global mathematical structures, which are out of reach of many other approaches. We propose a classifier based on applying the Mapper algorithm to data projected onto a latent space. We obtain the latent space by using PCA or autoencoders. Notably, a classifier based on the Mapper method is immune to any gradient based attack, and improves robustness over traditional CNNs (convolutional neural networks). We report theoretical justification and some numerical experiments that confirm our claims.
\end{abstract}

\section{Introduction}
Deep neural networks \cite{deeplearning,deeplearningbook} are well known to be not robust with respect to input image perturbations, which are designed by adding to images perturbations that are typically non-perceptible by humans  \cite{intriguing,advers,advphys}. In this paper we explore opportunities for combining deep learning techniques with a well known \emph{topological data analysis} (TDA) algorithm -- \emph{the Mapper algorithm} \cite{singh2007topological}, which we use to create classifiers with improved robustness. First, the training data is projected onto a latent space. The latent space in the simplest variant is constructed using \emph{PCA} components, and we also use nonlinear projections by utilizing various \emph{autoencoders} \cite{deeplearning,CAE,DAE,VAE}.
Then, a discrete graph representation (Mapper) is assigned to the training data projected onto the \emph{latent space}.  Having this trained graph structure built, any input can be binarized, by assigning the binary vector representing the nodes in the graph to which it belongs (see the algorithm on Fig.~\ref{fig:tdann}). We emphasize that such discretization step makes our algorithm essentially immune to any white-box \emph{gradient based adversarial attack}. The test data is treated by a special mapping procedure that is essentially performing a weighted $k$-nearest neighbor search in the preimage of some portion of the latent space in order to compute a vectorized representation of the testing points. We apply the algorithm we have developed, using methods from topological data analysis and more traditional approaches, to the \emph{MNIST} \cite{lecun-98} and \emph{FashionMNIST} \cite{fashionmnist} datasets as an application of robust computer vision. 

We stress that the general idea of applying a topological method (Mapper algorithm) on top of some latent space method is that the topology is less sensitive to both the perturbation direction and size in the original data space - i.e. no particular direction in the data space is very likely to fool the topological classifier.  Because of this, the topological classifier has some noise invariance built in from the outset.  In the case of neural networks, such a direction can be computed using the gradient of the input - and indeed, many such directions typically exist, many of which are exceedingly small perturbations.  Further, since our topological based classifier is not differentiable, it makes it difficult to efficiently find the input perturbations that are most likely to fool it.

The main ingredient in this algorithm is our implementation of a topological object, called a \emph{Mapper}, which captures global information about the data space.  Intuitively, these objects allow for some variance in the data while still producing the same desired output \cite{carriere2018structure}.  In general, methods in topological data analysis are robust to perturbations in data because the overall global structure remains relatively unchanged, and these methods capture this information.  Owing to this property,  due to the implemented ensemble approach the bias and variance in our predictive model is reduced.  Hence, our algorithm resolves, to an extent, the \emph{bias-variance tradeoff}.  We will note that Mappers have been used to classify error modes for CNNs when applied to MNIST \cite{fibres} and FashionMNIST \cite{fashionmnist} data. 

The original software to produce Mapper outputs is \emph{Python Mapper} \cite{mappersw}.  The code we have constructed for this analysis is a prototype implementation and is built around an R implementation of the Python Mapper software \cite{pearsontdamapper}.  Using Mapper as a means to compare shapes has been done before in \cite{singh2007topological} and \cite{biasotti2000extended}.  Our approach is different in that we utilize traditional machine learning approaches in conjunction with the Mapper algorithm.  The code is available online \cite{codes}.
%
%
\begin{figure}[t]
\begin{center}
   \includegraphics[width=0.98\linewidth]{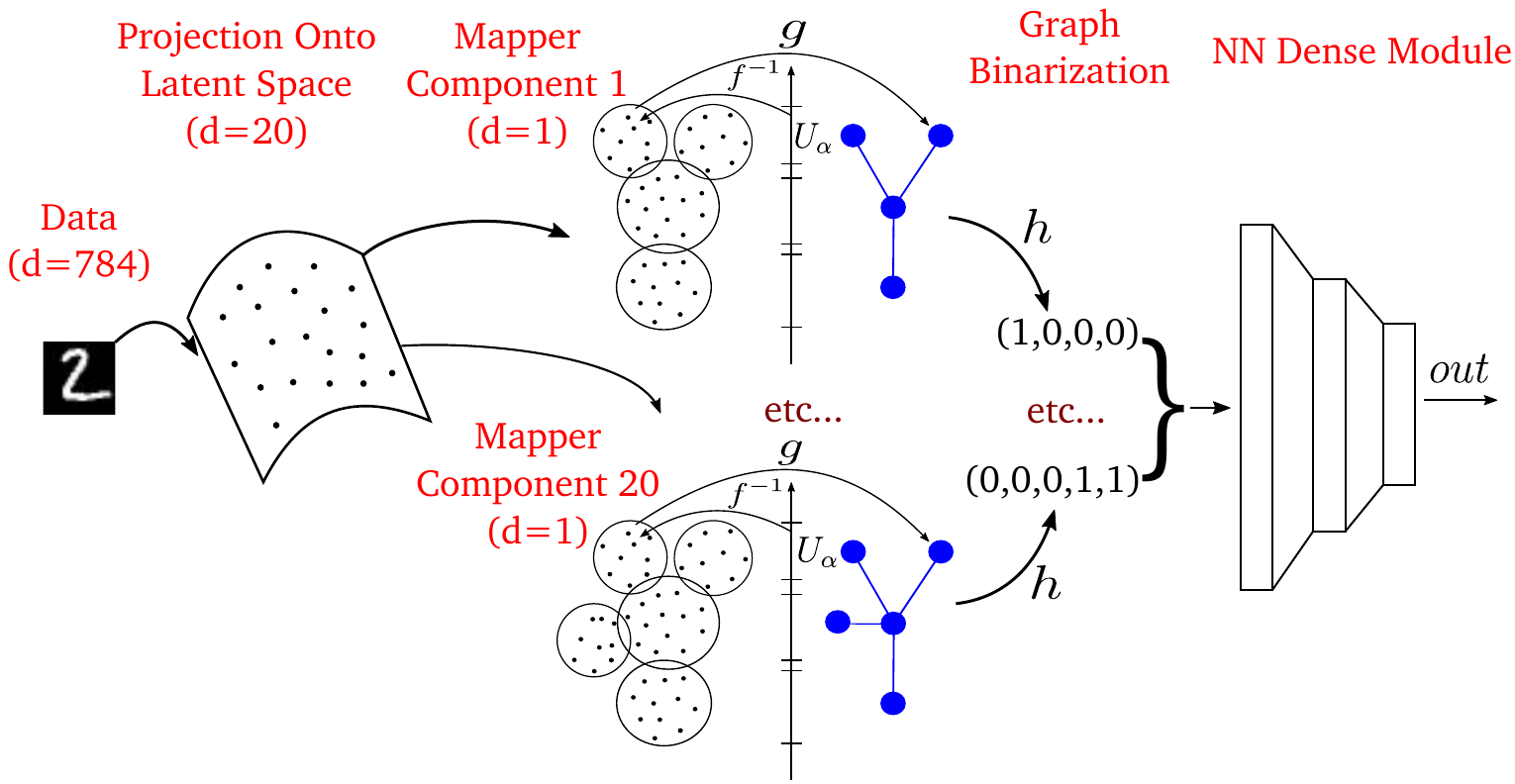}
\end{center}
   \caption{Illustration of our MC method. Refer to algorithm description in the paper. In practice we use not just one Mapper output graph but a whole committee of mapper graphs (see Sec.~\ref{sec:structure}).}
\label{fig:tdann}
\end{figure}

\subsection{What is Mapper?}
\label{secmapper}
The Mapper method \cite{singh2007topological}
 is a discretized analog of \emph{Reeb graphs} \cite{carlsson2009topology, reeb1946points}, which are tools used in \emph{Morse theory} \cite{milnor2016morse}. Both Mapper and Reeb graphs provide topological information pertaining to connectivity of the space.  More precisely, they describe changes in level sets of the filter $f$ (i.e. set of points in $\mathcal{X}$ at the level $l$ is $\{x\in\mathcal{X}\colon f(x) = l\}$) in a space given a function over this space.  Some motivations for using these approaches to understand data consist of: the ability to get a higher-level understanding of the structure of data by determining clustering information (which is based on clusters in $X$ and how various functions behave on $X$), and the low computation cost of producing these topological networks.  They have been used for a number of different applications, including: the discovery of significant clusters in breast cancer gene expression, the classification of player performance in the NBA, voting behavior in the United States Congress \cite{lum2013extracting}, and to study RNA hairpins to identify dominant folding paths \cite{bowman2008structural}.
 
We denote a specific Mapper output graph as $M=M(X, f)$ and the space of Mapper outputs by $\cM$.  Below, we describe the Mapper algorithm and provide specific examples of Mapper output graphs.  

\subsection{Mapper Algorithm:}
\label{secmapper}
\noindent\textbf{INPUT:}\\
* The dataset $X_{\rm train}\subset\R^{N_{\rm in}}$.\\
* The choice of metric for the pairwise distances,\\
* The function $f\colon\cX\to\R$, referred to as the ``lens'' or ``filter function.'' \footnote{The Mapper algorithm does not require a mapping to one dimension, and indeed it is possible to replace this step with $f: \cX \to \mathbb{R}^m$ for arbitrary $m$.}\\
* The number of intervals in the \emph{open cover} of $im(f) \subset \R$ is defined as $n_{\rm int}$,\  the percent overlap of the intervals is defined as the \emph{gain},\ the number of bins in a histogram consisting of distances at which clusters merge in a single-linkage clustering is defined as $n_{\rm bins}$.  Note: for our purposes, an \emph{open cover} is defined as a set of open intervals that cover the entire space upon taking their union \cite{hatcheralgebraic}. We set $n_{\rm bins}=n_{\rm int}=10$ and $gain=0.33$ (more on this in Sec. \ref{secnum}).\\
\noindent\textbf{OUTPUT:}\\
* The Mapper output graph $M\in\cM$ (undirected graph encoding the clustering of data and the intersection structure.\\
\noindent\textbf{begin}
\begin{enumerate}[topsep=0pt,itemsep=-0.5ex,partopsep=1ex,parsep=1ex]
    \item Set $I=\text{im}(f)$.  Choose an open cover $\{ U_{\alpha} \}$ for $I$.
    \item  Set $V_{\alpha}=f^{-1}(U_{\alpha})$.  Then $\{ V_{\alpha} \}$ is an open cover for $X$.
    \item  Refine $\{V_{\alpha}\}$ to $\{V_{\alpha,i_{\alpha}}\}$ where $i_{\alpha}$ indexes the $N_{\alpha}$ components of $X_{\rm train}\cap V_{\alpha}$ defined by connecting points that have distance less than $\epsilon>0$, where $\epsilon$ is dependent on $n_{\rm bins}$.  Let $\tilde N = \sum_{\alpha} N_{\alpha}$.
    \item Let $\{(\alpha,i_{\alpha})\}$ label the $\tilde N$ vertices of a simplicial complex. It is useful to think of data points living in these vertices. 
    \item Connect vertices labeled $(\alpha,i_{\alpha})$ and $(\alpha',i_{\alpha'})$ iff $V_{\alpha,i_{\alpha}} \cap V_{\alpha',i_{\alpha'}} \not=\emptyset$.
\end{enumerate}
\noindent\textbf{end}

In general, $\widetilde{N}$ will be a function of both $n_{\rm bins}$, $n_{\rm int}$, the function $f$, and the data. The Mapper procedure gives clustering information based on the original data $X_{\rm train}$ and the information contained in $f|_{X_{\rm train}}$. In step (3), a local neighborhood scale must be defined or recovered in order to produce a refinement of the open cover ${V_{\alpha}}$.  This is done by first producing a histogram of the number of components that become connected at varying length scales via single-linkage clustering (see Fig. \ref{fig:hists} in the Appendix available online).  If there are distinct clusters in the data, the histogram will have at least two main peaks: one peak corresponding to points that become connected at smaller distance scales, and another corresponding to points that become connected at larger lengths which represent the distance between clusters.  The heuristic we use to choose the small distance scale is the value at which the first break in the histogram occurs.  This process is repeated for each set $U_{\alpha}$ in the open cover, and a new local neighborhood scale is recovered for each of these sets.  

A higher bin value will tend to push down the distance scale that is required in step (3), and hence produce more nodes in the Mapper output graph.  Thus, a higher bin value can also be seen to correspond to, on average, an increase in the complexity of the Mapper output graph.  In general, there is a bias-variance tradeoff in setting this value.  See Fig. \ref{fig:mapper_objects} in the appendix available online to see how this choice can change the Mapper output. Additionally, we fix the gain for the sets in the open cover ${U_{\alpha}}$, to be a $33\%$ overlap.  This will cause each data point to be assigned to at most $2$ nodes in the Mapper output graph. 
\begin{remark}
\label{remnerve}
The undirected graph constructed from the set of vertices $\{(\alpha,i_{\alpha})\}$, and having edges whenever the intersection of two data components is nonempty, i.e., $V_{\alpha,i_{\alpha}} \cap V_{\alpha',i_{\alpha'}} \not=\emptyset$, is interpreted as the \emph{nerve complex} of the dataset $X$. In particular, when using a one-dimensional filtration the resulting nerve complex is one-dimensional (composed out of nodes and edges only). This construction is generalized to higher dimensions.
\end{remark}
\section{Description of our Training and Testing Classifier Method}\label{sec:structure}

We will refer to our algorithm as a \emph{Mapper Classifier}, or \emph{MC} for short.   In our MC algorithm, using the provided dataset $X_{\rm train}$, we construct a \emph{committee of Mapper graphs}, meaning a pair of sequences of Mapper output graphs and the corresponding filtration functions used to generate them:
\begin{equation}
\label{eqC}
\mathrm{C}(X_{\rm train}, f) := \left(\{M_j\}_{j=1}^{N_C}, \{f_j\}_{j=1}^{N_C}\right) = \left(\{M_j\}, \{f_j\}\right),
\end{equation}
where $f_j\colon \cX\to\R$ are the filter functions and $N_C$ is the number of members chosen to be included in the committee. We emphasize that building a committee of Mapper graphs requires in the first place, some \emph{systematic way of generating filter functions}.  For our analysis, we choose the overall filter $f$ (projection onto a latent space) to be either: 
\begin{enumerate}
\item PCA \cite{PCA}, where each $f_j$ is a mapping onto the $j$-th principal component which is constructed using $X_{\rm train}$.
\item An autoencoder, which we use several variants of, including: contractive \cite{CAE}, deep \cite{DAE}, and variational \cite{VAE}. $f_j$'s were constructed using an autoencoder as a mapping onto the latent space generated by hidden nodes of the autoencoder, where $f_j$ is the projection onto the $j$-th hidden node of the autoencoder.  The autoencoders are first trained using $X_{\rm train}$.
\end{enumerate}
We define a map that takes data points in $X_{\rm train}$ to the vector representation of the Mapper nodes in a single Mapper output graph $M$ by:
\begin{equation}
\label{gm}
g_{M}\colon X_{\rm train} \to \R^{N_M},    
\end{equation}
where $N_M$ is the number of nodes of Mapper $M$.  The map $g_{M}$ has a natural definition for all the data in $X_{\rm train}$, as it sends points to a vectorized binary representation: $g_{M}(X_{\rm train})\subset\{0,1\}^{N_M}$.  Each training point $x\in X_{\rm train}$ gets assigned by $g_M$ the binary vector representing the mapper vertices $\widetilde{V_{\alpha,l_{\alpha}}}$, such that $x\in\widetilde{V_{\alpha,l_{\alpha}}}$.
\begin{example}
Assume that a Mapper output graph $M$ is composed of $3$ nodes. Let $x_1,x_2,x_3\in X_{\rm train}$, such that they are in the $1$st, $2$nd and $3$rd vertices of $M$ respectively. Then, $g_M(x_1)=(1,0,0), g_M(x_2)=(0,1,0), g_M(x_3)=(0,0,1)$.
\end{example}

The map $g_M$ has a natural extension to the Mappers committee $\mathrm{C}$:
\begin{equation}
\label{gc}    
g_{C}\colon X_{\rm train} \to \R^{\sum_{j=1}^{N_C}{N_{M_j}}},
\end{equation}

where $M_j$ represents the $j$-th Mapper output graph in the committee, which consists of $N_C$ total Mappers, and $N_{M_j}$ corresponds to the number of nodes in $M_j$.  This procedure can be seen as joining the vector representations for all the individual Mappers in the committee into a long vector.  

The procedure that has been presented for mapping data points in $X_{\rm train}$ to the vectorized binary representation applies to the training data only, and for datapoints in $X_{\rm test}$ we need to utilize an alternate procedure, that we present in Sec.~\ref{sectestproc}.

\subsection{Training Procedure}\label{sec:train}
\noindent\textbf{INPUT:}\\
    * The internal parameters of the Mapper algorithm (see Sec.~\ref{secmapper}).\\
    * The number of components ($N_C>0$) used to build the committee of Mapper graphs.  We set $N_C=20$ (more on this in Sec. \ref{secnum}).\\
    * The choice in the latent space projection method (either PCA or an autoencoder).\\
    * A split for the training dataset: $X_{\rm train} = X_{\rm train}^1\cup X_{\rm train}^2\cup \dots\cup X_{\rm train}^n$, such that $X_{\rm train}^i\cap X_{\rm train}^j = \emptyset$.  \\
    * Labels $Y_{\rm train}$ for the training set.\\
\noindent\textbf{OUTPUT:}\\
* Classifier for the training set $X_{\rm train}$ with labels $Y_{\rm train}$.\\
\noindent\textbf{begin}
\begin{enumerate}[topsep=0pt,itemsep=0ex,partopsep=1ex,parsep=1ex]
    \item Build the committee of Mapper graphs for each data-split $X_{\rm train}^i$:
$\mathrm{C}_i = \mathrm{C}(X_{\rm train}^i,f^i) = \left(\{M_j^i\}, \{f_j^i\}\right)$.  Note that each projection $f^i$ is trained independently on each of the $X^i_{\rm train}$. 
    \item Using the map $g_{C_i}$ (see Equation \eqref{gc}) applied to each of the $C_i$ committees, compute the binary matrix representations of the subsets denoted by $i$: 
$g_{C_i}(X_{\rm train}^i)\in \{0,1\}^{N_i\times\sum_k{N_{M_k}}}$,
where $N_i$ is the number of examples in $X_{\rm train}^i$.
    \item (only if $X_{\rm train}$ is split) using the computed committee of Mapper graphs, compute the `off-diagonal' binary matrix representations, i.e., for all $i=1,\dots,n$ compute $g'_{C_i}(X_{\rm train}^j) \in \R^{N_j\times\sum_{k}{N_{M_k}}}$ $\text{ for all $j \neq i$,}$
using the mapping procedure presented in Sec.~\ref{sectestproc}.  
    \item Train an end classifier (in our case we use a neural network), using the merged data from the previous step.  The merged data is represented as a matrix which can be seen as a map $g(X_{\rm train}) \in \{0,1\}^{N_{\rm tot}\times N_{M_{\rm tot}}}$, where $N_{\rm tot}$ is the total number of instances in $X_{\rm train}$, and $N_{M_{\rm tot}}$ is the total number of nodes over the entire collection of committees, i.e., $N_{M_{\rm tot}} = \sum_i \sum_{j}{N_{M_j}}$ where $i$ runs over all data splits and $j$ denotes specific Mappers in a split.  In block form:
{\scriptsize\begin{align*}
    \begin{bmatrix}
    g_{C_1}(X_{\rm train}^1) & g'_{C_2}(X_{\rm train}^1) & \dots & g'_{C_n}(X_{\rm train}^1)\\
    g'_{C_1}(X_{\rm train}^2) & g_{C_2}(X_{\rm train}^2) & \dots & g'_{C_n}(X_{\rm train}^2)\\
    \vdots & \vdots & \vdots & \vdots \\
    g'_{C_1}(X_{\rm train}^n) & g'_{C_2}(X_{\rm train}^n) & \dots & g_{C_n}(X_{\rm train}^n)
    \end{bmatrix}
\end{align*}}
\end{enumerate}
\noindent\textbf{end}
\begin{figure}[h]
	\centering
	\begin{subfigure}{0.125\textwidth} 
   \label{fig:mapper_10_10} 
		\includegraphics[width=\textwidth]{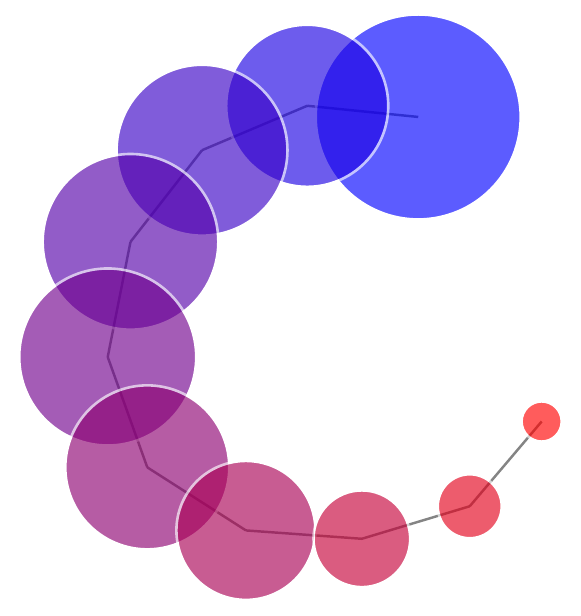}
	\end{subfigure}
	\hspace{5em} 
	\begin{subfigure}{0.125\textwidth}
   \label{fig:mapper_20_20} 
		\includegraphics[width=\textwidth]{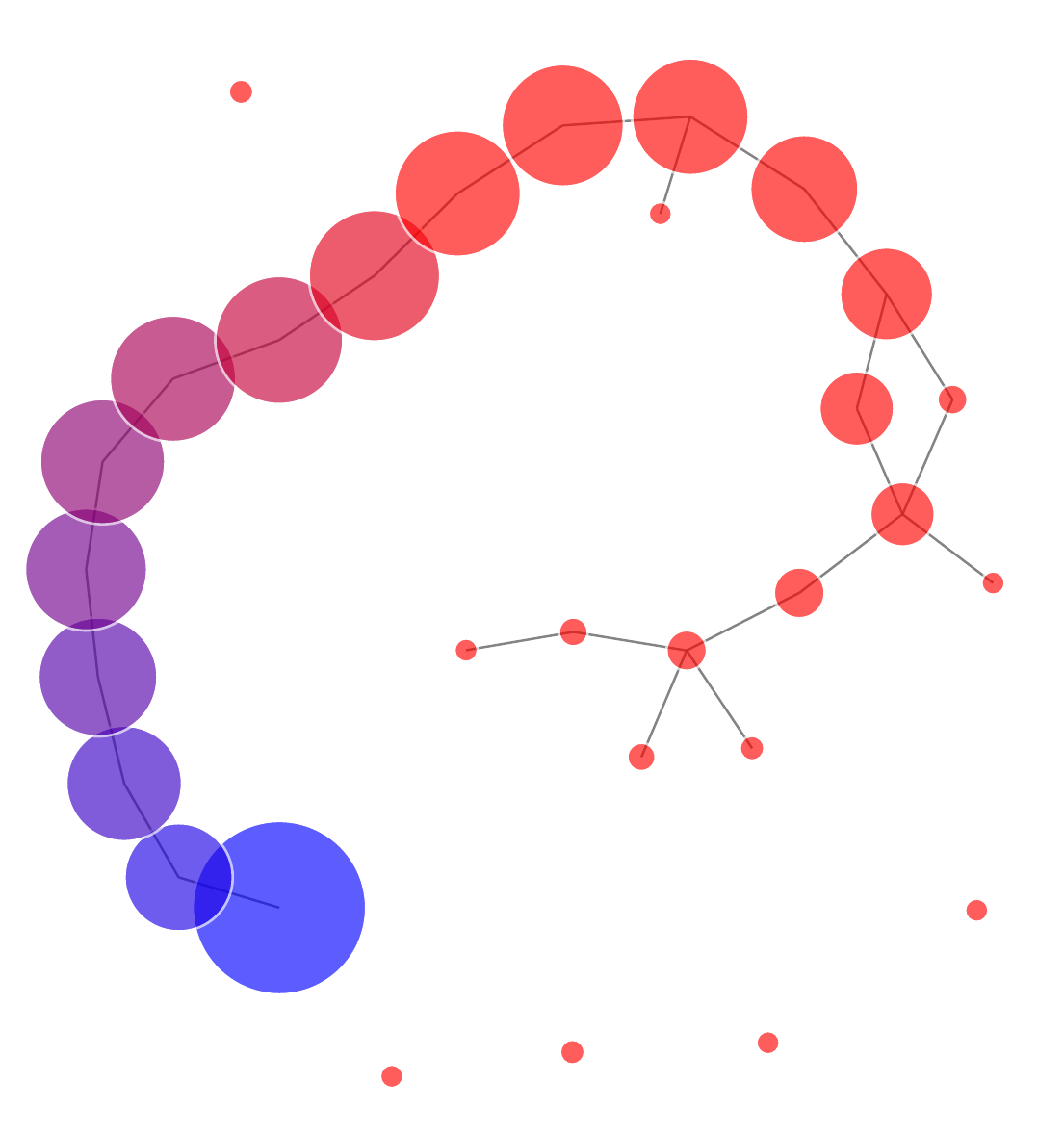}
	\end{subfigure}
\caption{Mapper output graphs computed for $10$k MNIST training data, with contractive autoencoder projection (see Sec. \ref{sec:structure}).  Color corresponds to projection value, and node size corresponds to total number of points in the node. $n_{\rm int}=n_{\rm bins}=10$ for the LHS figure, $n_{\rm int}=n_{\rm bins}=20$ for the RHS figure.}
\label{fig:mapper_objects}
\end{figure}

\subsection{Mapping Unseen Points to the Committee}
\label{secmapping}
We describe how we construct the $g'_{C_i}$ map, a generalization of the $g_{C_i}$ map \eqref{gc} to test datasets (analogously define $g'_{M}$ map). $g'_{C_i}$ is used in order to map test data points to an existing committee of Mapper graphs that is constructed using the $i$-th split of data. 

The data-points that are being tested through the committee of Mapper graphs can be any set in principle. 
In the algorithm below, we will denote this set as $X_{\rm new}$, and assume that it is provided as input to our testing algorithm. Examples are $X_{\rm new}=X^i_{\rm train}$ and $X_{\rm new}=X_{\rm test}$ (or some perturbations of data in $X_{\rm test}$ as used for robustness testing). We describe in more detail the splitting procedure and its utility in Appendix~\ref{sec:split} available online.\\
\label{sectestproc}
\noindent\textbf{IN:}\\
    * The same input information from the training procedure in Sec. \ref{sec:train},\\
    * $k \geq 1$, the number of nearest neighbors considered in the algorithm.  We set $k=6$ (more on this in Sec. \ref{secnum}),\\
    * the testing data $X_{\rm new}$.\\
\noindent\textbf{OUT:}\\
    * Vector committee representation of the new points $g'_{C_i}(X_{\rm new})$.\\
\noindent\textbf{begin}
\begin{enumerate}[topsep=0pt,itemsep=-0.5ex,partopsep=1ex,parsep=1ex]
    \item Find $\alpha$'s such that $\left|f_j(x) - \mbox{mid}(U^{ij}_{\alpha})\right| < \frac{\mbox{max}(U^{ij}_{\alpha})-\mbox{min}(U^{ij}_{\alpha})}{2}(1+\delta)$, for $x \in X_{\rm new}$. If $f_j(x) < \min(I^{ij})$ then set $\alpha=1$, if $f_j(x) > \max(I^{ij})$ set $\alpha=n$.
    The $\delta$ parameter consequently enhances the robustness as it broadens the search space (defined in step 3).  The $\alpha$ denotes the interval in the cover for $\mathbb{R}$ (see Sec. \ref{secmapper}), $i$ denotes the split, and $j$ denotes the filter (i.e $j=1$ for PCA would mean the first principal component).
    \item For these $\alpha$, collect the corresponding refined vertices
$\widetilde{V^{ij}_{\alpha,l_{\alpha}}}$.  These are the clusters in the $i$-th data split $X^i_{\rm train}$ that are mapped to $U^{ij}_{\alpha}$ by $f_j$ and are indexed by $l_{\alpha}$.
    \item For all $x\in X_{\rm new}$ and for all splits indexed by $i$ perform the $k$-nearest neighbors search within the Mapper vertices found in the previous step $\{\widetilde{V^{ij}_{\alpha,l_{\alpha}}}\}$  to find $nn^i_1(x),\dots,nn^i_k(x)\in X_{\rm train}$, where the distance function is chosen to be consistent with the choice of metric used to compute the Mapper committee (Euclidean).
    \item For all $x\in X_{\rm new}$ and for all splits $X^i_{\rm train}$ define: 
\begin{multline*}
g^\prime_{C^i}(x) = w^i_1\cdot g_{C_i}(nn^i_1(x)) + w^i_2\cdot g_{C_i}(nn^i_2(x))+\dots\\+w^i_k\cdot g_{C_i}(nn^i_k(x)),
\end{multline*}
where the weights are defined by
$$
w^i_j = \left[d(x, nn^i_j(x))+\eta\right]^{-1} / \sum_{j=1}^{k}{\left[d(x, nn^i_j(x))+\eta\right]^{-1}},
$$
where $\eta=10^{-5}$.
\end{enumerate}
\noindent\textbf{end}

$\eta$ is a stability parameter and does not  have a large impact on the results so we opt for setting it to this small value.  We use this procedure for any new data point not yet assigned to the $i$-th committee, including both $X_{\rm train} \setminus X_i$ and $X_{\rm test}$.  This process is precisely the $g'$ map mentioned in Sec. \ref{sec:train}, and we use it to fill in the missing values of $g(X_{\rm train})$ and to construct in its entirety $g(X_{\rm test})$.

The computational complexity of the training and testing portions of solely the Mapper procedure are $\mathcal{O}({n^2})$, where $n$ is the number of data points.  There are additional computational costs incurred before (i.e. when constructing $f$) and after this procedure (i.e. when training the end classifier - see Fig. \ref{fig:tdann}).
\section{Numerical Experiments}\label{secnum}
In order to quantify the overall robustness of studied classifiers we perform several black box random noise attacks. We do not perform white box attacks, as it is not clear for us how to efficiently perform such attack on our classifier. As noted earlier any gradient based attack is essentially not applicable due to the graph discretization step. Perturbations (small) ($x'$) are generated from the correctly classified test-set examples ($x$). Hence, the ground truth $c$ stays the same for the perturbed examples $c(x)=c(x')$. Let $h\colon\mathcal{X}\to\mathcal{Y}$ denotes the predictions of a classifier, where $\mathcal{Y}$ is the set of labels. To measure the robustness of a classifier we use the function counting the number of misclassified examples for a dataset $X$, with all the examples being perturbed with Gaussian blur noise ($blur$) within $l^2$ range $(0,x]$
\[
\text{missclassified\,}_{X,blur}((0,x])\in\{0,\dots,\#X\},
\]
where $blur$ can be replaced with $s\&p$ or $gauss$, which is in turn used in the definition of the normalized accuracy $a(x)$ \eqref{eq:accuracy}.

The main numerical results we report are this accuracy with respect to different noise models for various classifiers, and all are applied to the usual $10$k test MNIST/Fashion-MNIST data.  The noise models we use include: Gaussian blur, random noise selected from a Gaussian, and salt \& pepper noise. The \emph{Gaussian blur} model performs a convolution with a $2$-dimensional Gaussian centered at each pixel in the image.  The $\lambda$ parameter we use for robustness calculations is related to the Gaussian standard deviation by: $\sigma = 28 \lambda$.
The \emph{salt \& pepper} or \emph{s\&p} model replaces random pixel values with the minimum (i.e. ``pepper'') or maximum (i.e. ``salt'') in the image.  Setting $q_1$ as the probability of flipping a pixel, and $q_2$ as the ratio of salt to pepper, we use: $q_1 = \lambda$ and $q_2=\frac{1}{2}$. The \emph{Gaussian} model adds in noise to each pixel which is sampled from a Gaussian distribution.  The distribution we use is centered at zero and has $\sigma=0.1\sqrt{\epsilon}$.  Additionally, we train on both a $10$k subset (examples chosen by random) in the data and the full $60$k set. We do not normalize data using std. dev. and mean. For testing, we use the usual $10$k MNIST/Fashion-MNIST test set. We compute the normalized accuracy for the dataset $X_{test}$ and perturbation method (blur/gauss/s\&p) as:

\vspace{-0.5em}
\begin{equation}\label{eq:accuracy}
    a_{X_{test},blur}(x) = 1-\frac{\text{missclassified\,}_{X_{test},blur}((0,x])}{\text{initial correct}},
\end{equation}
where $blur$ can be replaced with $gauss$ or $s\&p$, $\text{missclassified\,}_{X_{test},blur}((0,x])$ is the number of misclassified perturbations within $l^2$ perturbation norm range $(0,x]$.  This equation has the benefit that it removes, to an extent, potential dependencies of robustness on the data itself (i.e., we would like robustness to be more a property of the classifier rather than how the classifier interacts with the specific dataset).  

There are a few hyperparameters we use, which we will briefly mention here.  We set $N_C$, the number of Mappers in a committee to 20, because for various classifiers, the initial classification accuracy levels off around this number.  By initial classification accuracy, we mean the number of initially correctly classified instances with no noise added. 
To fine-tune the hyperparameters we used some heuristics we derived by experimenting using a single  10k datasplit.
For example we set the number of bins and intervals $n_{\rm int}=n_{\rm bin}=10$ heuristically, through a combination of what provides a high initial classification accuracy while still uncovering interesting topology as seen by the Mapper output graph (see Fig. \ref{fig:mapper_objects} and Tab. \ref{tab:mapperdimensions} in the appendix available online).  The $\delta$ parameter is set to $0.2$ (see Sec. \ref{sectestproc}) as this appears to provide the best robustness.  Lastly, we set $k=6$ when mapping new points to the committee as this provided the best overall robustness for our Mapper classifier (see Sec. \ref{secmapping}).

We compare our approach based on the Mapper method to the robustness results of a CNN (the standard architecture LeNet \cite{lenet} was used for MNIST, and a 5 layer CNN with batch normalization achieving accuracy close to state of the art listed at \cite{fashionurl} for Fashion-MNIST).  The Mapper based methods only differ in their initial projections.  For $10$k training of MNIST/Fashion-MNIST, we investigate Mapper approaches based on: PCA, contractive autoencoder (``CAE'') $784$-sigmoid-$20$-linear-$784$. For the CAE we include a contractive loss term with a multiplying factor of $10^{-4}$. Deep autoencoder (``DAE''): $784$-ReLU-$1000$-ReLU-$500$-ReLU-$250$-linear-$20$ encoder and symmetric decoder; variational autoencoder (``VAE''): $784$-ReLU-$512$-linear-$20$ encoder and symmetric decoder with sigmoid output.  For the VAE, we insert a $\beta$ term multiplying the KL divergence in the loss \cite{higgins2017beta}, and in our case, $\beta=0$ provides the most robust results.  All these methods can be thought of as projection models to $20$-dimensional space (i.e. $N_C=20$), and for an example of the VAE projection in a $2$-dimensional subspace, see Fig. \ref{fig:latentspace} in the appendix available online.  For $60$k training of MNIST, we investigate Mapper approaches based on just PCA and VAE since they appear to be the best performing on average. For 60k training of Fashion-MNIST we investigate Mapper approach based on just PCA, which appear to be best performer. In this case we study just PCA due to the computational limitation of our prototype implementation. We will investigate other also other approaches and ways of improving efficiency as a future study. We present the structure of the end classifier that we use in our MC method in Appendix~\ref{secclassifier} available online.

For the initial accuracies of our methods and the traditional CNN approach, see Tab. \ref{tab:accuracies}; for the overall robustness calculations, see Figs. \ref{fig:l2_mnist_10k} through \ref{fig:l2_fashion_60k}.

\subsection{Examples of Image Data as a Function of Noise Parameter}
In Fig.~\ref{fig:mnistp} we present some representative examples of perturbation of images from MNIST dataset. On Fig.~\ref{fig:fmnistp} we present some examples of perturbation of images from Fashion-MNIST dataset. More detailed examples can be found in the appendix to extended version of this paper.
\begin{figure}[h]
\begin{center}
  \begin{subfigure}[]{0.1\textwidth}
    \includegraphics[width=\textwidth]{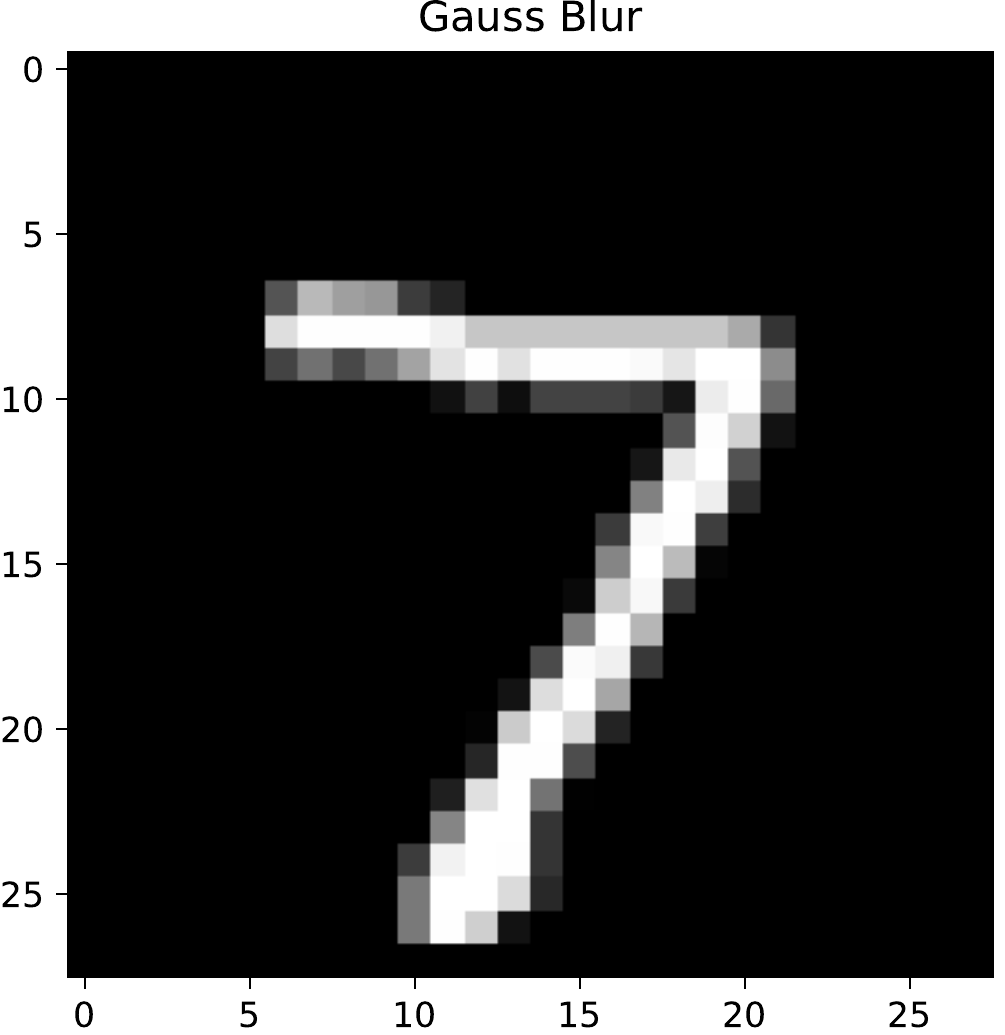}
     \caption{$l^2 \approx 0.02$}
   \label{fig:gb_1} 
  \end{subfigure}
\hspace{0.1em}
  \begin{subfigure}[]{0.1\textwidth}
    \includegraphics[width=\textwidth]{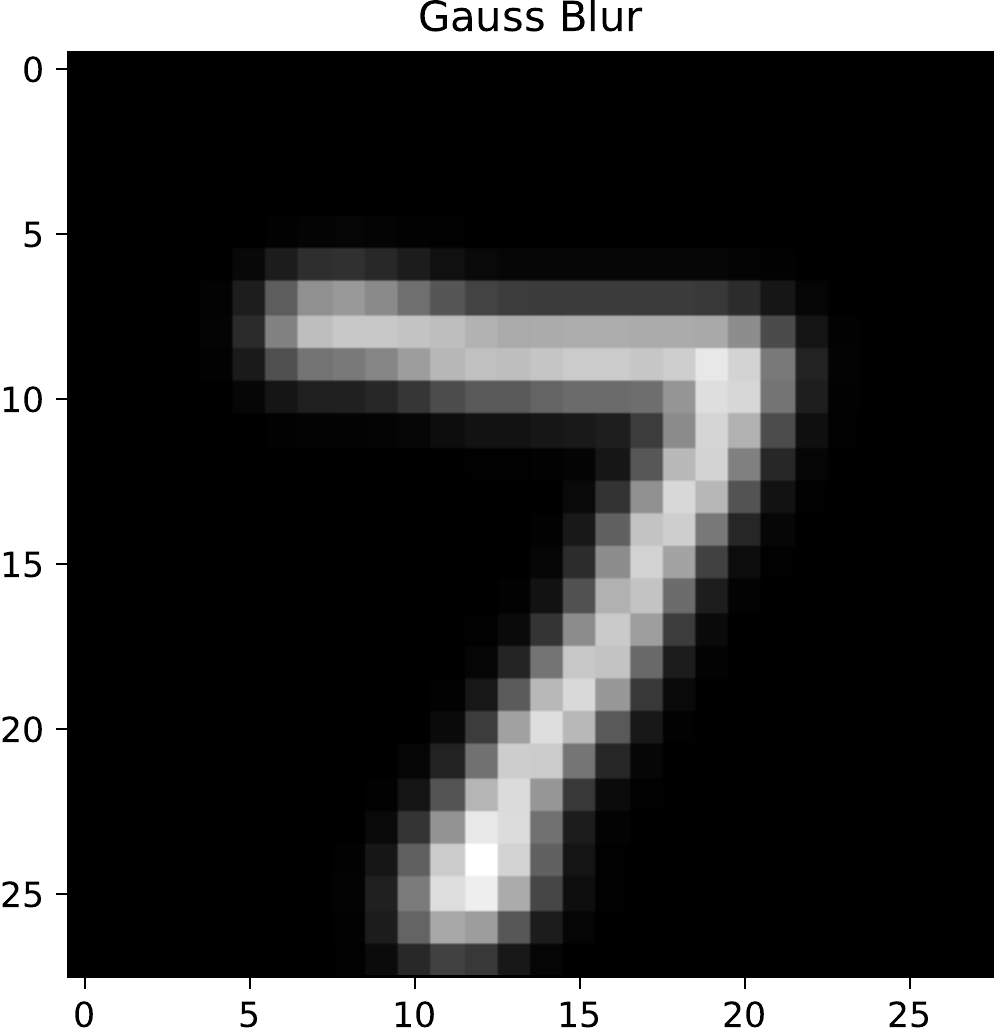}
     \caption{$l^2 \approx 2.5$}
   \label{fig:gb_3} 
  \end{subfigure}  
\hspace{0.1em}
  \begin{subfigure}[]{0.1\textwidth}
    \includegraphics[width=\textwidth]{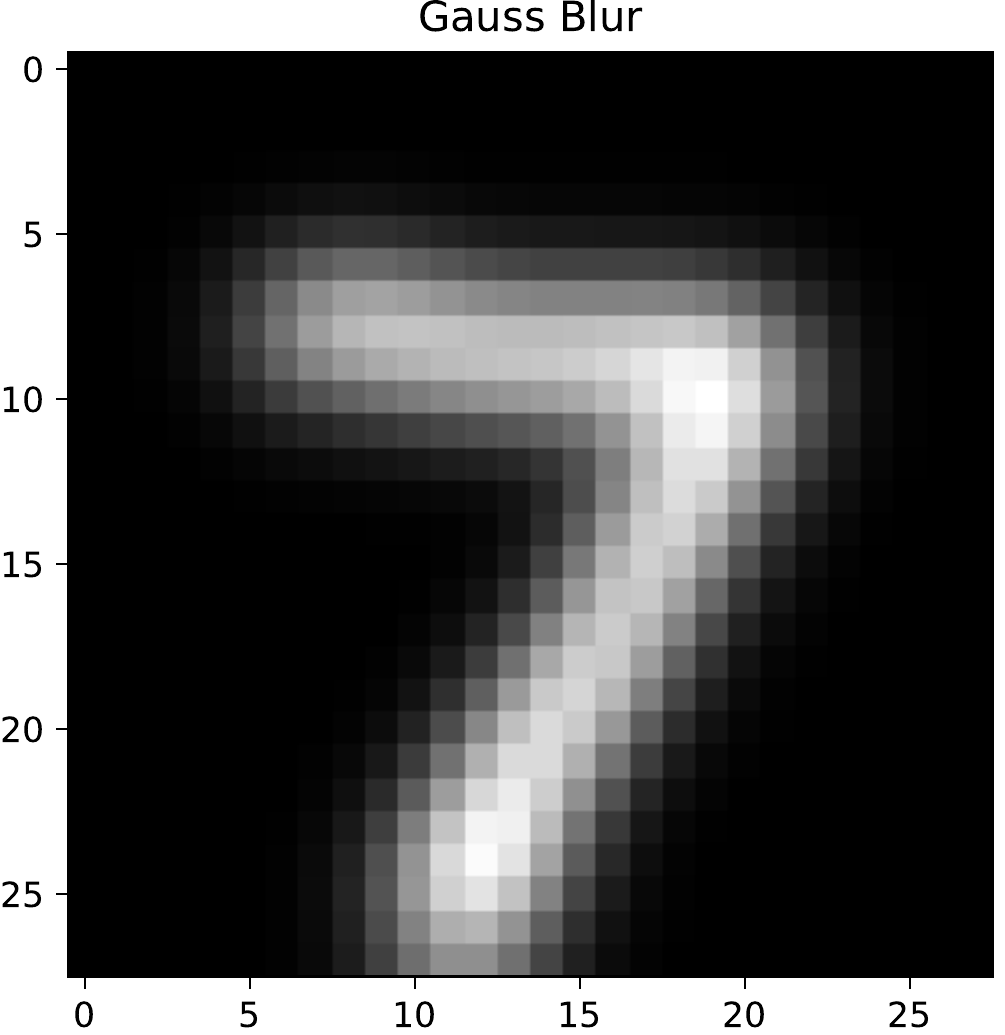}
     \caption{$l^2 \approx 4.6$}
   \label{fig:gb_6} 
  \end{subfigure}  

  \begin{subfigure}[]{0.1\textwidth}
    \includegraphics[width=\textwidth]{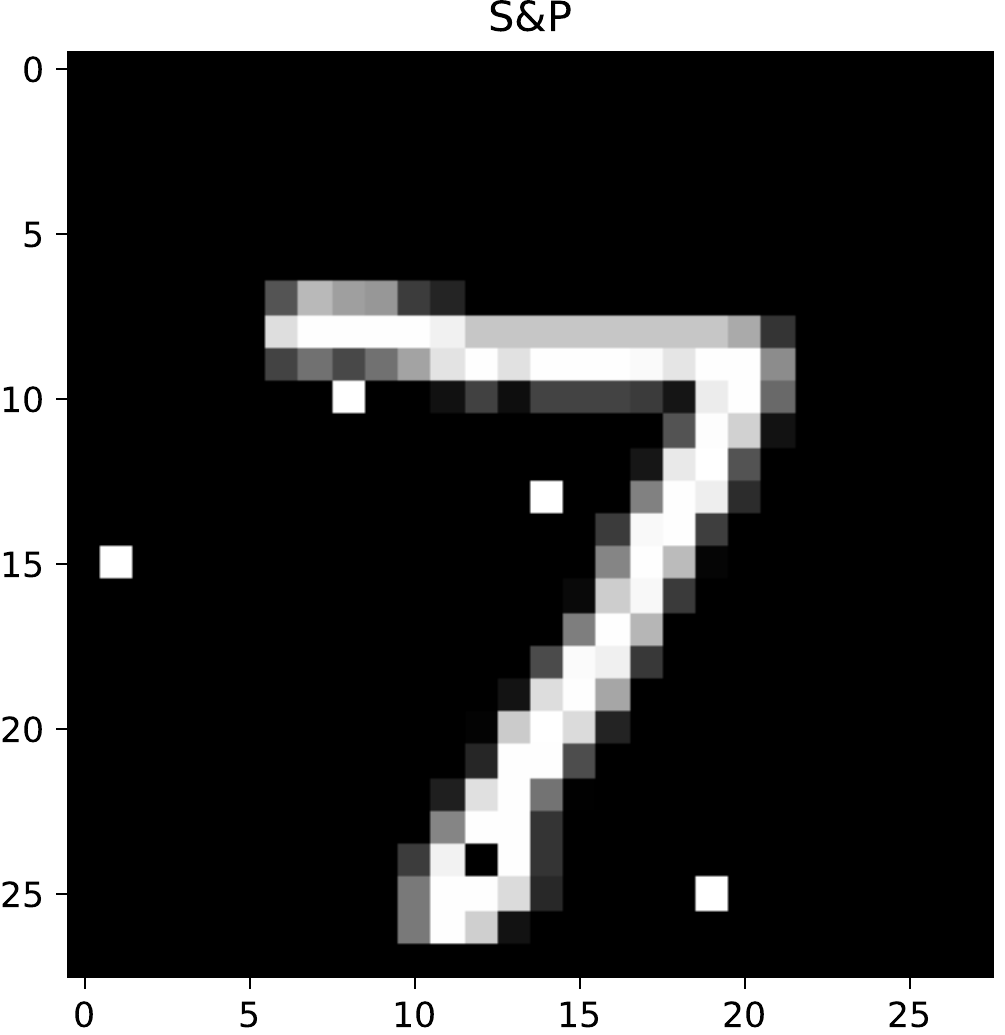}
     \caption{$l^2 \approx 2.4$}
   \label{fig:sp_10} 
  \end{subfigure}
\hspace{0.1em}
  \begin{subfigure}[]{0.1\textwidth}
    \includegraphics[width=\textwidth]{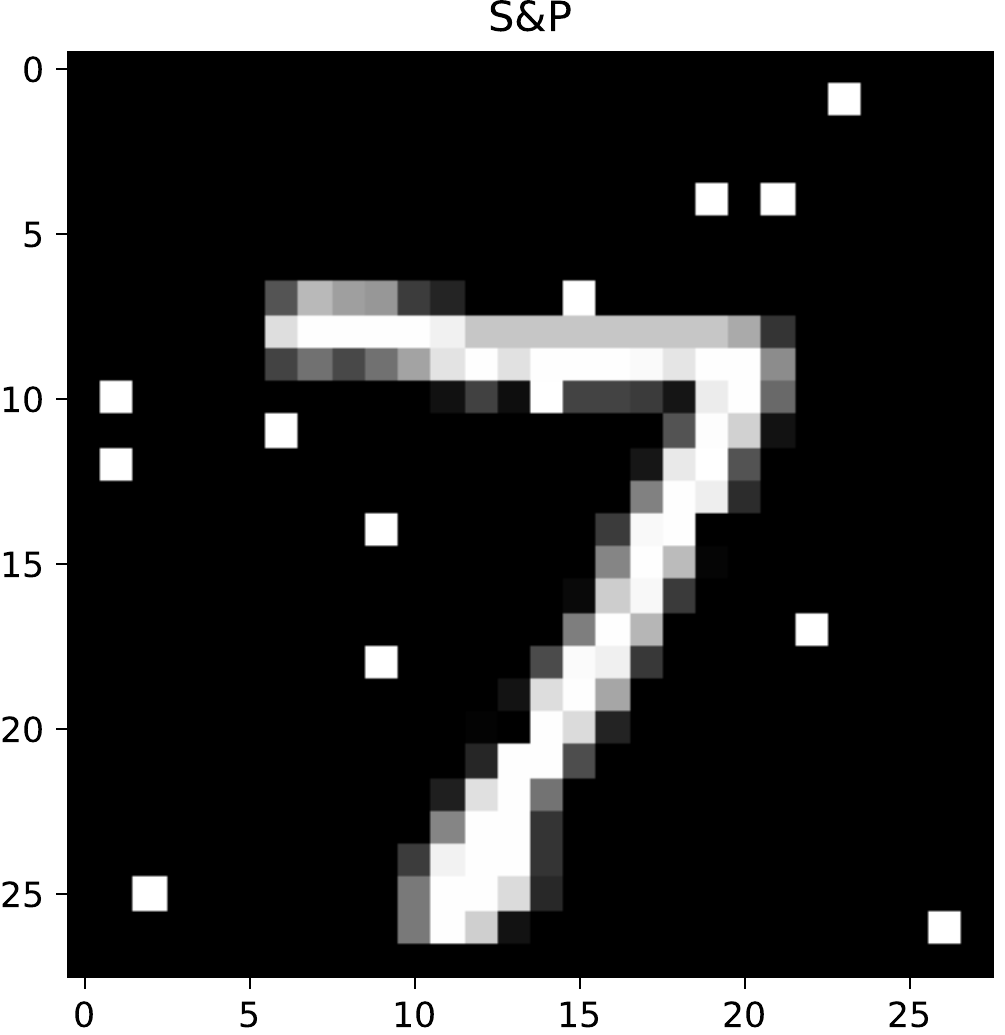}
     \caption{$l^2 \approx 3.6$}
   \label{fig:sp_30} 
  \end{subfigure}  
\hspace{0.1em}
  \begin{subfigure}[]{0.1\textwidth}
    \includegraphics[width=\textwidth]{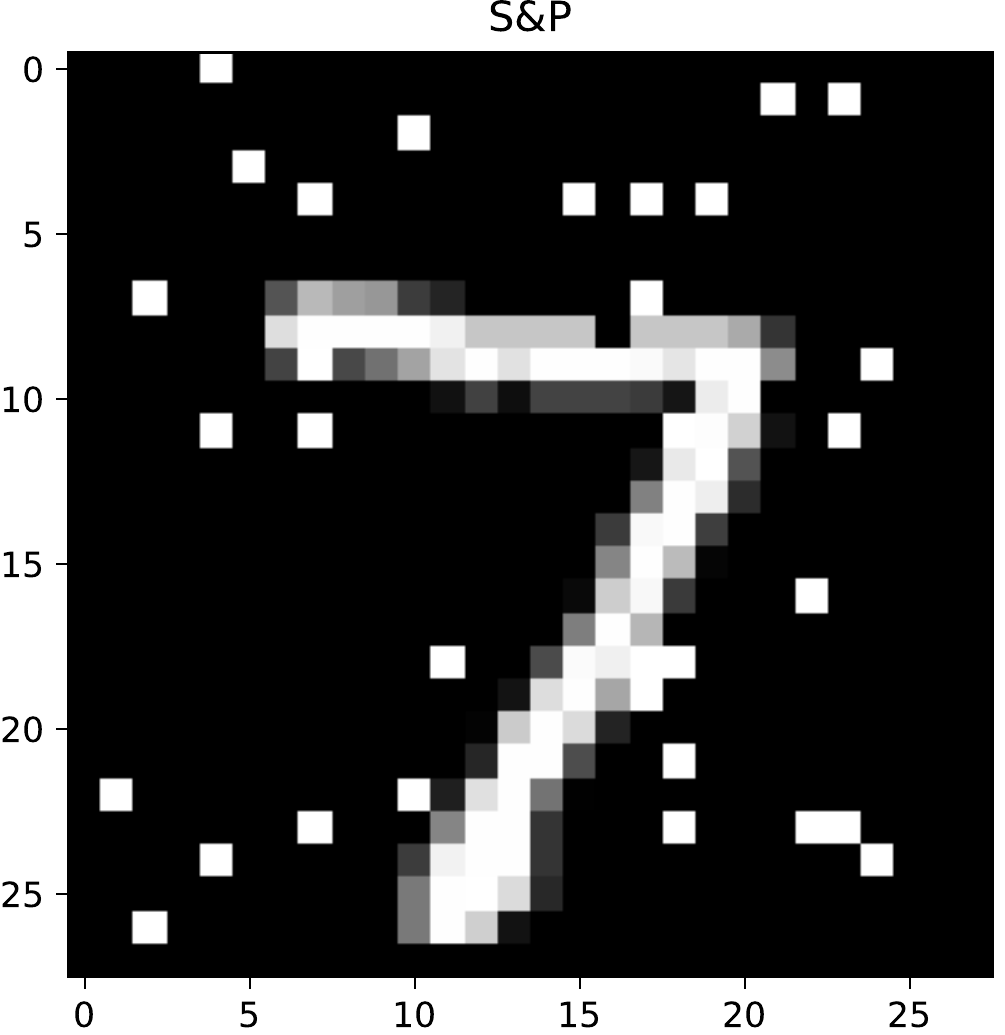}
     \caption{$l^2 \approx 5.6$}
   \label{fig:sp_60} 
  \end{subfigure} 
  
  \begin{subfigure}[]{0.1\textwidth}
    \includegraphics[width=\textwidth]{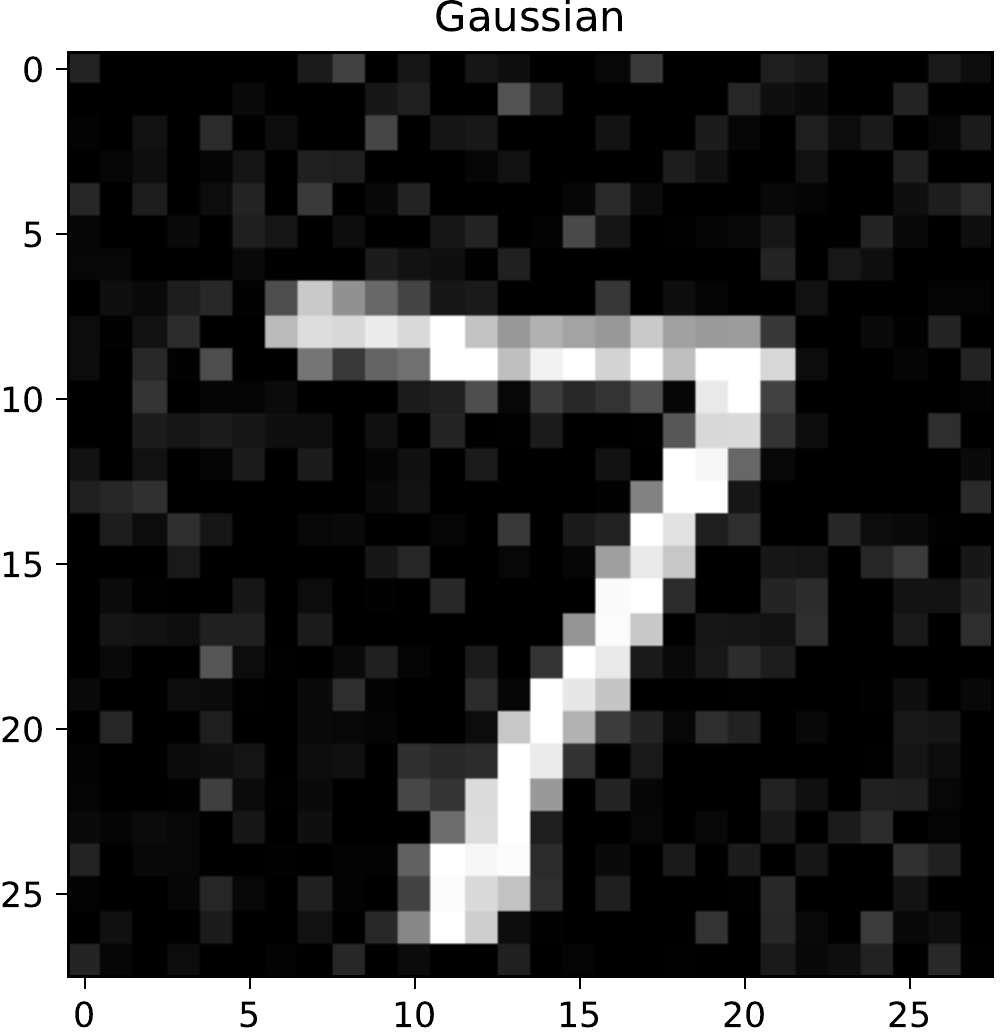}
     \caption{$l^2 \approx 2.2$}
   \label{fig:gaussian_1} 
  \end{subfigure}
\hspace{0.1em}
  \begin{subfigure}[]{0.1\textwidth}
    \includegraphics[width=\textwidth]{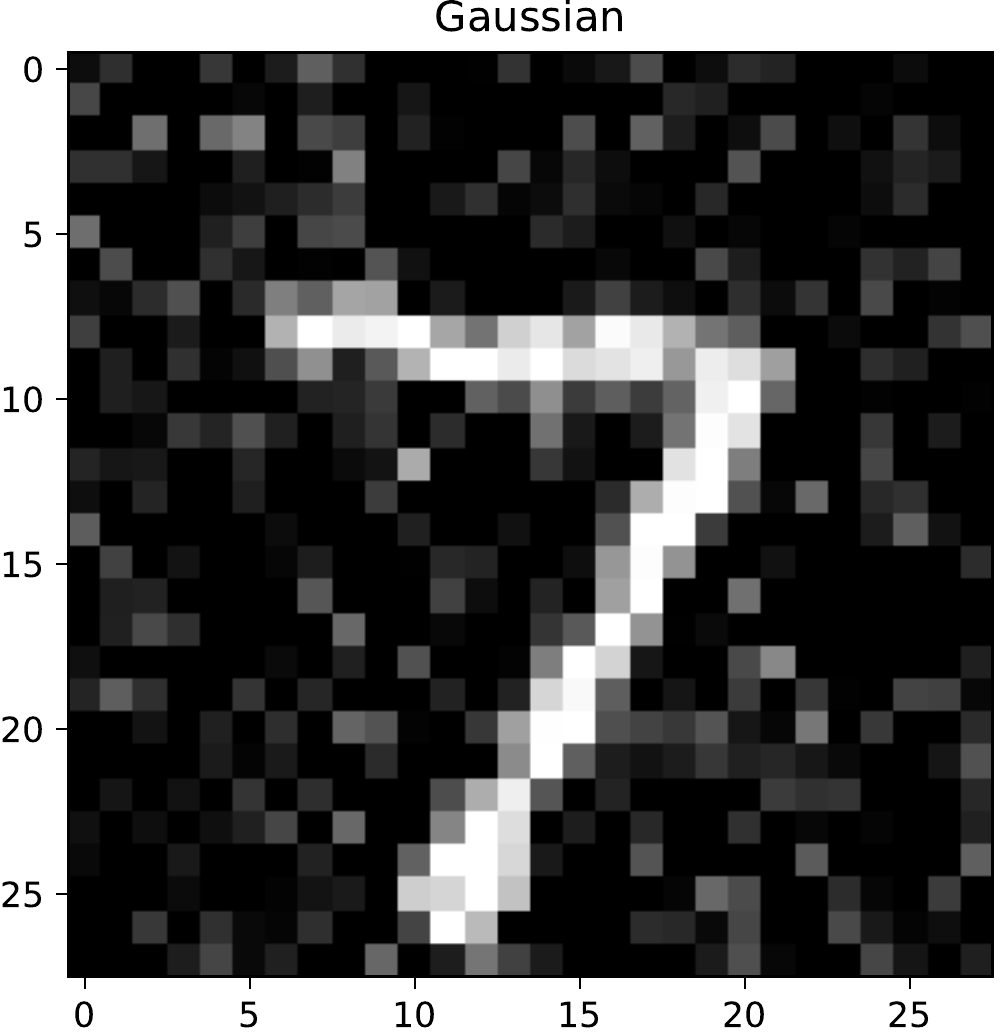}
     \caption{$l^2 \approx 3.8$}
   \label{fig:gaussian_4} 
  \end{subfigure}  
\hspace{0.1em}
  \begin{subfigure}[]{0.1\textwidth}
    \includegraphics[width=\textwidth]{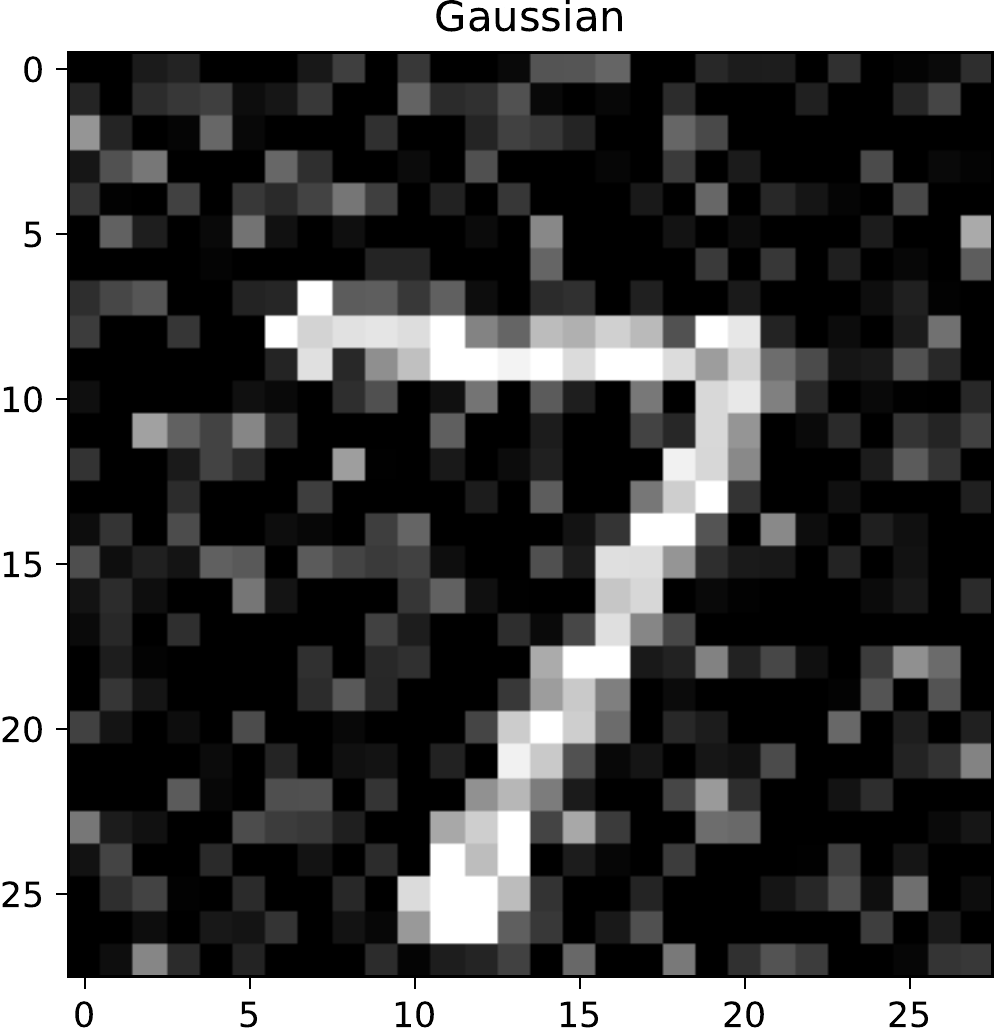}
     \caption{$l^2 \approx 4.6$}
   \label{fig:gaussian_6}
  \end{subfigure}
  \end{center}
\caption{Some examples of a digit 7 from MNIST database perturbed using the noise models.}
\label{fig:mnistp}
\end{figure}
\begin{figure}[h]
\begin{center}
  \begin{subfigure}[]{0.1\textwidth}
    \includegraphics[width=\textwidth]{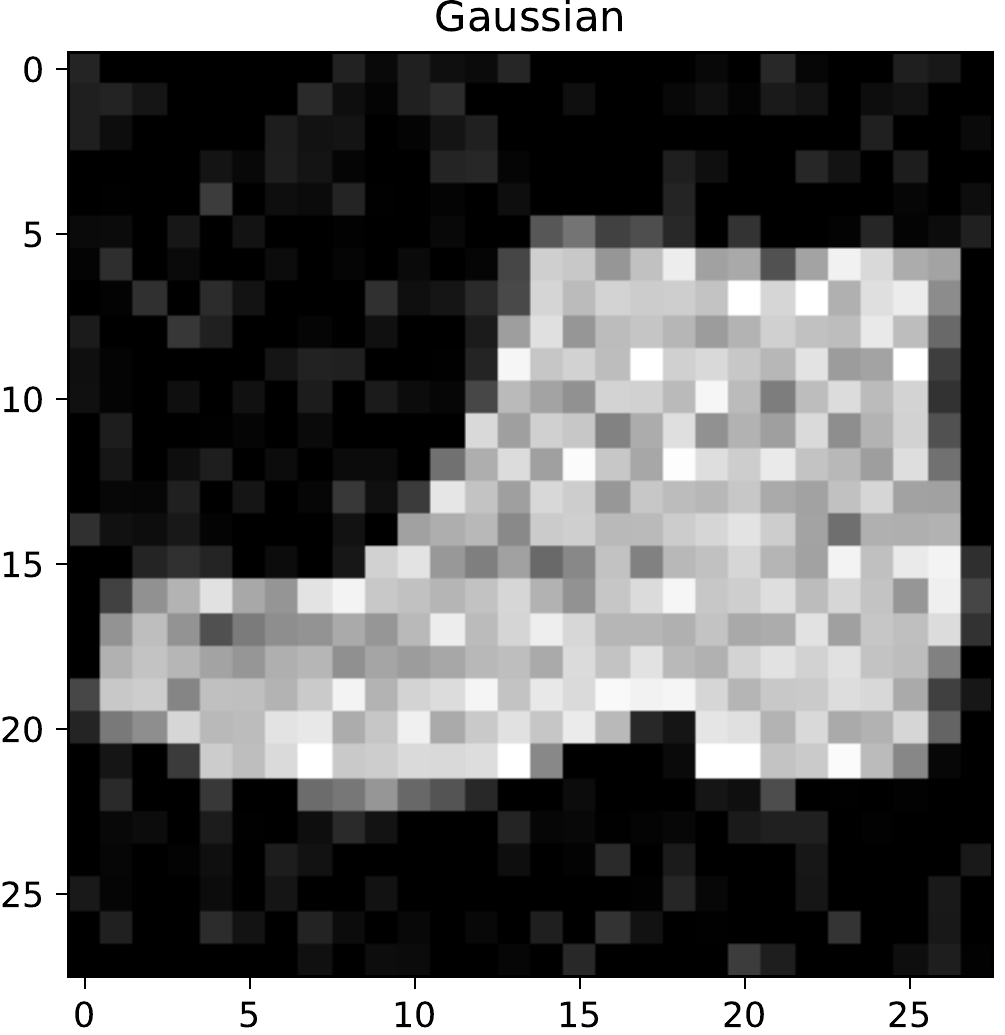}
     \caption{$l^2 \approx 2.4$}
   \label{fig:gaussian_1f} 
  \end{subfigure}
\hspace{0.1em}
  \begin{subfigure}[]{0.1\textwidth}
    \includegraphics[width=\textwidth]{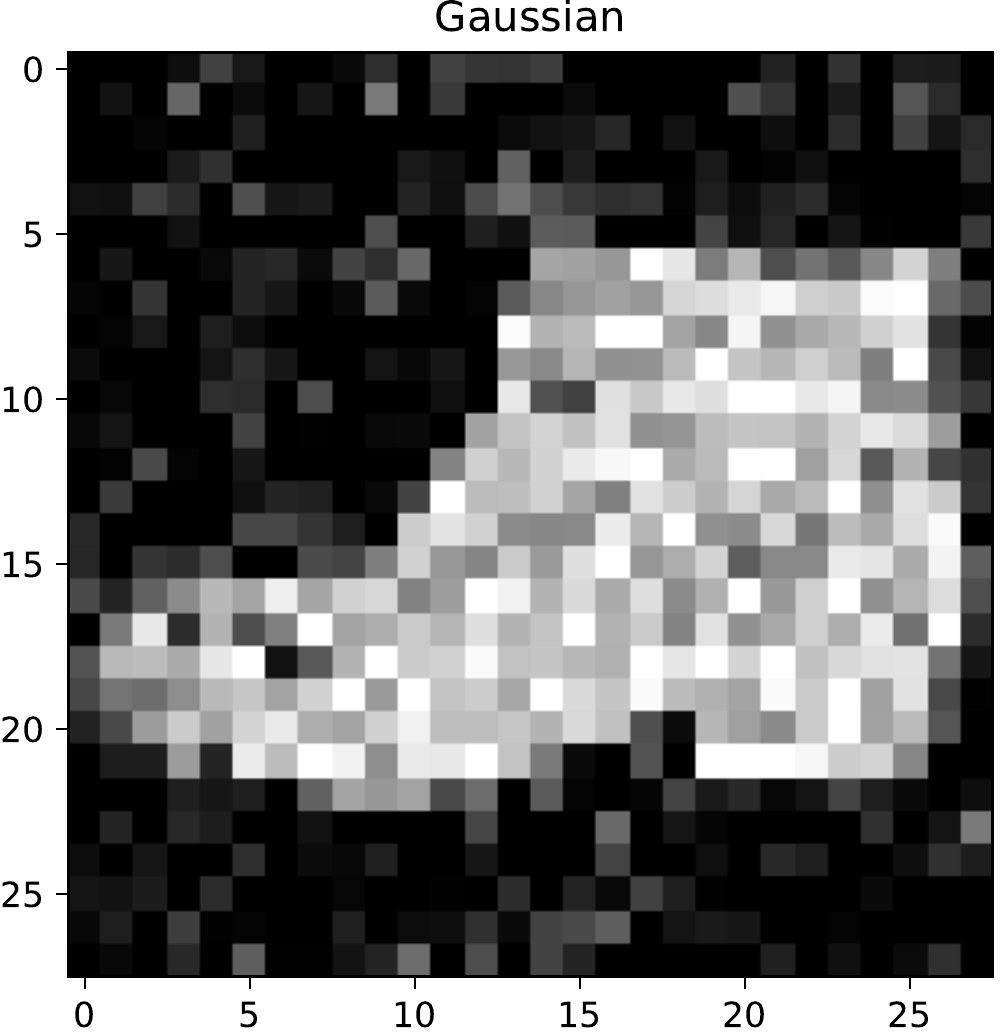}
     \caption{$l^2 \approx 3.9$}
   \label{fig:gaussian_3f} 
  \end{subfigure}  
\hspace{0.1em}
  \begin{subfigure}[]{0.1\textwidth}
    \includegraphics[width=\textwidth]{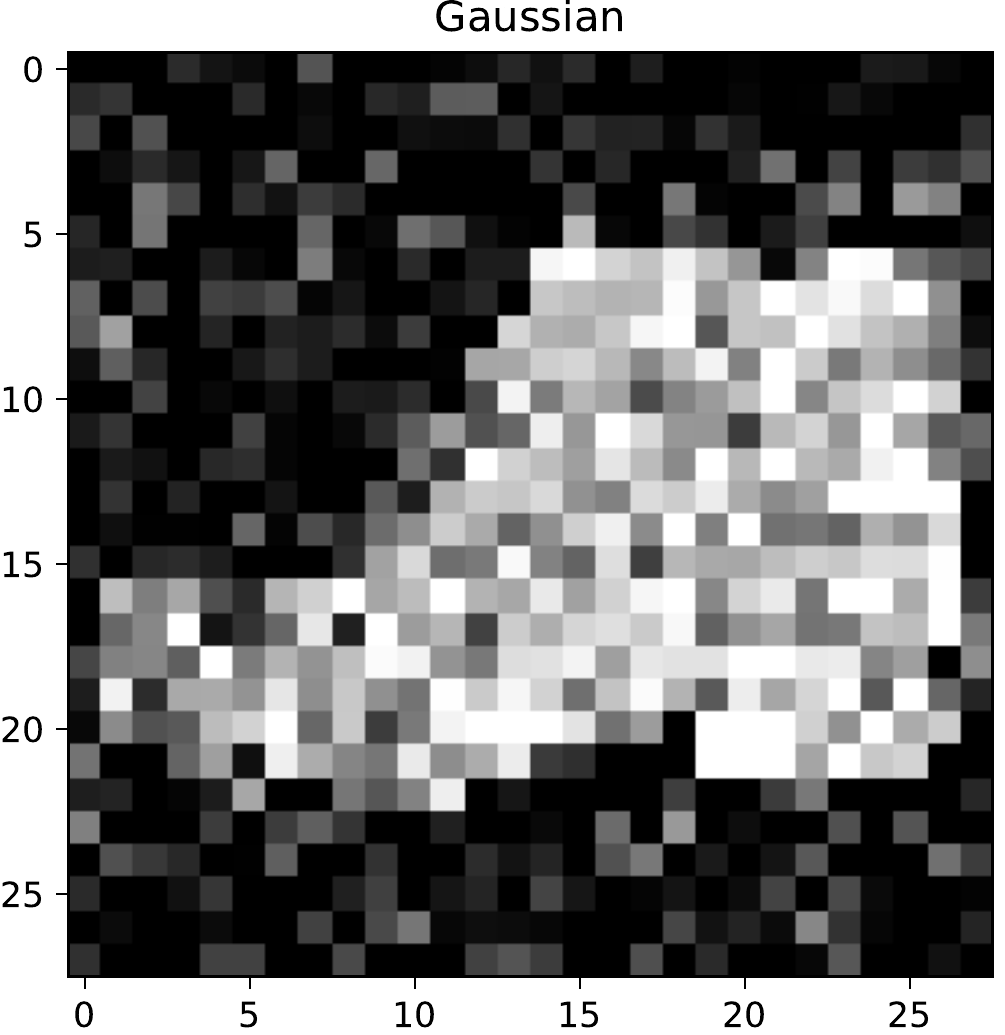}
     \caption{$l^2 \approx 5.5$}
   \label{fig:gaussian_6f}
  \end{subfigure}  
  
  \begin{subfigure}[]{0.1\textwidth}
    \includegraphics[width=\textwidth]{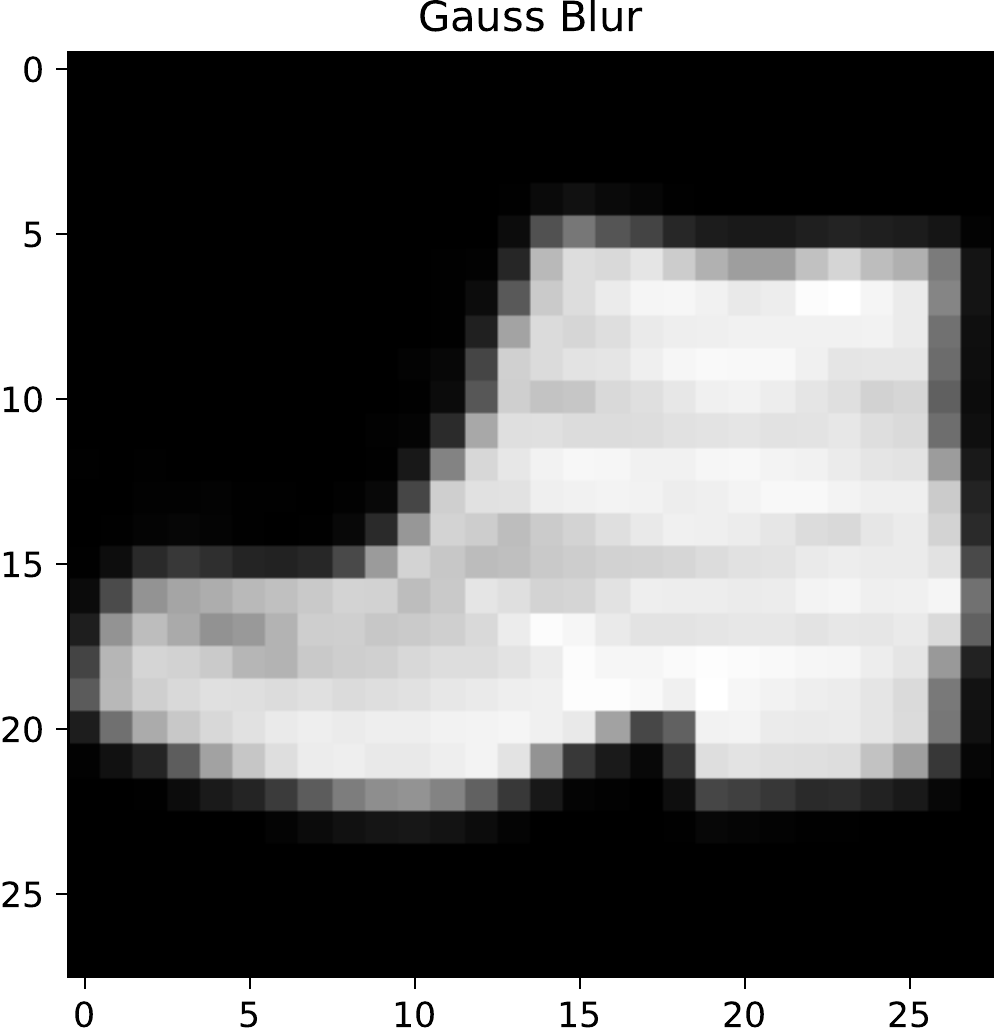}
     \caption{$l^2 \approx 1.4$}
   \label{fig:gb_2f} 
  \end{subfigure}
\hspace{0.1em}
  \begin{subfigure}[]{0.1\textwidth}
    \includegraphics[width=\textwidth]{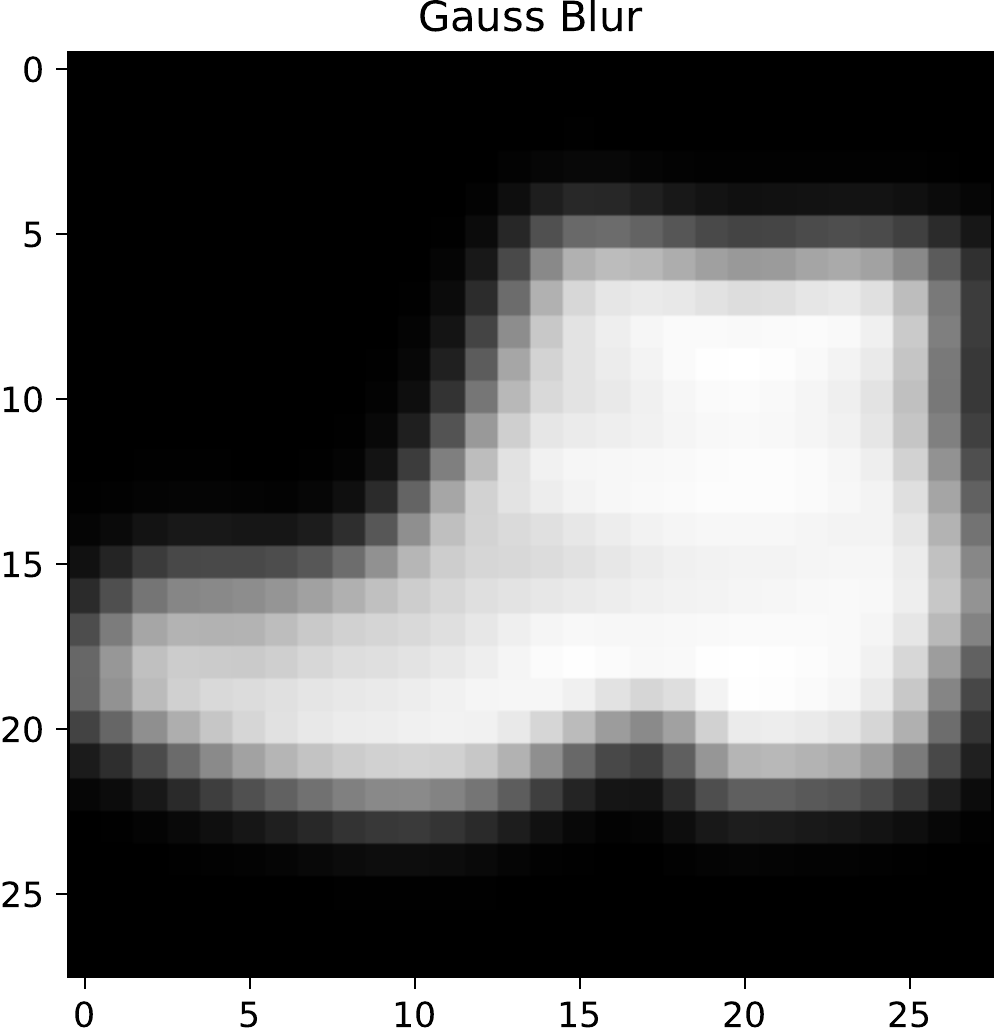}
     \caption{$l^2 \approx 3.1$}
   \label{fig:gb_4f} 
  \end{subfigure}  
\hspace{0.1em}
  \begin{subfigure}[]{0.1\textwidth}
    \includegraphics[width=\textwidth]{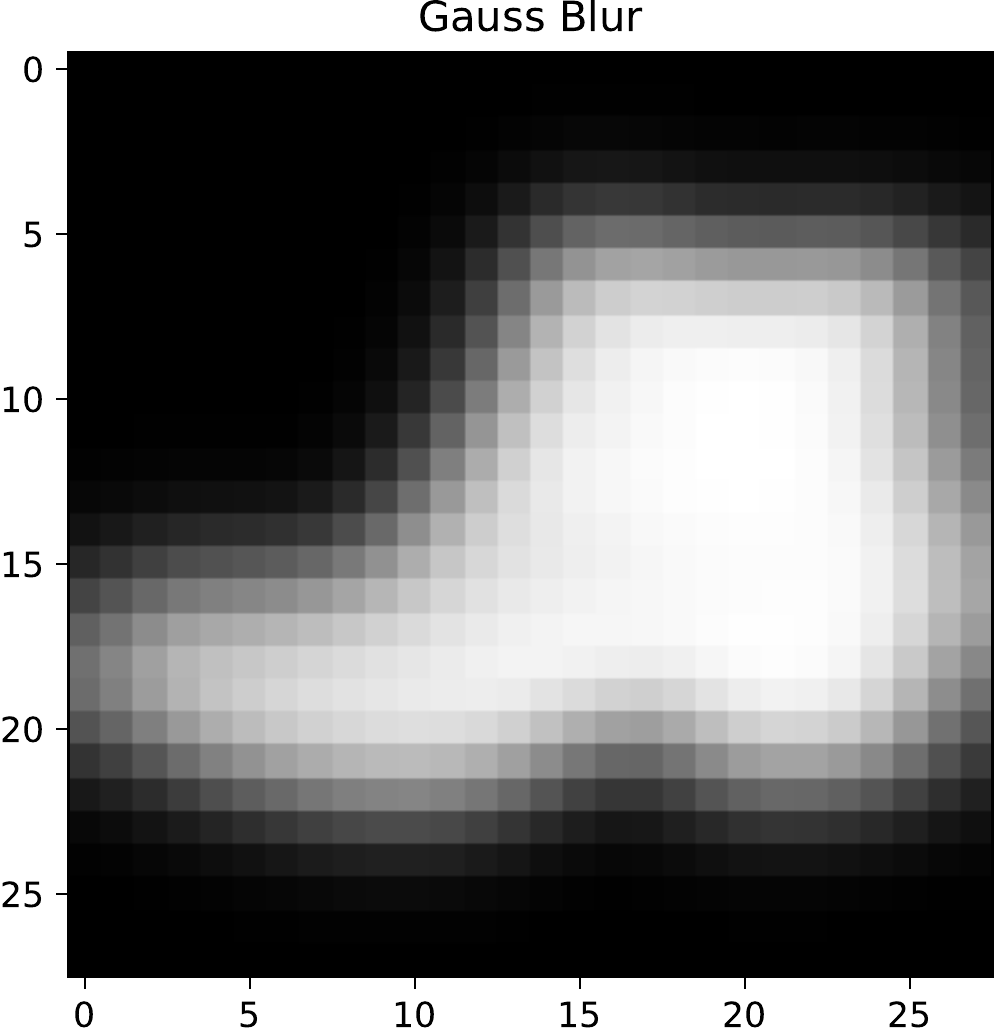}
     \caption{$l^2 \approx 4.0$}
   \label{fig:gb_6f} 
  \end{subfigure}  
  
\begin{subfigure}[]{0.1\textwidth}
    \includegraphics[width=\textwidth]{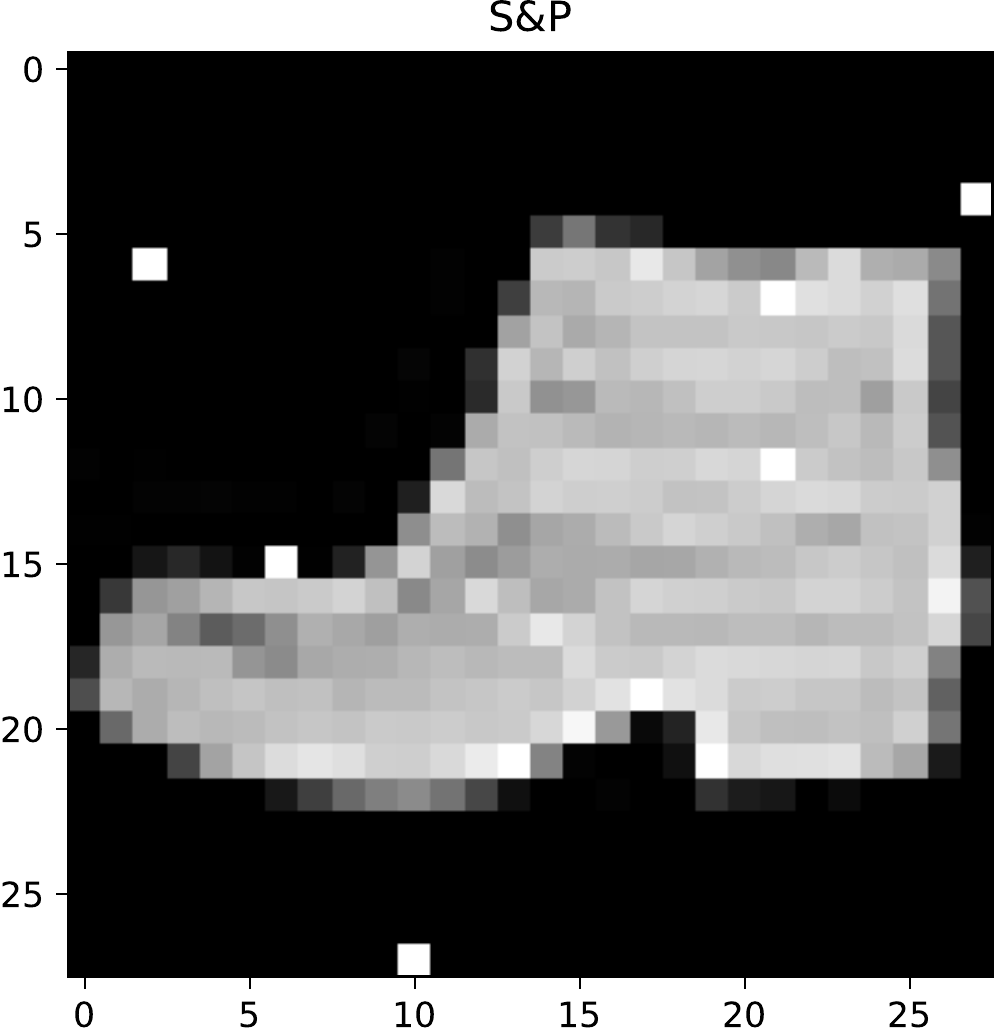}
     \caption{$l^2 \approx 2.0$}
   \label{fig:sp_10f} 
  \end{subfigure}
\hspace{0.1em}
  \begin{subfigure}[]{0.1\textwidth}
    \includegraphics[width=\textwidth]{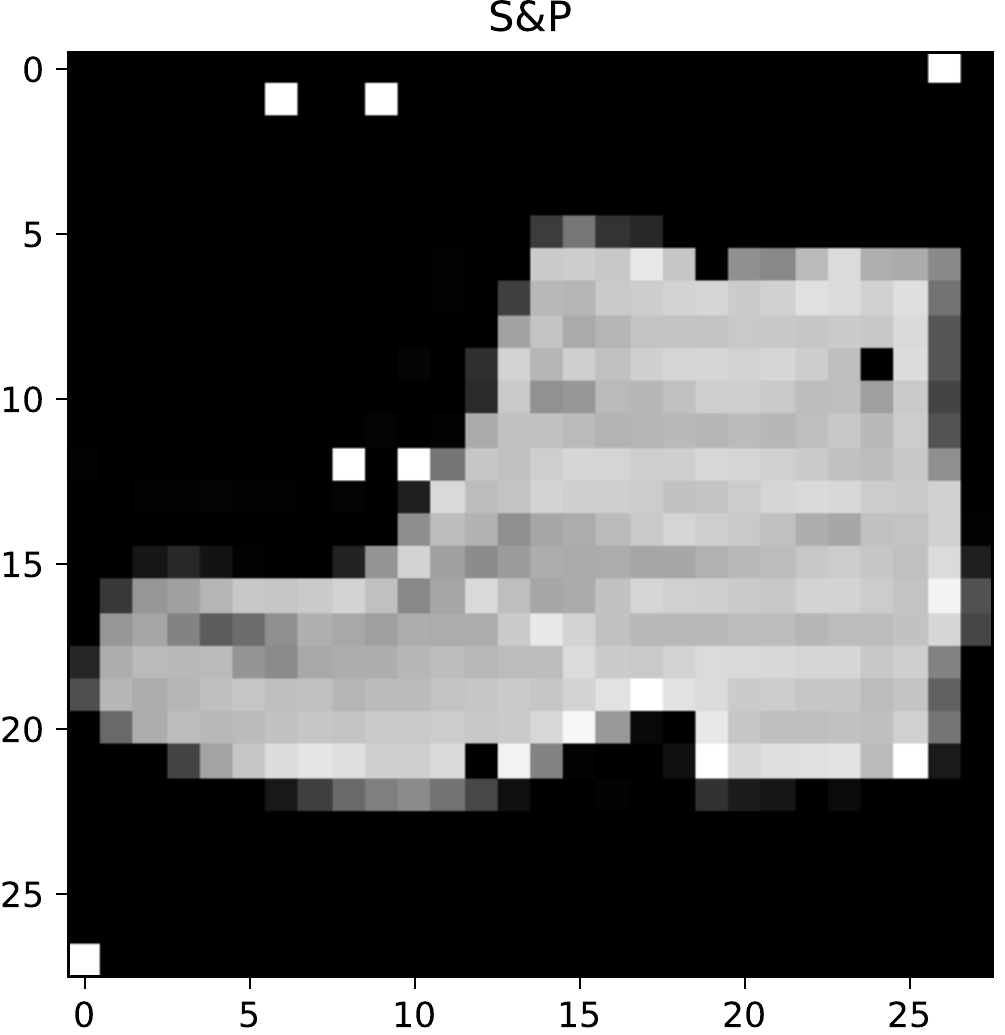}
     \caption{$l^2 \approx 2.8$}
   \label{fig:sp_30f} 
  \end{subfigure}  
\hspace{0.1em}
  \begin{subfigure}[]{0.1\textwidth}
    \includegraphics[width=\textwidth]{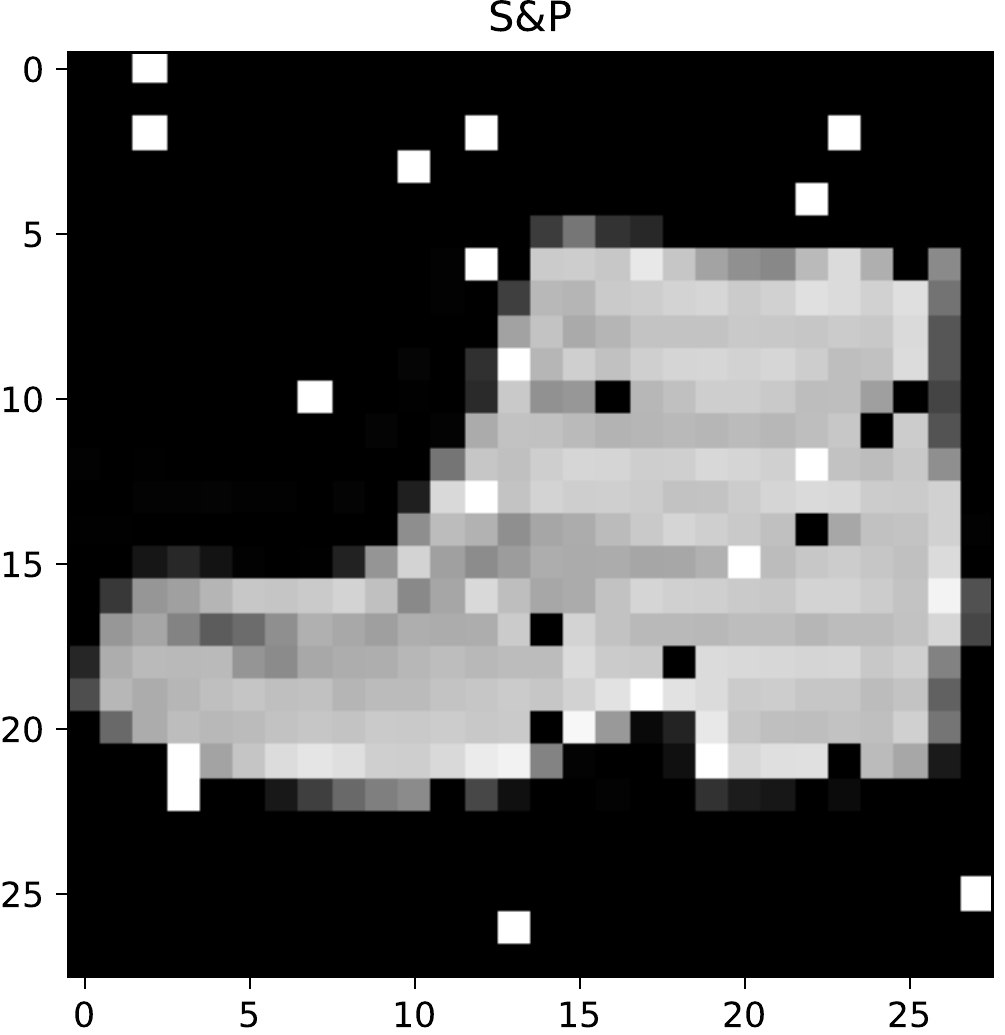}
     \caption{$l^2 \approx 4.2$}
   \label{fig:sp_60f} 
  \end{subfigure}  
\caption{Some examples of an image of a boot perturbed using the noise models.}
\label{fig:fmnistp}
\end{center}
\end{figure}

\begin{table}[!h]
  \vspace{-1em}
    \caption{The initial classification accuracies.  PCA, CAE, DAE, VAE are all Mapper based approaches, and only differ in the type of projection.  
    These values give the normalization in Equation \eqref{eq:accuracy}.} \label{tab:accuracies}
\begin{center}
{\scriptsize
 \begin{tabular}{| c | c | c |c|c|} 
    \hline
    &\multicolumn{4}{|c|}{{\bf Init. Accuracy}}\\\hline
 & {\bf $10$k MNIST} & {\bf $60$k MNIST} &
{\bf $10$k F.-MNIST} & {\bf $60$k F.-MNIST} \\\hline\hline
CNN & 97.00 & 98.51 & 89.35 & 93.4\\\hline
 \hline\hline
    PCA & 94.53 & 97.33 & 81.64 & 87.20 \\ 
    CAE & 93.61 & NA & 81.27 & NA \\
    DAE & 93.99  & NA & NA & NA \\
    VAE & 94.95 & 97.68 & NA & NA \\
  \hline
  \end{tabular}}
\end{center}
\end{table}

\begin{figure}[h]
\centering
  \begin{subfigure}[t]{0.24\textwidth}
  \centering
    \includegraphics[width=\textwidth]{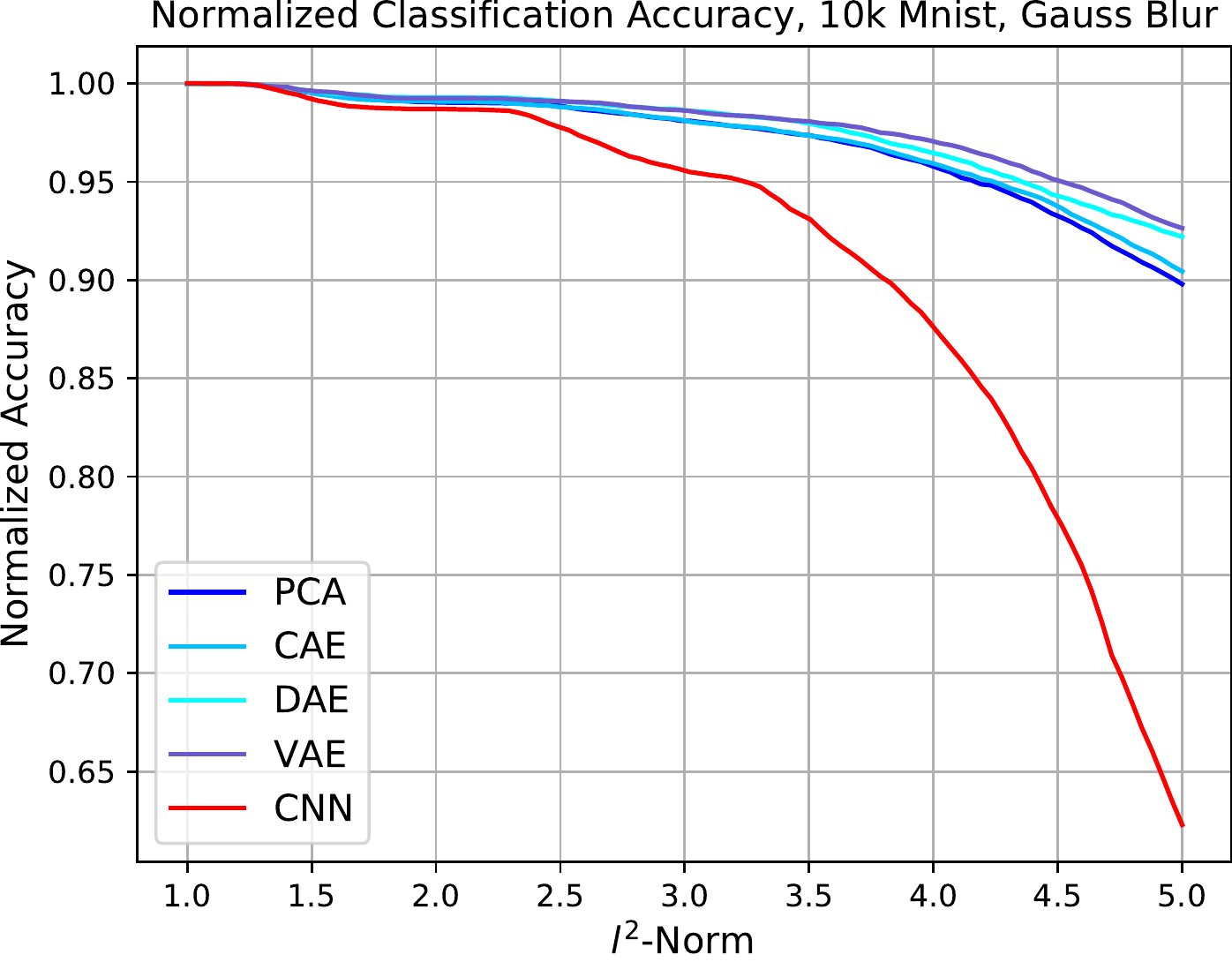}
   \label{fig:l2_mnist_gb_10k} 
  \end{subfigure}
  \begin{subfigure}[t]{0.24\textwidth}
    \includegraphics[width=\textwidth]{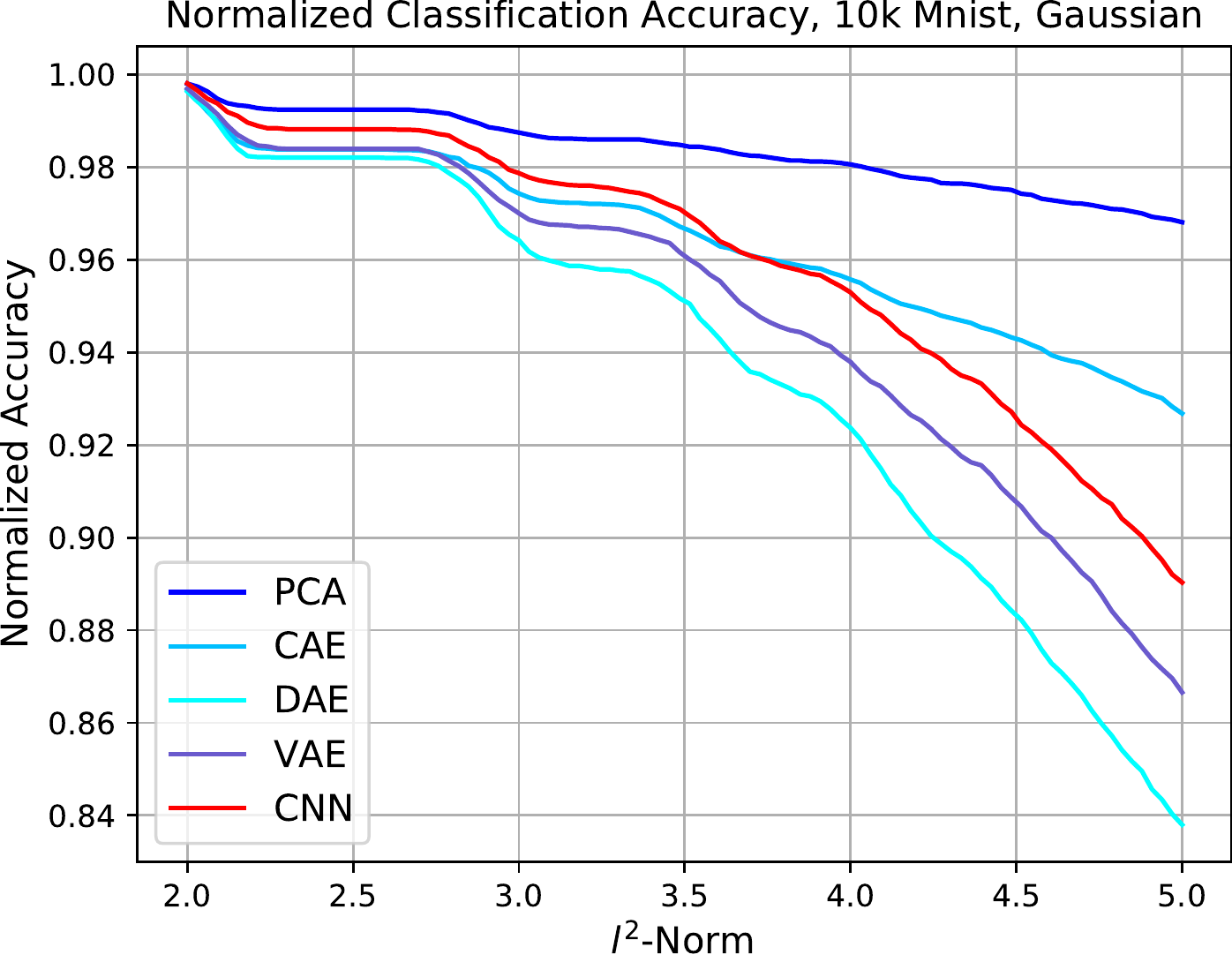}
   \label{fig:l2_mnist_ga_10k} 
  \end{subfigure}
    \begin{subfigure}[t]{0.24\textwidth}
    \includegraphics[width=\textwidth]{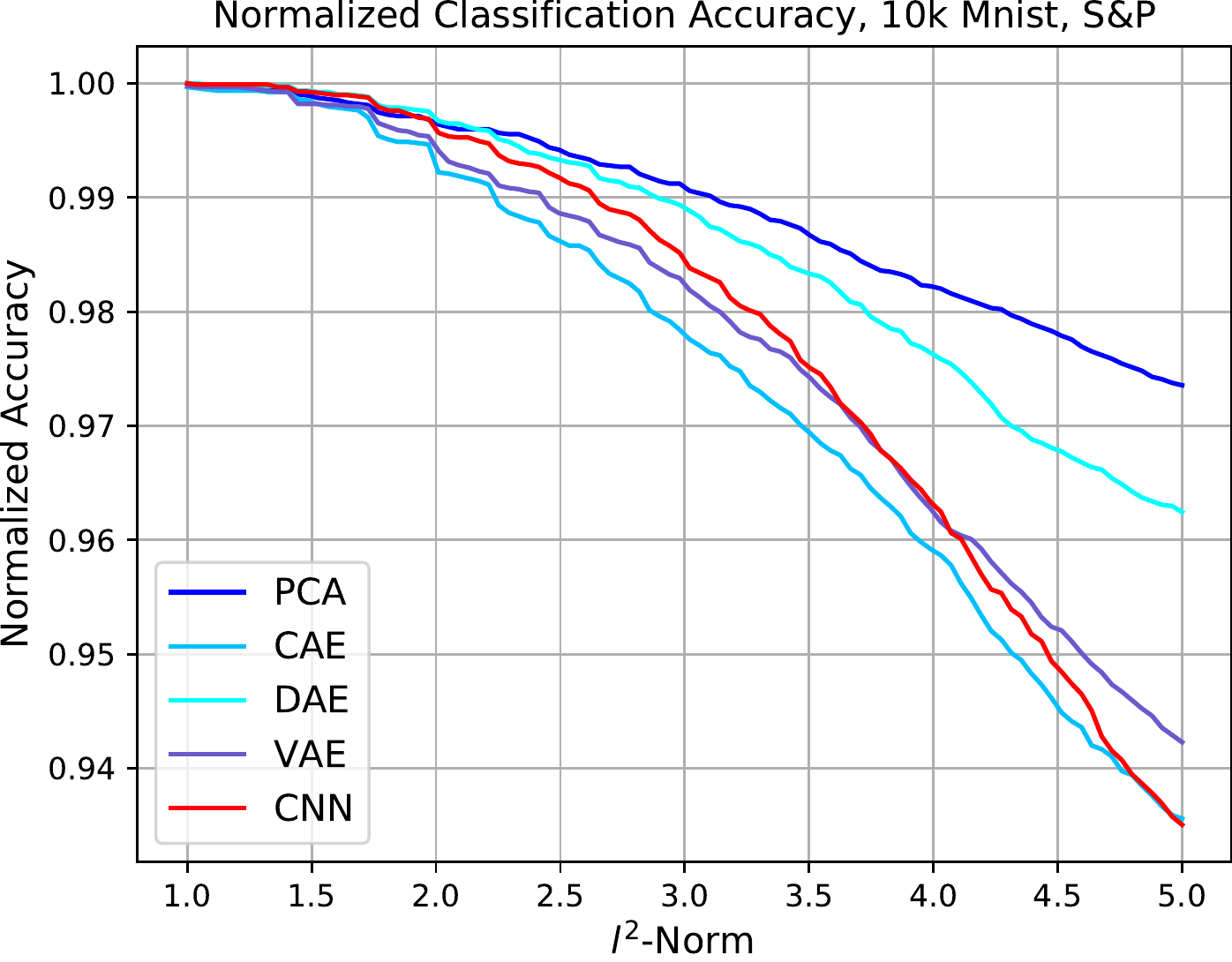}
   \label{fig:l2_mnist_sp_10k} 
  \end{subfigure}
\caption{The normalized accuracy (Eq. \eqref{eq:accuracy}) with respect to $l^2$-norm for $10$k MNIST.  The PCA based Mapper approach far outperforms the CNN approach for all the noise models.}
\label{fig:l2_mnist_10k}
\end{figure}

\begin{figure}[h]
\centering
  \begin{subfigure}[t]{0.24\textwidth}
    \includegraphics[width=\textwidth]{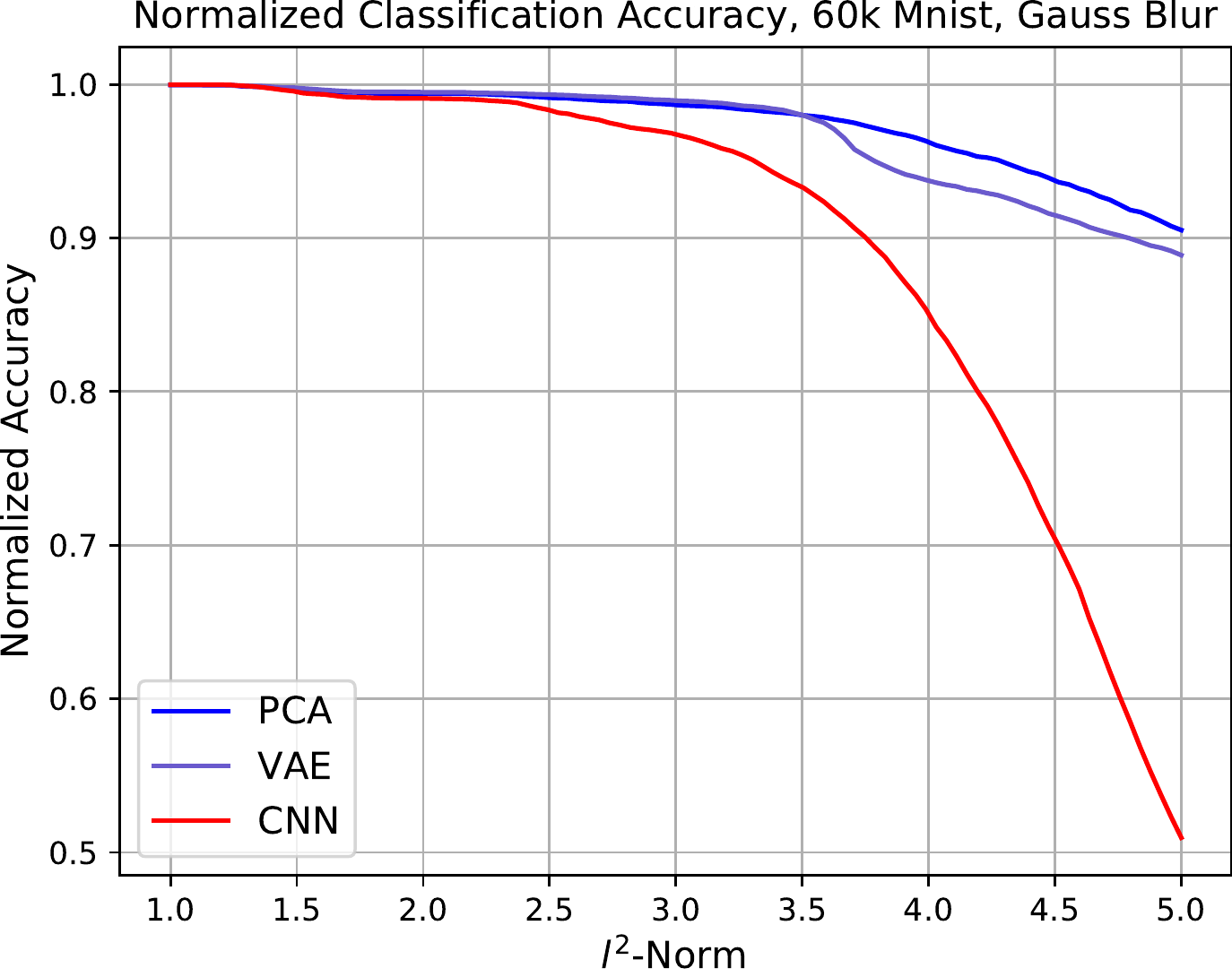}
   \label{fig:l2_mnist_60k} 
  \end{subfigure}
  \begin{subfigure}[t]{0.24\textwidth}
    \includegraphics[width=\textwidth]{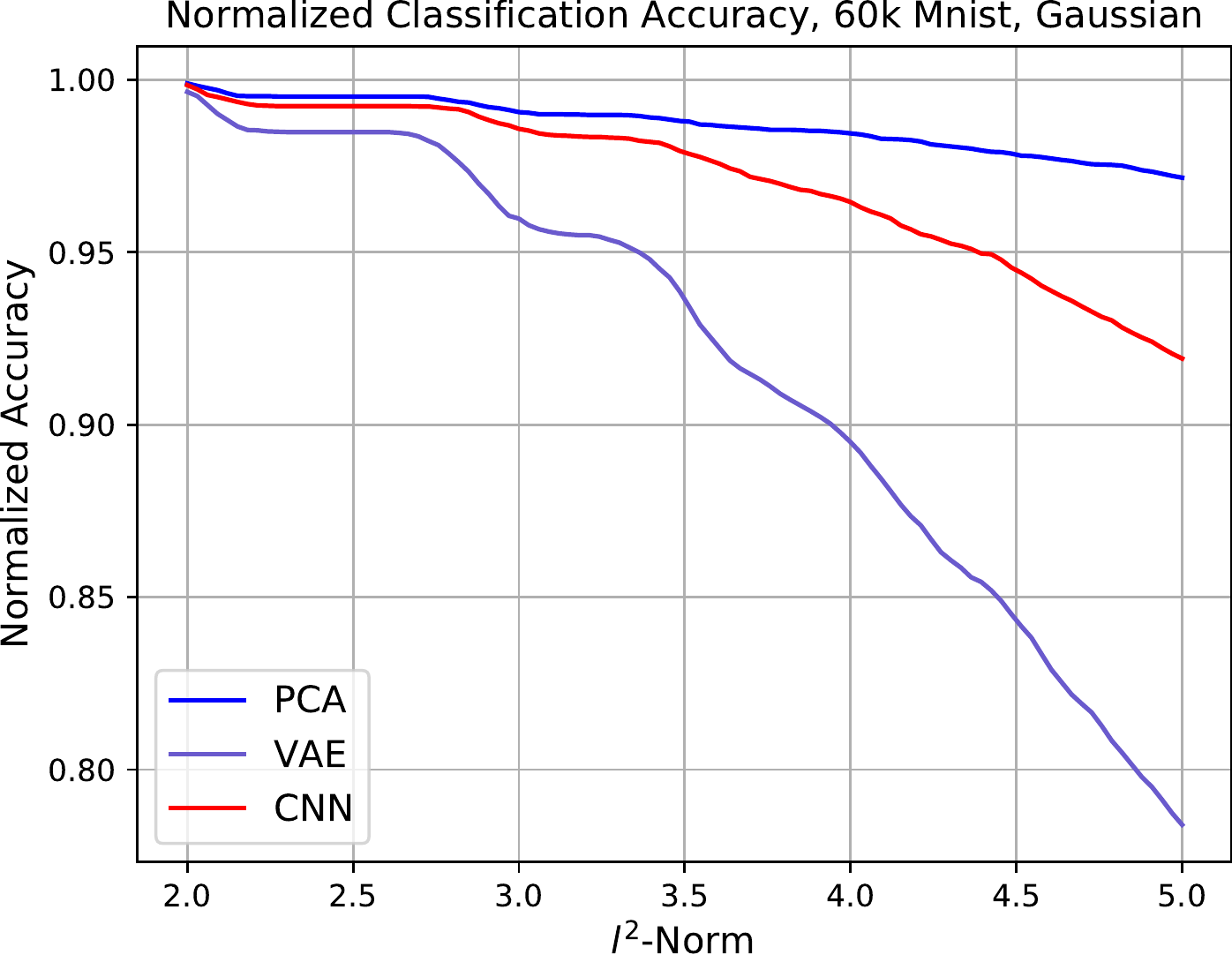}
   \label{fig:l2_mnist_ga_60k} 
  \end{subfigure}
  \begin{subfigure}[t]{0.24\textwidth}
    \includegraphics[width=\textwidth]{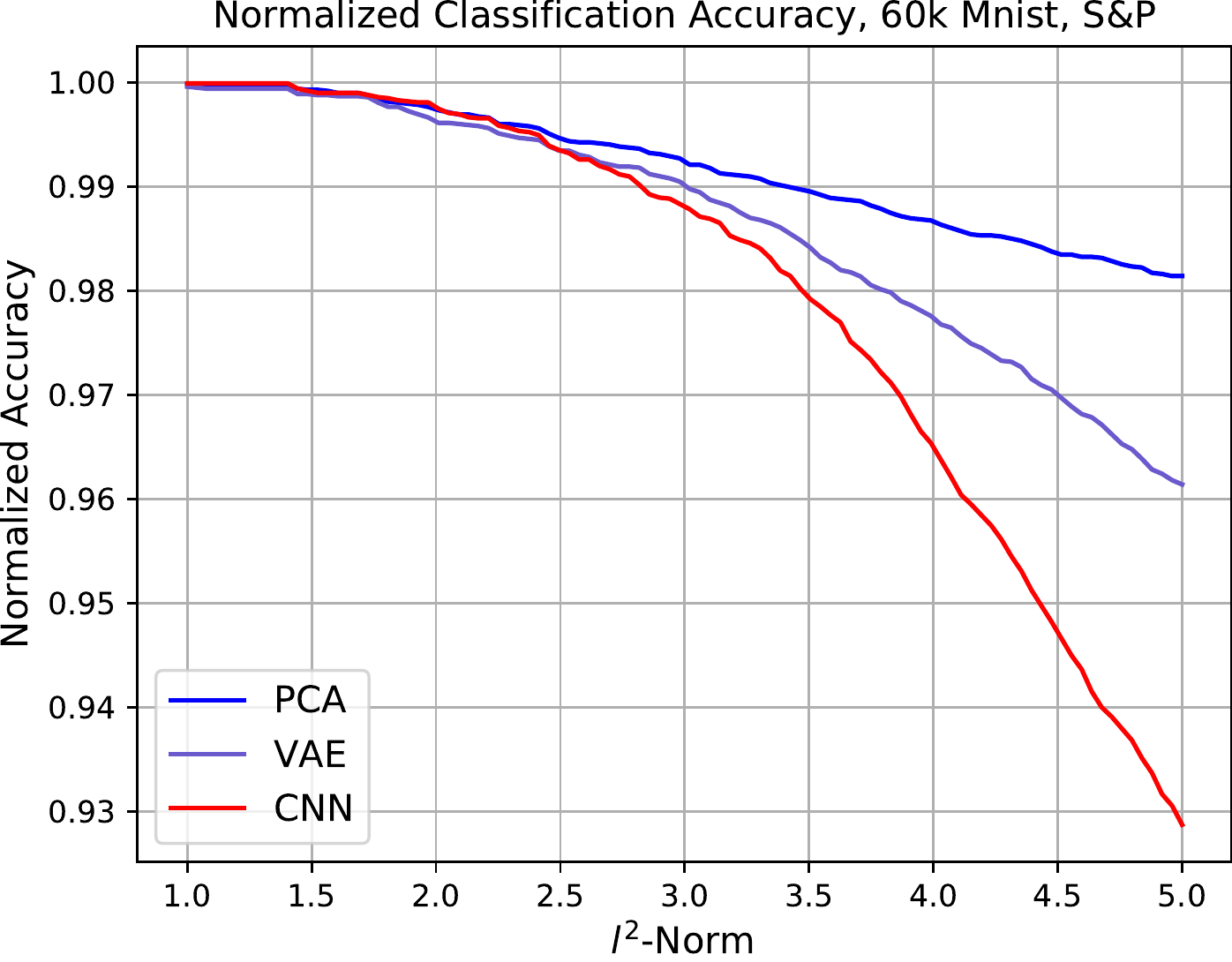}
   \label{fig:l2_mnist_sp_60k} 
  \end{subfigure}  
\caption{The normalized accuracy (Eq. \eqref{eq:accuracy}) with respect to $l^2$-norm for $60$k MNIST.  The PCA based Mapper approach far outperforms the CNN approach for all the noise models.  For $60$k MNIST, we choose to investigate just two Mapper based methods: PCA and VAE which seem to perform the best on average.}
\label{fig:l2_mnist_60k}
\end{figure}

\begin{figure}[h]
\centering
  \begin{subfigure}[t]{0.24\textwidth}
    \includegraphics[width=\textwidth]{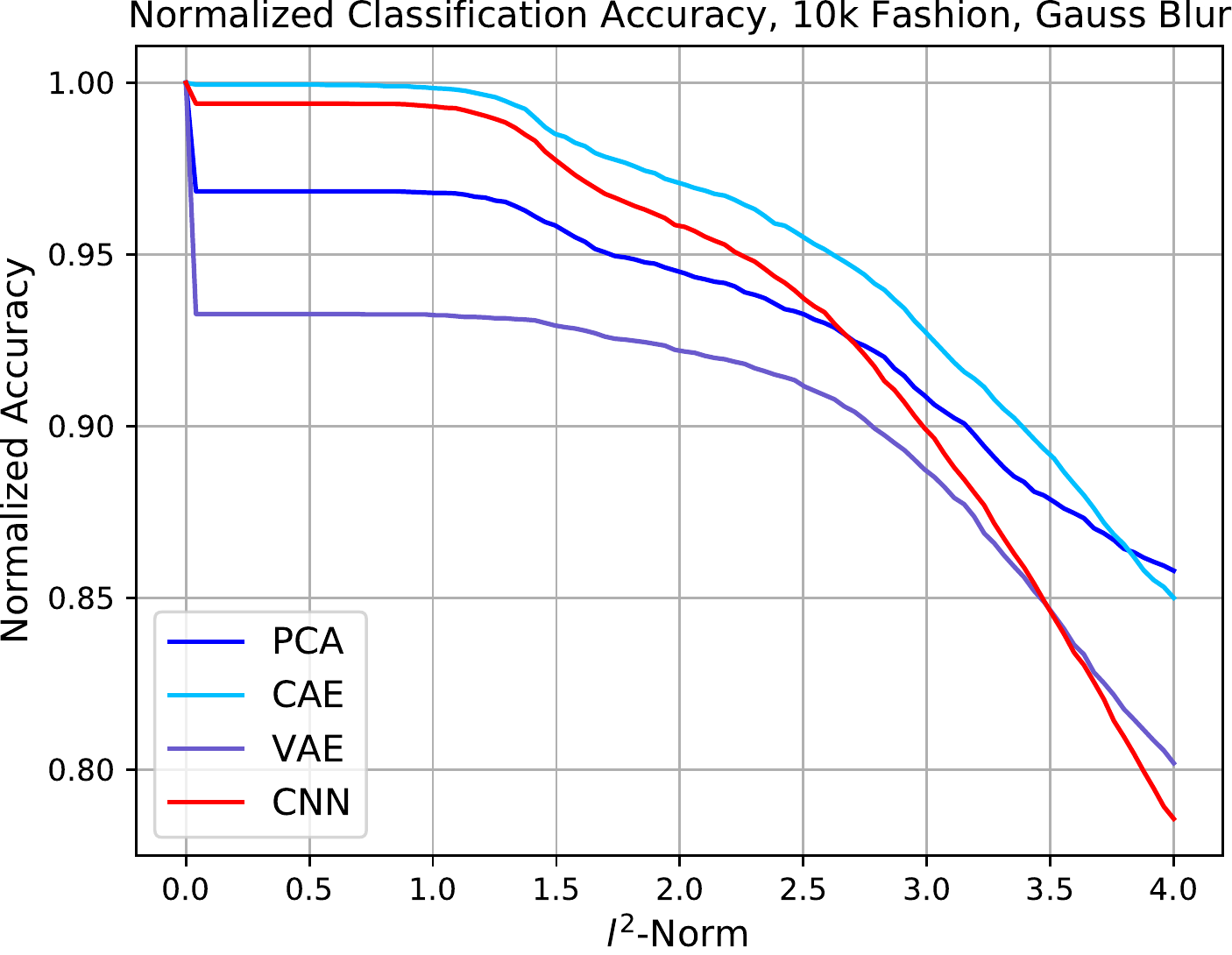}
   \label{fig:l2_fashion_gb_10k} 
  \end{subfigure}
  \begin{subfigure}[t]{0.24\textwidth}
    \includegraphics[width=\textwidth]{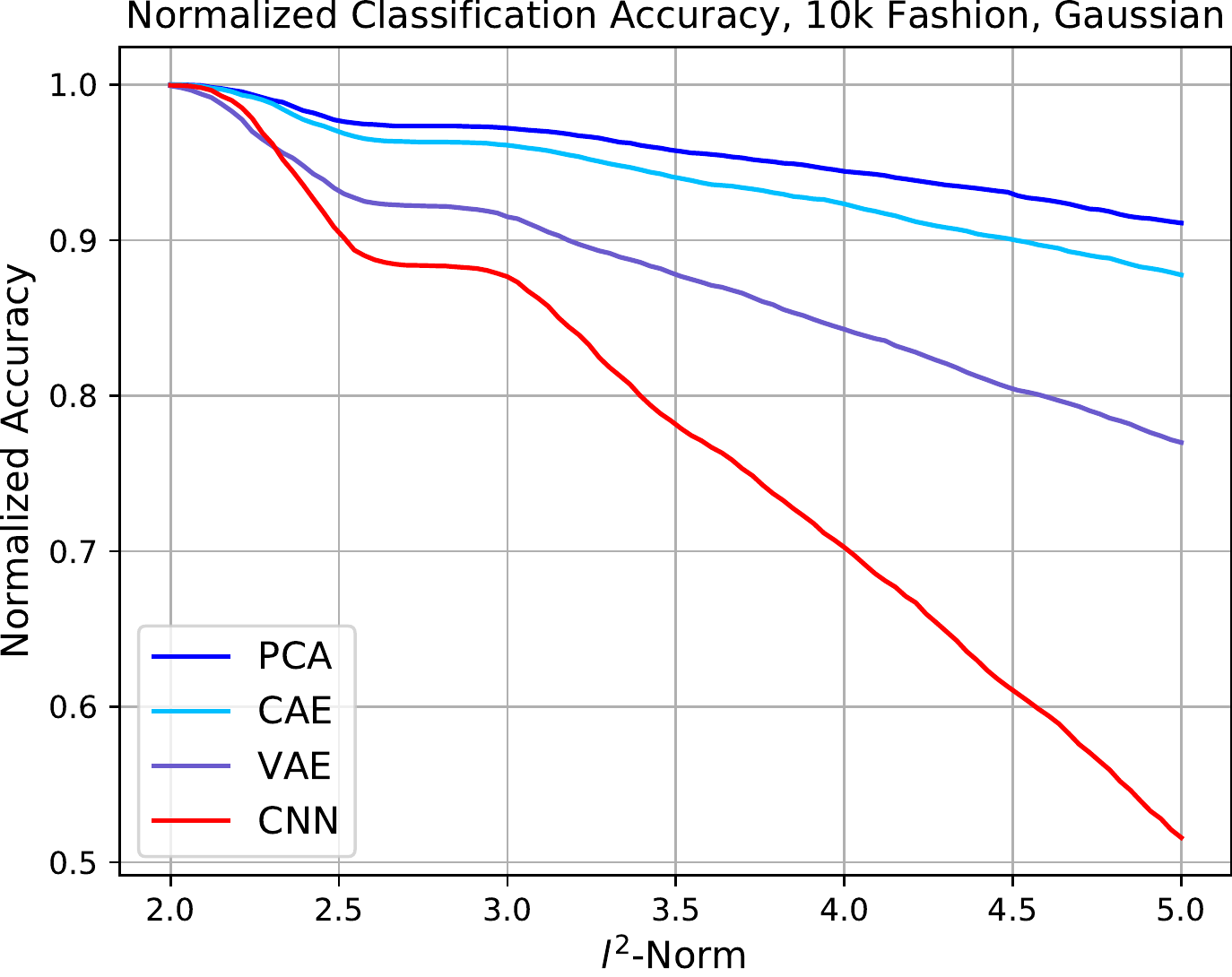}
   \label{fig:l2_fashion_ga_10k} 
  \end{subfigure}
  \begin{subfigure}[t]{0.24\textwidth}
    \includegraphics[width=\textwidth]{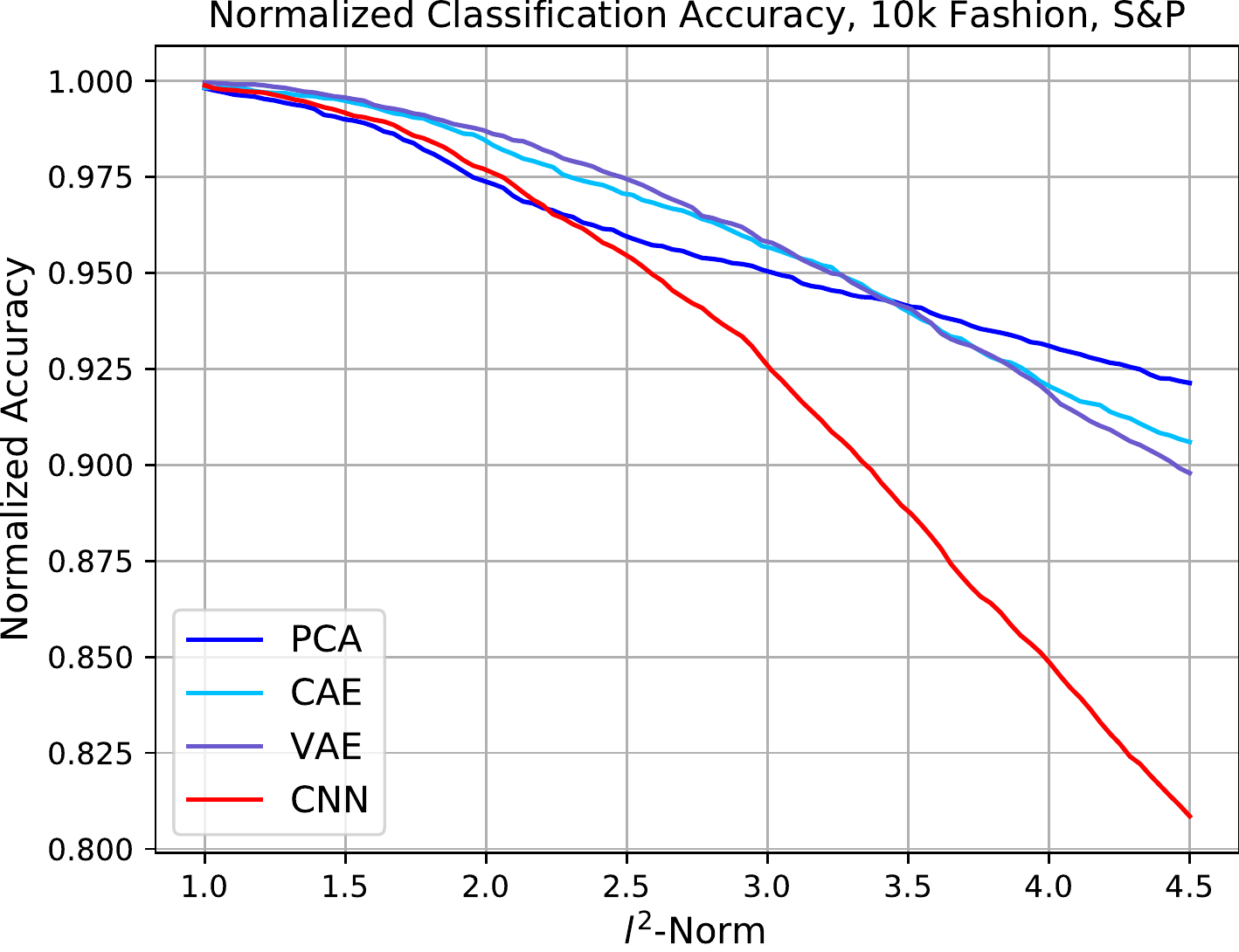}
   \label{fig:l2_fashion_sp_10k} 
  \end{subfigure}  
\vspace{-1.5em}
\caption{The normalized accuracy (Eq. \eqref{eq:accuracy}) with respect to $l^2$-norm for $10$k Fashion.  The PCA based Mapper approach far outperforms the CNN approach for all the noise models.  We choose to investigate PCA, CAE  and VAE based Mapper methods.}
\label{fig:l2_fashion_10k}
\end{figure}

\begin{figure}[h]
\centering
  \begin{subfigure}[t]{0.24\textwidth}
    \includegraphics[width=\textwidth]{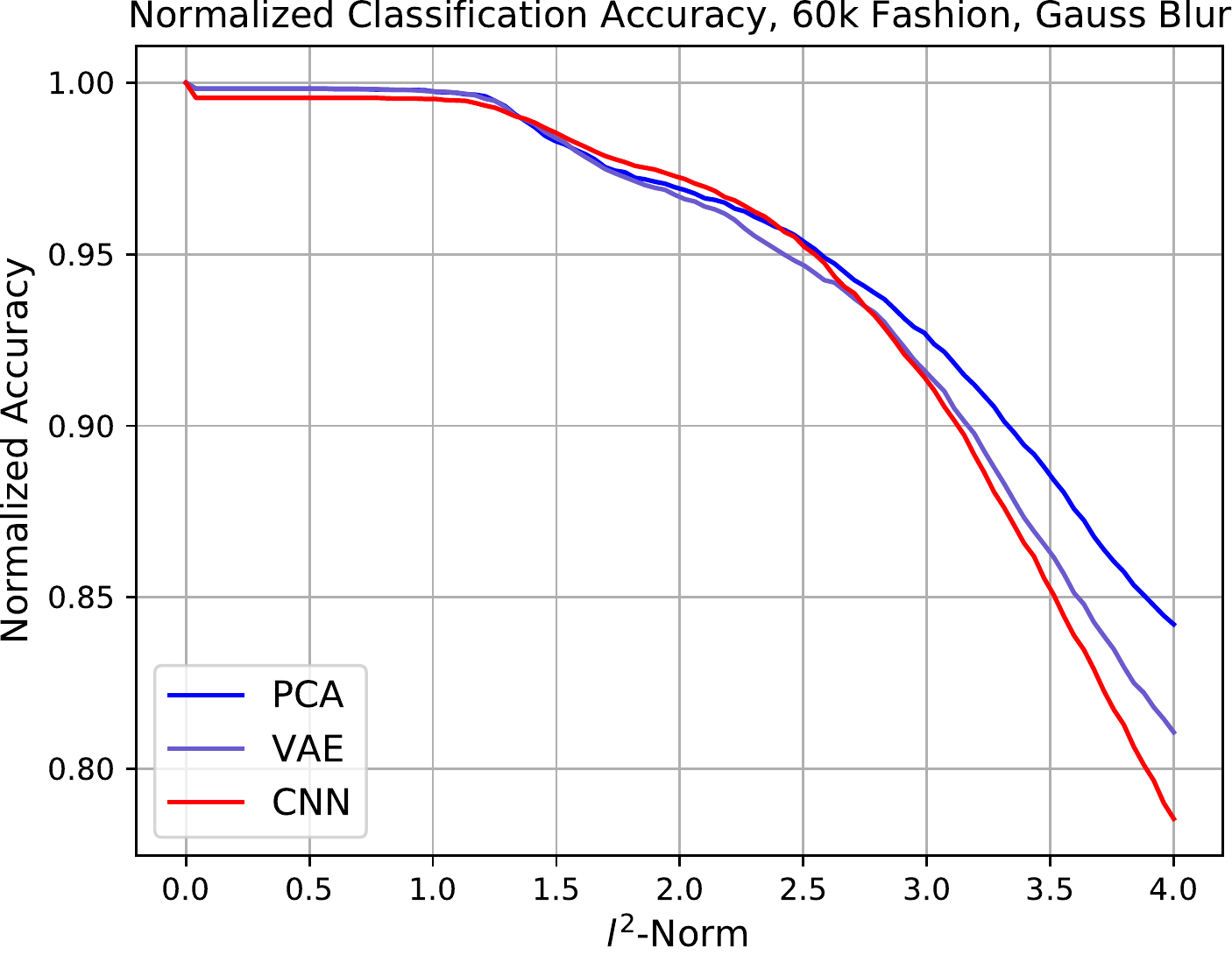}
   \label{fig:l2_fashion_gb_60k} 
  \end{subfigure}
  \begin{subfigure}[t]{0.24\textwidth}
    \includegraphics[width=\textwidth]{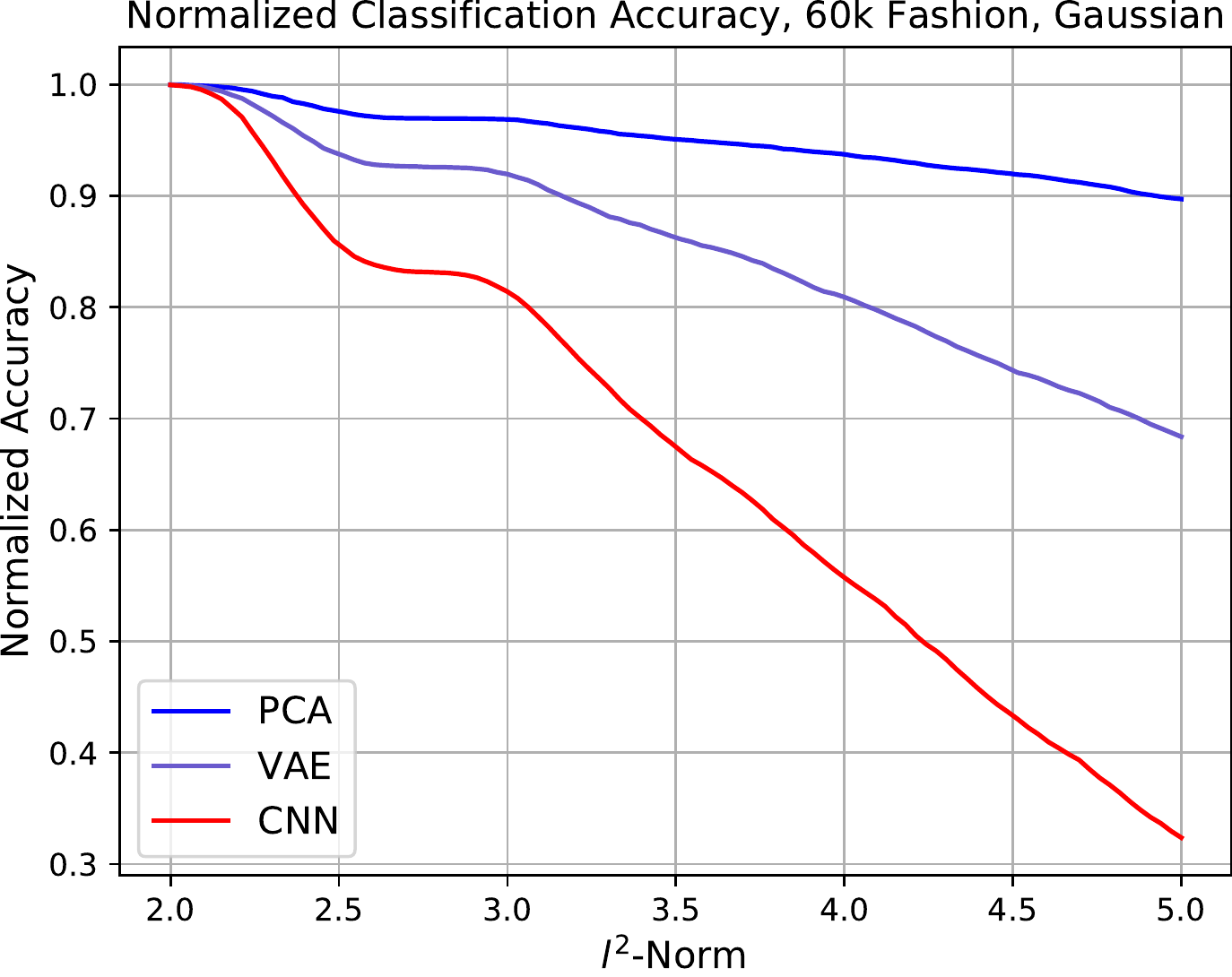}
   \label{fig:l2_fashion_ga_60k} 
  \end{subfigure}
  \begin{subfigure}[t]{0.24\textwidth}
    \includegraphics[width=\textwidth]{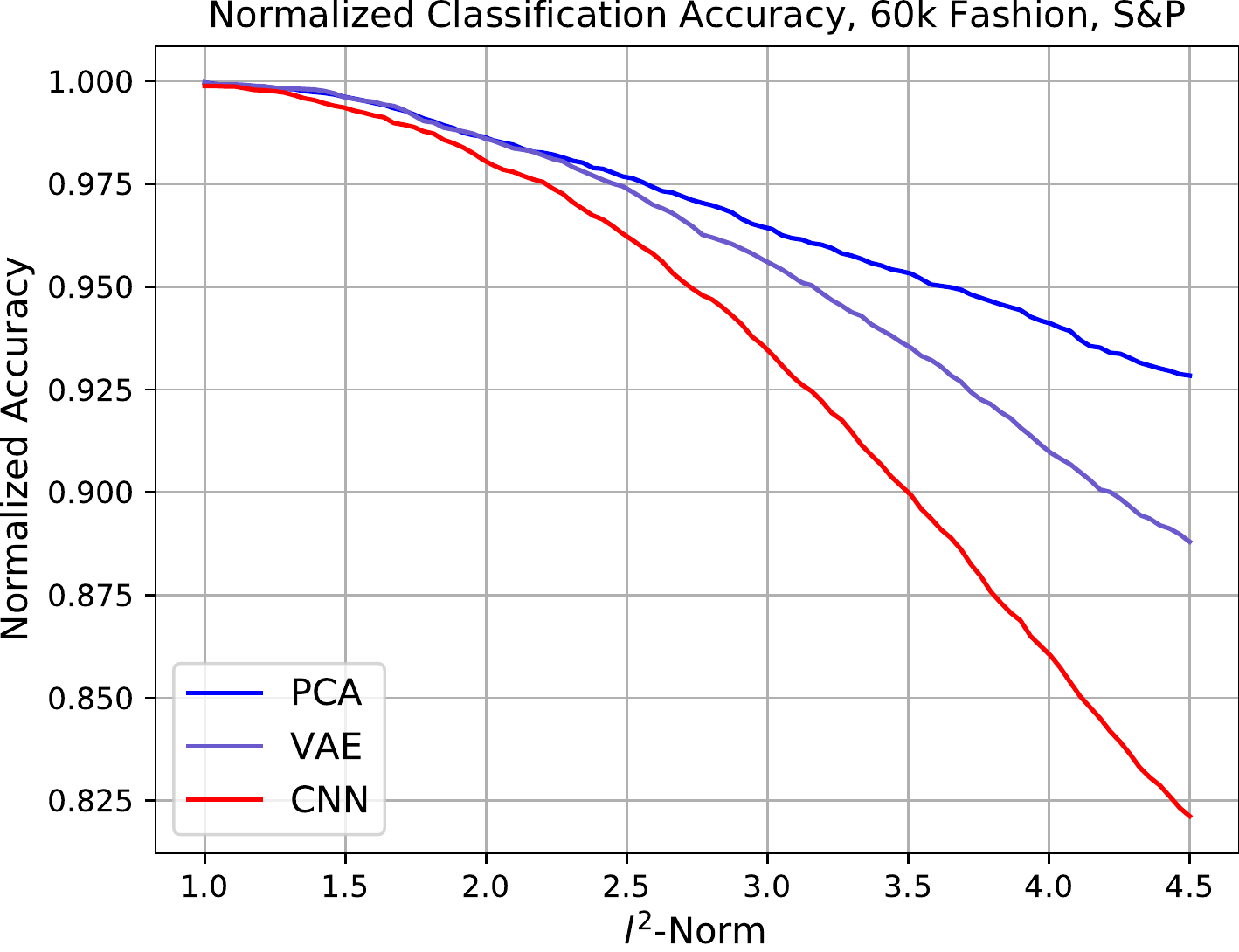}
   \label{fig:l2_fashion_sp_60k} 
  \end{subfigure}  
\vspace{-1.5em}
\caption{The normalized accuracy (Eq. \eqref{eq:accuracy}) with respect to $l^2$-norm for $60$k Fashion.  The PCA based Mapper approach far outperforms the CNN approach for all the noise models. We choose to investigate just two Mapper based methods: PCA and VAE which seem to perform the best on average.}
\label{fig:l2_fashion_60k}
\end{figure}

\subsection{Numerical Experiments Discussion}
We present some conclusions from the performed numerical experiments. Our MC method is in general more robust than the CNN, however MC achieves slightly less initial accuracy. We verify this claim by applying our method to two diverse datasets, MNIST -- hand written digits and Fashion-MNIST -- small shapes of piece of clothing. We believe that a more extensive hyperparameter optimization of our algorithm would result in a higher initial accuracy. Rather surprisingly, the nonlinear latent space generation using autoencoders performs, on average, worse than the linear PCA method. This is interesting especially because one of the applications of CAE \cite{CAE} and VAE \cite{VAE} methods is for adversarial defense by projecting perturbed data onto a small neighborhood in the latent space.  Using an autoencoder resulted in better robustness when compared to PCA only in the case of Gaussian blur for $10$k training (more visible in the case of Fashion-MNIST), but this difference is not large. When we extend the analysis to $60$k, we see PCA being the clear winner.  While we currently do not have a precise answer as to why PCA does so well overall, we expect that the nonlinear methods may be overfitting to the mathematical structure and hence the overall robustness is negatively affected when using them.  For $60k$ training some results are missing (all encoders except VAE for MNIST and all encoders for Fashion-MNIST) due to the time constraints and limitations of our prototype implementation.

Although our MC method performs particularly well for all noise models, it far outperforms the CNN for the case of  Gaussian blur noise (MNIST), and Gaussian noise (Fashion-MNIST). It remains unclear Why MC outperforms CNN by a large margin in the particular case of  the Gaussian blur (a local noise) for MNIST and the Gaussian noise (a global noise) for Fashion-MNIST. And we find this an interesting research problem and will investigate the case for other datasets.

Also, we do not go beyond $l^2$ perturbation norm around $5$ as this range is a more typical ``adversarial'' range. Flat regions that appear in Figs. \ref{fig:l2_mnist_10k} through \ref{fig:l2_fashion_60k} turn out to be an artifact of our sampling procedure, where we sample noise perturbations scaled uniformly by a ``lambda'' parameter (see the description in App.~\ref{secnoisemodel}). This procedure is not equivalent to sampling among $l^2$-norm perturbations (again, see App.~\ref{secnoisemodel}).
\section{Mathematical Intuition about the Robustness}\label{sec:intuition}

We present a Proposition that formalizes the intuition that our method should be robust with respect to small perturbations of input images.
\begin{figure}[h!]
	\centering
	\begin{subfigure}{0.22\textwidth} 
	    \label{fig:beta_0}
		\includegraphics[width=\textwidth]{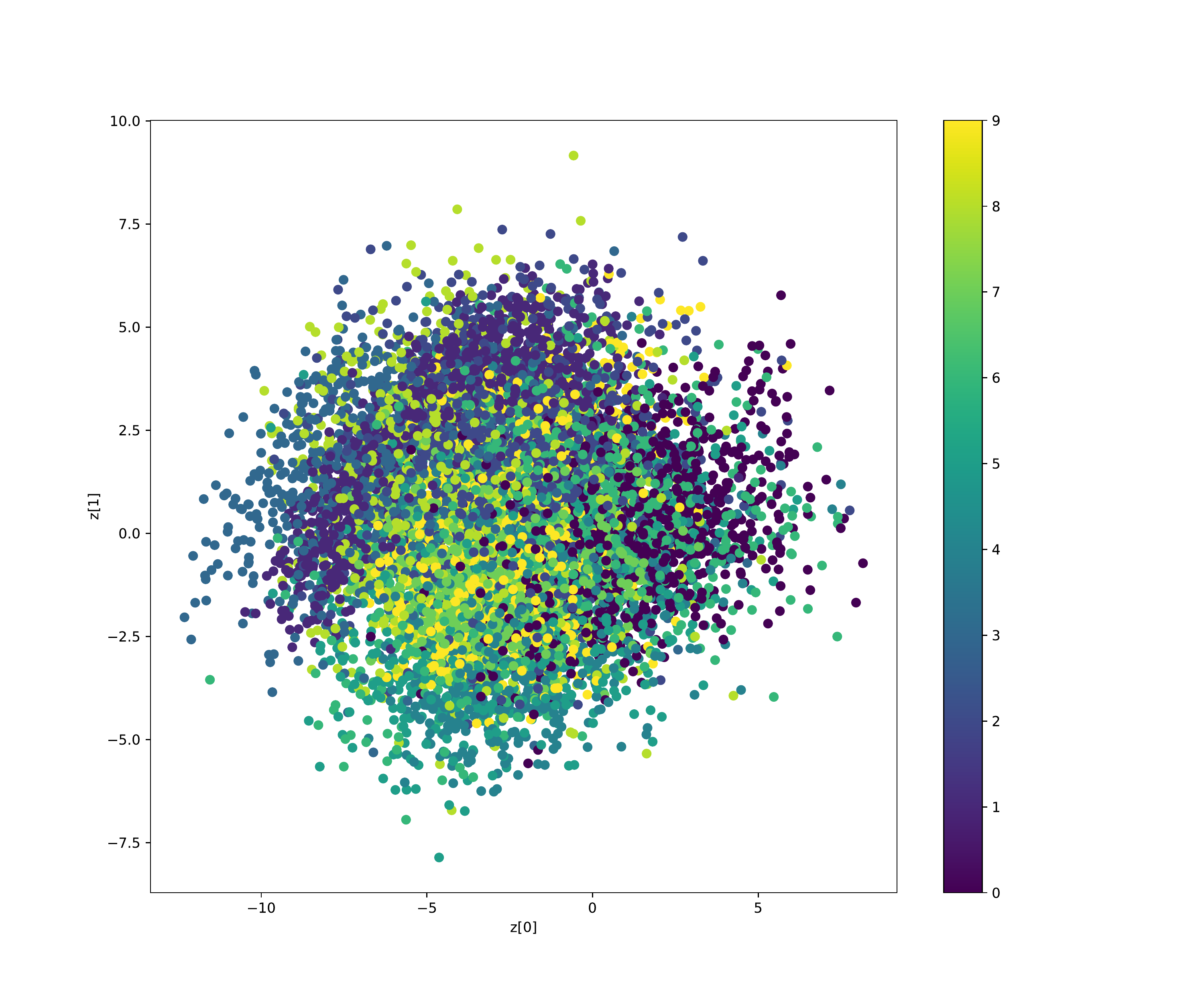}
		\caption{$\beta=0$} 
	\end{subfigure}
	\vspace{1em} 
	\begin{subfigure}{0.22\textwidth}
	    \label{fig:beta_1}
		\includegraphics[width=\textwidth]{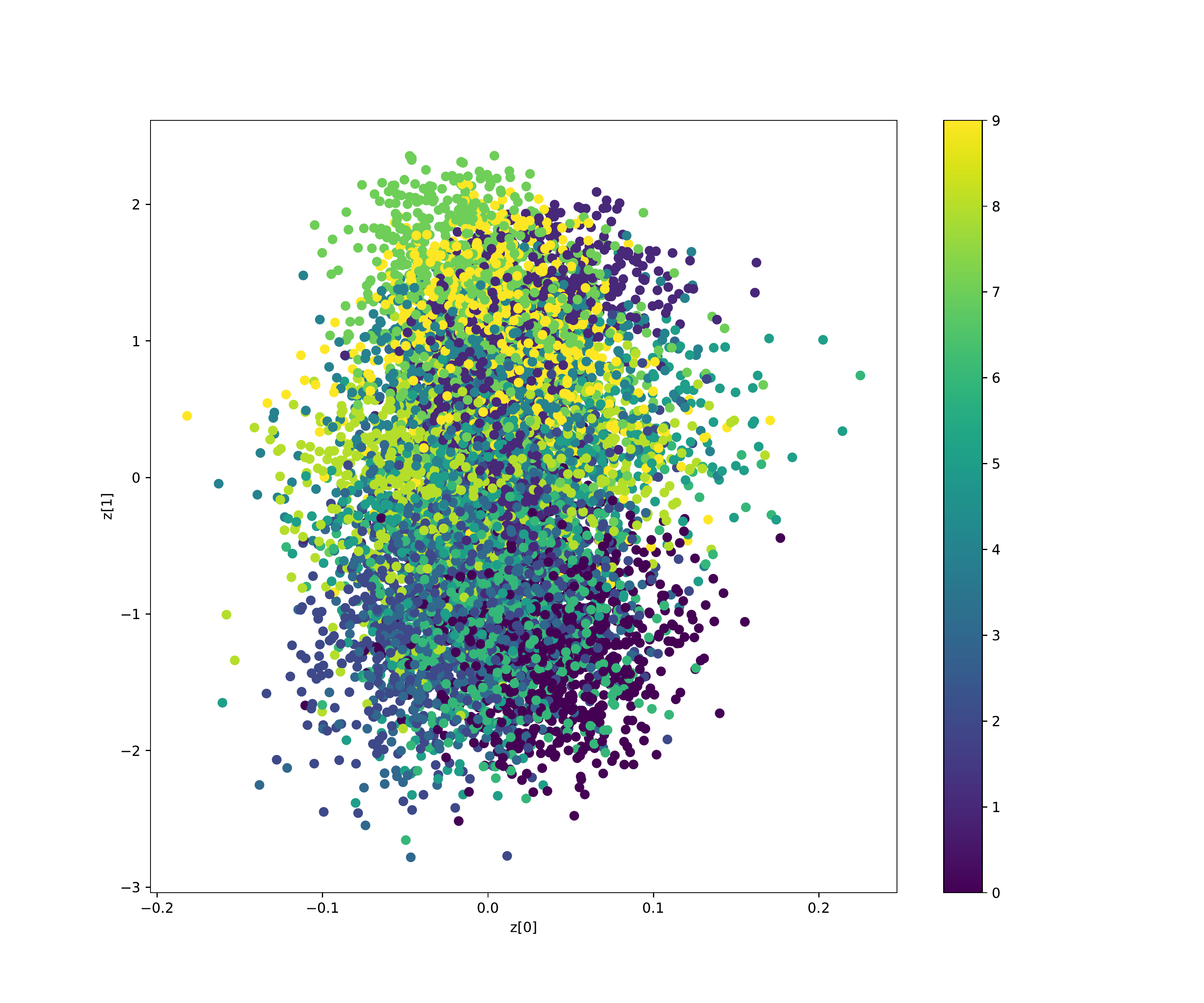}
		\caption{$\beta=1$} 
	\end{subfigure}
	\caption{Latent space representations of two nodes in the compressed layer of the VAE (i.e. a $2$-dimensional subspace of the projection to $20$-dimensions).  We use a $\beta$-term that multiplies the KL divergence.  $\beta=0$ yields the best overall robustness.} 
\label{fig:latentspace}
\end{figure}
Intuitively, the presented proposition states that for any training point $x\in X_{\rm train}$ it holds that a slight perturbation of the point within range $\varepsilon$ will also satisfy $\|g_M(x) - g'_M(x')\|_{l^1} \sim \varepsilon$ (dependence is linear with a small constant). This in turn implies that small changes in inputs will transfer onto slight changes in the $g_M$ map output. Eventually, the $g_M$ map outputs are fed into a neural-network based classifier, hence, the eventual robustness is dependent on the precise properties of the employed classifier and how it interacts with the $g$ map. We denote the $k$ nearest neighbors of $x'$ by $nn_1(x'), nn_2(x'),\dots, nn_k(x')$.

\begin{prop}\label{prop:arch1}
Let $\{U_\alpha\}$ be the open interval cover of $I$. Let $f\colon\mathcal{X}\to\mathbb{R}$ be a Mapper filter function; let $0< \delta< 1$ be a parameter of the method (in the actual algorithm $\delta=0.2$), and let $k$ be the number of nearest neighbors considered in the algorithm. $d(\cdot,\cdot)$ is the $l^2$ distance. Let $h_k(x') = \sum_{l=1}^k d(nn_l(x'),x')^{-1}$. To simplify the notation, we denote below $X = X_{\rm train}$. 

Let $x\in X$ be perturbed to $x^\prime\in\cX$ such that $\|x-x'\|_2\leq\varepsilon$ for some small $\varepsilon\ll\mathbb{E}_{x\in X}{h_{k-1}(x)}$. Let $g'_M(x')$ be computed using the algorithm described in Sec.~\ref{secmapping}.
If  for all intervals $U_\alpha\ni f(x)$, $\left|f(x') - \mbox{mid}(U_\alpha)\right| < \frac{\mbox{max}(U_\alpha)-\mbox{min}(U_\alpha)}{2}(1+\delta)$,  then 
\[
\mathbb{E}_{x\in X}{\left|g'_M(x') - g_M(x)\right|_{l_1}}\leq \varepsilon\cdot C(k, \mathbb{E}_{x\in X}{h_{k-1}(x)}) + O(\varepsilon^2),
\]
in particular  $\mathbb{E}_{x\in X}{\left|g'_M(x') - g_M(x)\right|_{l_1}}\to 0$ as $\varepsilon\to 0$. 
\end{prop}
\begin{proof}
First, the condition $\left|f(x') - \mbox{mid}(U_\alpha)\right| < \frac{\mbox{max}(U_\alpha)-\mbox{min}(U_\alpha)}{2}(1+\delta)$ guarantees that the perturbed point $x'$ is mapped by the filter $f$ into the same intervals in the open cover as the original point $x$ (see the corresponding description in the algorithm).

Observe that the assumption guarantees that the nearest neighbor is $x$. That is, $nn_1(x') = x$, and hence $ w_1 = \frac{1}{\varepsilon\sum_{l=1}^{k}{d(nn_l(x'),x')^{-1}}}$. Therefore, we obtain
\begin{multline*}
g'_M(x') = \frac{1}{\varepsilon\sum_{l=1}^{k}{d_l^{-1}}} g_{M}(x) + w_2g_{M}(nn_2(x')) +\\\dots + w_kg_{M}(nn_k(x')),
\end{multline*}
and for $j=2,\dots,k$: $w_j = \frac{1}{d_j\sum_{l=1}^{k}{d_l^{-1}}}$,
where we denote $d_j = d(nn_j(x'),x')$. It holds that $g_M(x)=(1,0,\dots, 0)$ or $g_M(x)=(1,1,0,\dots, 0)$, as $nn_1(x) = x$. Also we have
\[
h_k(x') = \varepsilon^{-1} + \sum_{l=2}^k d(nn_l(x'),x')^{-1} = \varepsilon^{-1} + h_{k-1}(x) + O(\varepsilon).
\]
Also, 
\begin{multline*}
w_l = \bigl(d(nn_l(x'),x')h_k(x')\bigr)^{-1} = 1/\bigl(d_l \varepsilon^{-1} + d_l h_{k-1}(x) + d_l O(\varepsilon)\bigr) \\
\approx\varepsilon/d_l
 = \varepsilon/\bigl(d(nn_{l-1}(x),x)+O(\varepsilon)\bigr)
 =  \varepsilon/d(nn_{l-1}(x),x) + O(\varepsilon^2).
\end{multline*}
Thus, by the triangle inequality,
\begin{multline*}
\|g'_M(x')-g_M(x)\|_{l^1} \le \frac{|1-\varepsilon h_k(x')||g_M(x)|}{\varepsilon h_k(x')} +\\
\varepsilon\sum_{l=2}^k \frac{\|g_M(nn_{l-1}(x))\|_{l^1}}{d(nn_{l-1}(x),x)} + O(\varepsilon^2)
\leq\\\varepsilon\Bigl(2 h_{k-1}(x) + \sum_{l=2}^k \frac{\|g_M(nn_{l-1}(x))\|_{l^1}}{d(nn_{l-1}(x),x)}\Bigr)\\+ O(\varepsilon^2)\le 4\varepsilon h_{k-1}(x) + O(\varepsilon^2).
\end{multline*}
In the last inequality above we used the bound $\|g_M(nn_{l-1}(x))\|_{l^1}\leq 2$, as $g_M(nn_{l-1}(x))$ have either one or two nonzero entries equal to $1$ ($x$ is in one or two nodes of the mapper $M$). Finally, taking the expectation we have that 
\[
\mathbb{E}_{x\in X}{\left|g'_M(x') - g_M(x)\right|_{l_1}}\leq 4\varepsilon\mathbb{E}_{x\in X}h_{k-1}(x) + O(\varepsilon^2),
\]
and obviously as $\varepsilon\to 0$, then $\mathbb{E}_{x\in X}{\left|g'_M(x') - g_M(x)\right|_{l_1}}\to 0$.
\end{proof}
We computed in practice an estimate for the constant $C(k, \mathbb{E}_{x\in X}{h_{k-1}(x)})$ appearing in Prop.~\ref{prop:arch1}, taking $k=6$ (the value used in practice), the empirical expectation of the pairwise distance between points is equal $\mathbb{E}_{x\in X}{h_{k-1}(x)} \approx 0.5$, we obtain that $C(k, \mathbb{E}_{x\in X}{h_{k-1}(x)}) \approx 1.5$.
\section{Conclusion}
We have developed an algorithm which performs classification that is both robust and also highly accurate, resolving, to an extent, the bias-variance trade-off present in many machine learning methods.  Although we apply our MC to the task of improving robustness of image classifiers, we expect the algorithm to lend itself well to many other classification tasks in general.

There are various avenues for future work pertaining to this research.  One such avenue is to perform a more extensive hyperparameter search.  Many of these were set heuristically, or scanned over slightly, but with little scientific approach on converging to optimal values.  This is partly due to time considerations in our algorithm - in order to accomplish this search, we will need to construct our software to be more scalable.  Another avenue will be to understand why PCA does so well versus the nonlinear projections in our MC.  We expect that the MC is less prone to overfitting when using PCA, but this should be verified.  A related path to explore is to determine how the MC methods interact with specific noise models. MC vs CNN behavior should be investigated also for other datasets, especially those composed out of color images and other types of data.
%
%
%
%
%
%
%
%
%
%
%
%
%
%
%
%
%
%
\bibliographystyle{IEEEtran}
\bibliography{IEEEabrv,conference_041818}

\newpage\ \newpage
\appendix

\begin{center}
\Large{Appendix for the paper entitled\\
\emph{Mapper Based Classifier}\\
submitted to NeurIPS 2019}    
\end{center}

\section{Proof of the mathematical statement from the Sec.~4 in the paper}
\label{secproof}
Here we present a proof of our result providing a mathematical intuition about the robustness of MC.

\section{Splitting the Dataset}\label{sec:split}
As presented in Sec.~\ref{sec:train} for the purpose of training and testing we operate on the split training dataset
\[
X_{train} = X_{train}^1\cup X_{train}^2\cup \dots\cup X_{train}^n.
\]
There are two main advantages of splitting the training dataset into subsets. First, it provides a natural way of parallelizing computations during the training/testing phases.  This distributed computation procedure is amenable for modern architectures.  Computing several small Mapper outputs instead of a single large one allows for an easy model by distributing computation among the processor cores. Second, it improves the overall accuracy of the classifier, as illustrated by our results using the MNIST dataset presented in Fig.~\ref{fig:split} in the appendix available online. The comparison is done using different splittings of a $30$k subset of the MNIST training set. Due to computational and memory complexity of the Mapper algorithm, a $30$k training dataset is the limit of what we were able to compute using a PC machine having $32$gb memory.  The run time for $30$k was several hours and the memory was fully utilized.

\begin{figure}[h!]
\begin{center}
   \includegraphics[width=0.7\linewidth]{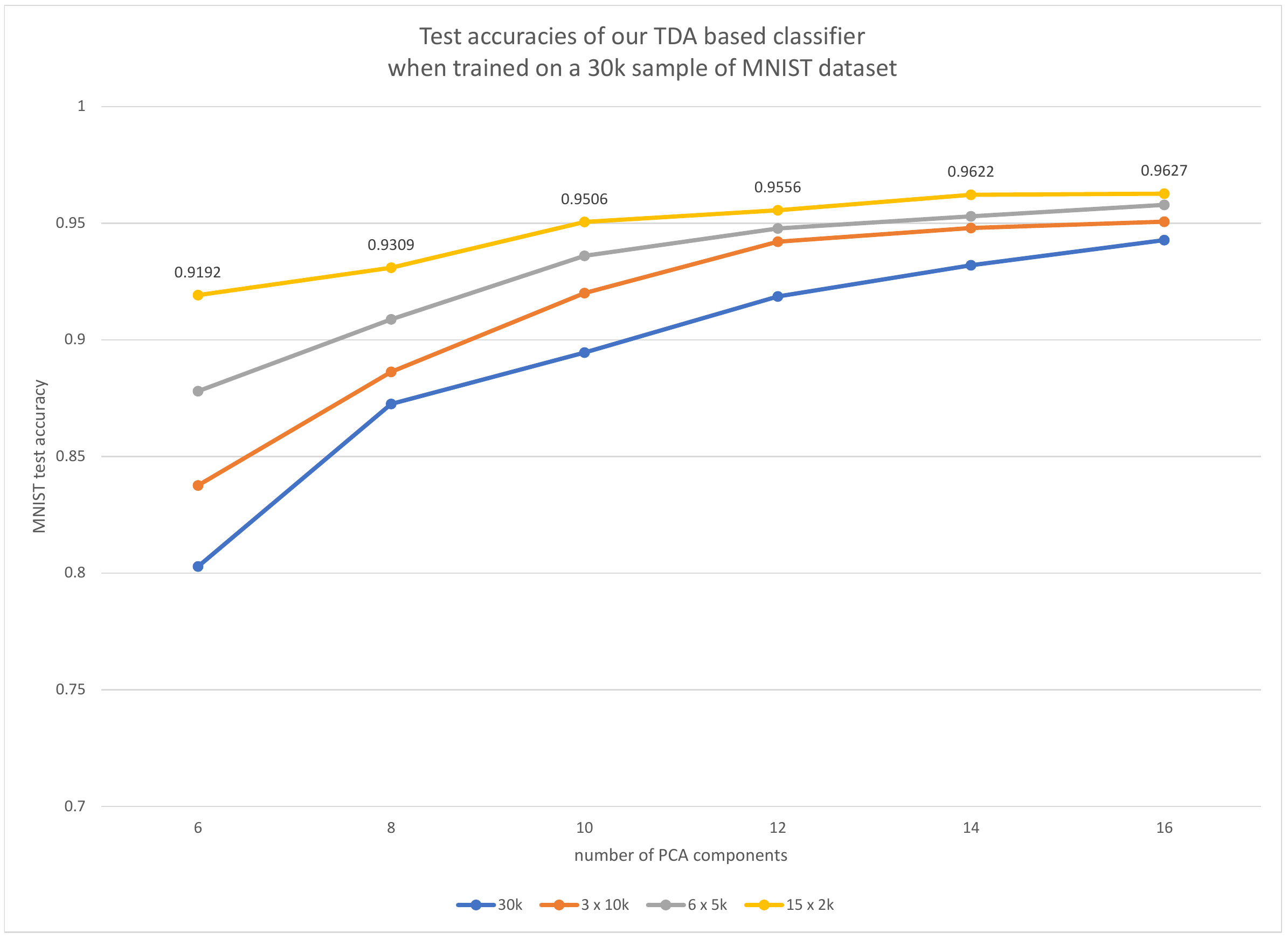}
\end{center}
   \caption{Test accuracy of MC trained using 30k sample from MNIST dataset. We present the accuracy with respect to the number of PCA filters used in the committee of Mapper graphs, and the number of subsets the whole 30k sample was split into.}
\label{fig:split}
\end{figure}

\section{Description of Noise Models}\label{sec:noisemodels}
\label{secnoisemodel}
\subsection{Models}
We implement three different noise models in order to determine overall classifier robustness.  The models we use include: Gaussian blur, Gaussian, and salt \& pepper.  All are consistent with the models in \cite{foolbox2}, and we give a brief explanation below.  For each of these models, we use a parameter to control the extent of the perturbation, which we refer to as $\lambda$.

The \emph{Gaussian blur} model performs a convolution with a $2$-dimensional Gaussian centered at each pixel in the image.  Each pixel is replaced by the Gaussian weighted sum of nearby pixel values.  The $\lambda$ parameter we use for robustness calculations is related to the Gaussian standard deviation by: $\sigma = 28 \lambda$.  The final perturbed image is clipped to the max and min values of the original image.  The $28$ comes from the image size of $28\times 28.$

The \emph{salt \& pepper} or \emph{s\&p} model replaces random pixel values with the minimum (i.e. ``pepper'') or maximum (i.e. ``salt'') in the image.  Setting $q_1$ as the probability of flipping a pixel, and $q_2$ as the ratio of salt to pepper, we use: $q_1 = \lambda$ and $q_2=\frac{1}{2}$.

The \emph{Gaussian} model adds in noise to each pixel which is sampled from a Gaussian distribution.  The distribution we use is centered at zero and has $\sigma=0.1\sqrt{\epsilon}$.  The final perturbed image is clipped.

There are a few subtle differences between the robustness results when $\lambda$ vs the $l^2$-norm is used, which we remark on here.  $\lambda$ is the internal parameter we use to quantify the scale of noise that is added.  While the value of $\lambda$ correlates to the actual $l^2$ distance an image is perturbed, there is not a one-to-one correspondence.  Perhaps a more useful way to think about $\lambda$ is that it sets a range over which $l^2$ perturbations may occur.  As $\lambda$ increases, so does this range.  Since the $l^2$-distance is a more physical measure in this experiment, we report robustness as a function of $l^2$ rather than $\lambda$.

\subsection{Image Data as a Function of Noise Parameter}
\label{secimdata}
\begin{figure}[h]
  \begin{subfigure}[]{0.15\textwidth}
    \includegraphics[width=\textwidth]{gauss_blur_1.pdf}
     \caption{$\lambda=0.01$; $l^2 \approx 0.02$}
   \label{fig:gb_1} 
  \end{subfigure}
\hspace{0.1em}
  \begin{subfigure}[]{0.15\textwidth}
    \includegraphics[width=\textwidth]{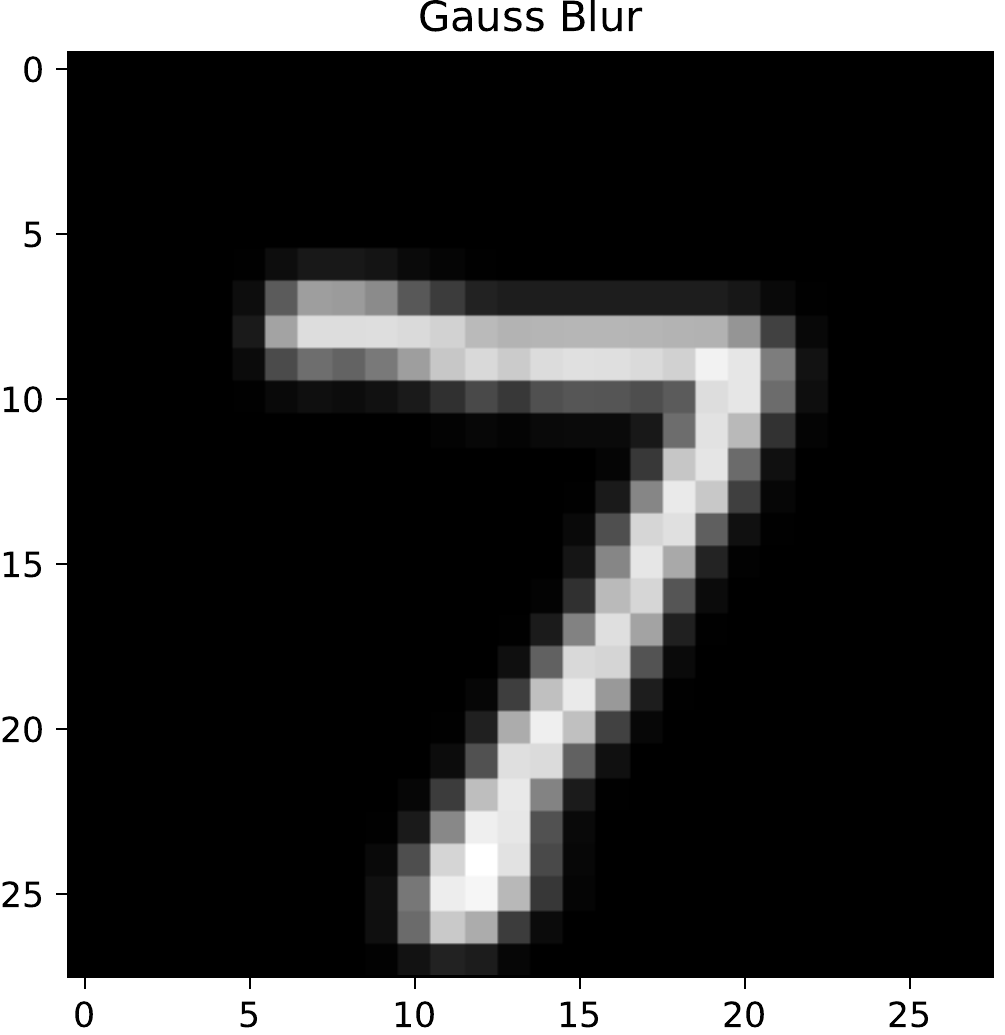}
     \caption{$\lambda=0.02$; $l^2 \approx 1.4$}
   \label{fig:gb_2} 
  \end{subfigure}
\hspace{0.1em}
  \begin{subfigure}[]{0.15\textwidth}
    \includegraphics[width=\textwidth]{gauss_blur_3.pdf}
     \caption{$\lambda=0.03$; $l^2 \approx 2.5$}
   \label{fig:gb_3} 
  \end{subfigure}  
\hspace{0.1em}
  \begin{subfigure}[]{0.15\textwidth}
    \includegraphics[width=\textwidth]{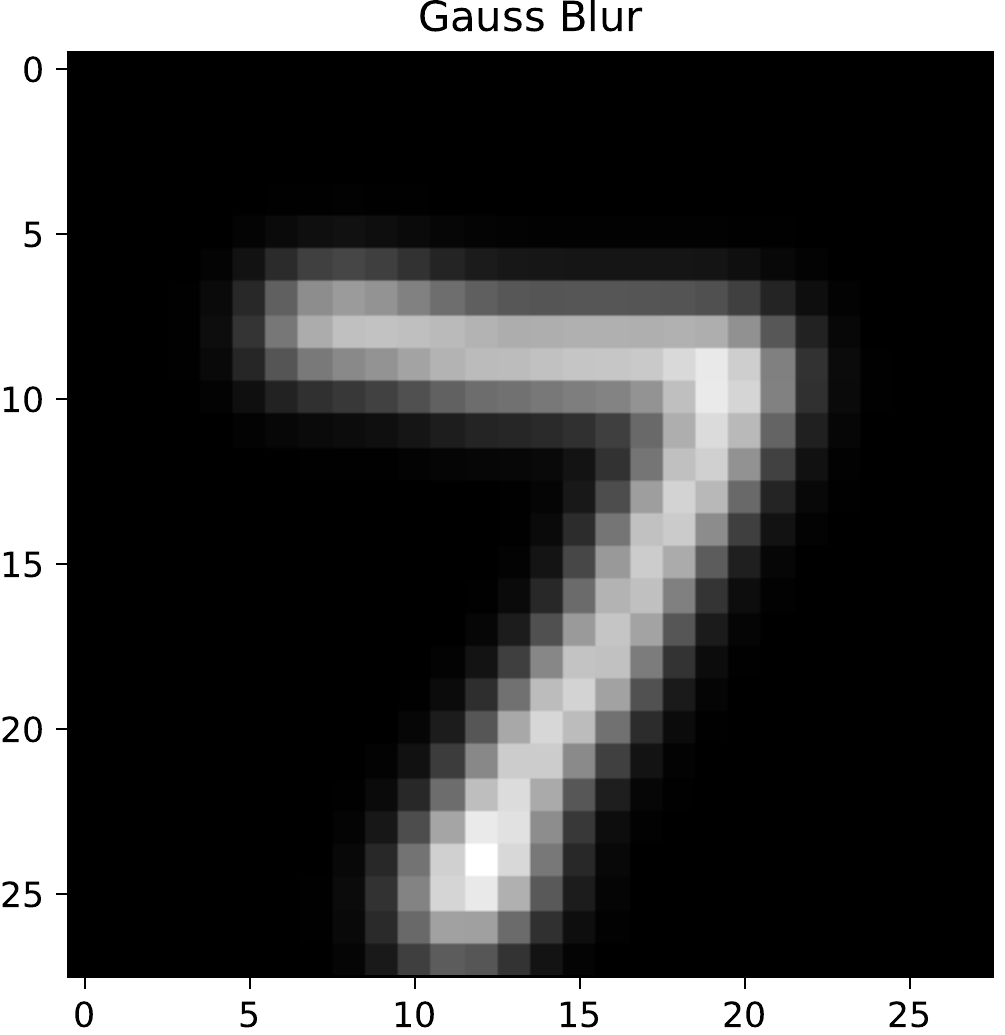}
     \caption{$\lambda=0.04$; $l^2 \approx 3.4$}
   \label{fig:gb_4} 
  \end{subfigure}  
\hspace{0.1em}
  \begin{subfigure}[]{0.15\textwidth}
    \includegraphics[width=\textwidth]{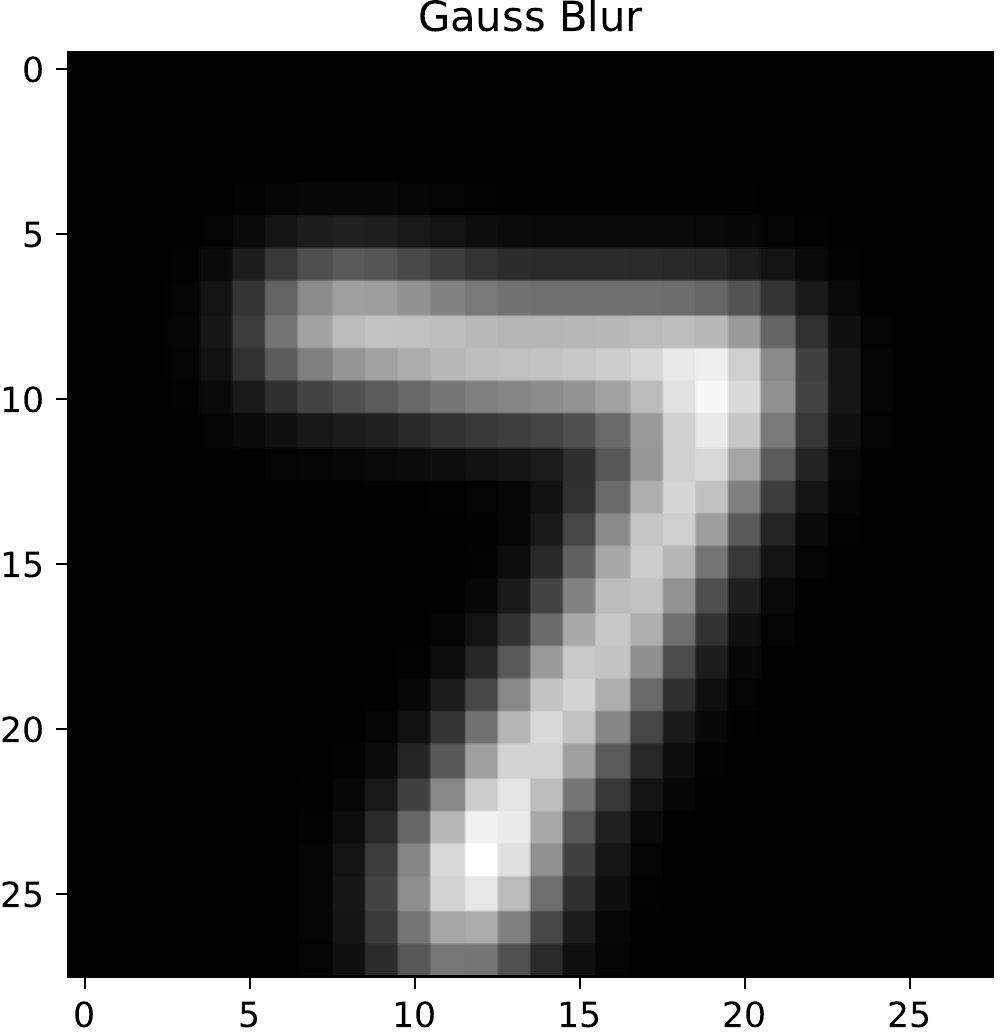}
     \caption{$\lambda=0.05$; $l^2 \approx 4.1$}
   \label{fig:gb_5} 
  \end{subfigure}
\hspace{0.1em}
  \begin{subfigure}[]{0.15\textwidth}
    \includegraphics[width=\textwidth]{gauss_blur_6.pdf}
     \caption{$\lambda=0.06$; $l^2 \approx 4.6$}
   \label{fig:gb_6} 
  \end{subfigure}  
\caption{The number $7$ as a function of $\lambda$ for the Gauss blur noise model.  $\lambda$ can be thought of as a percentage of $28$, a fundamental length scale in this data.}
\label{fig:gb}
\end{figure}

\begin{figure}[h]
  \begin{subfigure}[]{0.15\textwidth}
    \includegraphics[width=\textwidth]{gaussian_1.pdf}
     \caption{$\lambda=0.01$; $l^2 \approx 2.2$}
   \label{fig:gaussian_1} 
  \end{subfigure}
\hspace{0.1em}
  \begin{subfigure}[]{0.15\textwidth}
    \includegraphics[width=\textwidth]{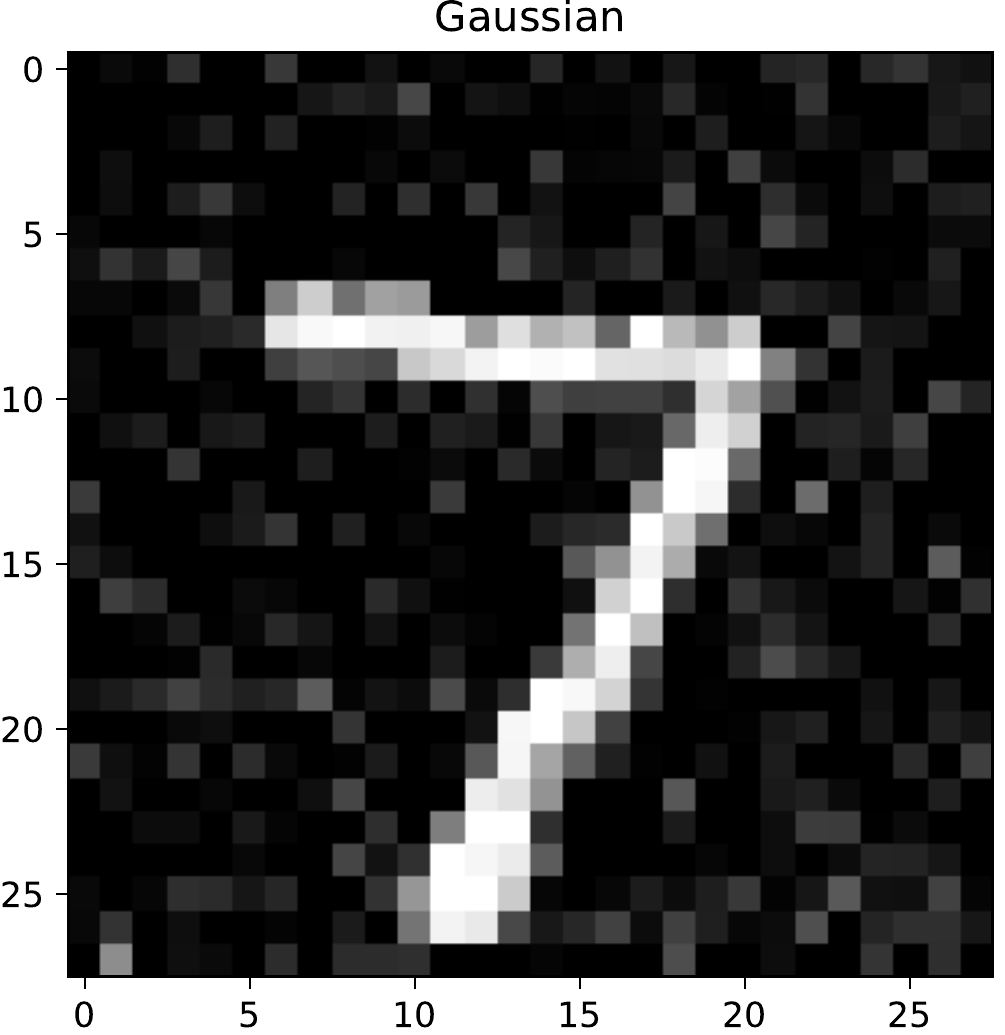}
     \caption{$\lambda=0.02$; $l^2 \approx 2.9$}
   \label{fig:gaussian_2} 
  \end{subfigure}
\hspace{0.1em}
  \begin{subfigure}[]{0.15\textwidth}
    \includegraphics[width=\textwidth]{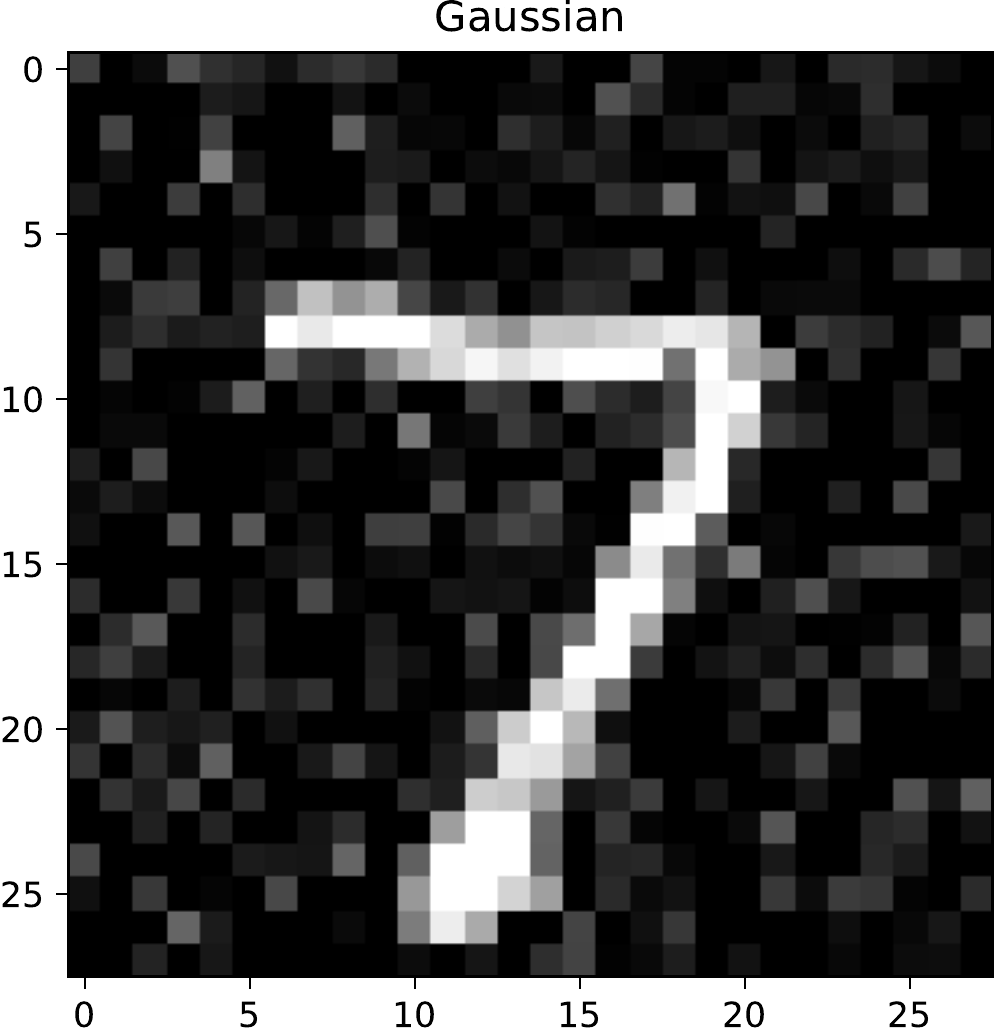}
     \caption{$\lambda=0.03$; $l^2 \approx 3.4$}
   \label{fig:gaussian_3} 
  \end{subfigure}  
\hspace{0.1em}
  \begin{subfigure}[]{0.15\textwidth}
    \includegraphics[width=\textwidth]{gaussian_4.pdf}
     \caption{$\lambda=0.04$; $l^2 \approx 3.8$}
   \label{fig:gaussian_4} 
  \end{subfigure}  
\hspace{0.1em}
  \begin{subfigure}[]{0.15\textwidth}
    \includegraphics[width=\textwidth]{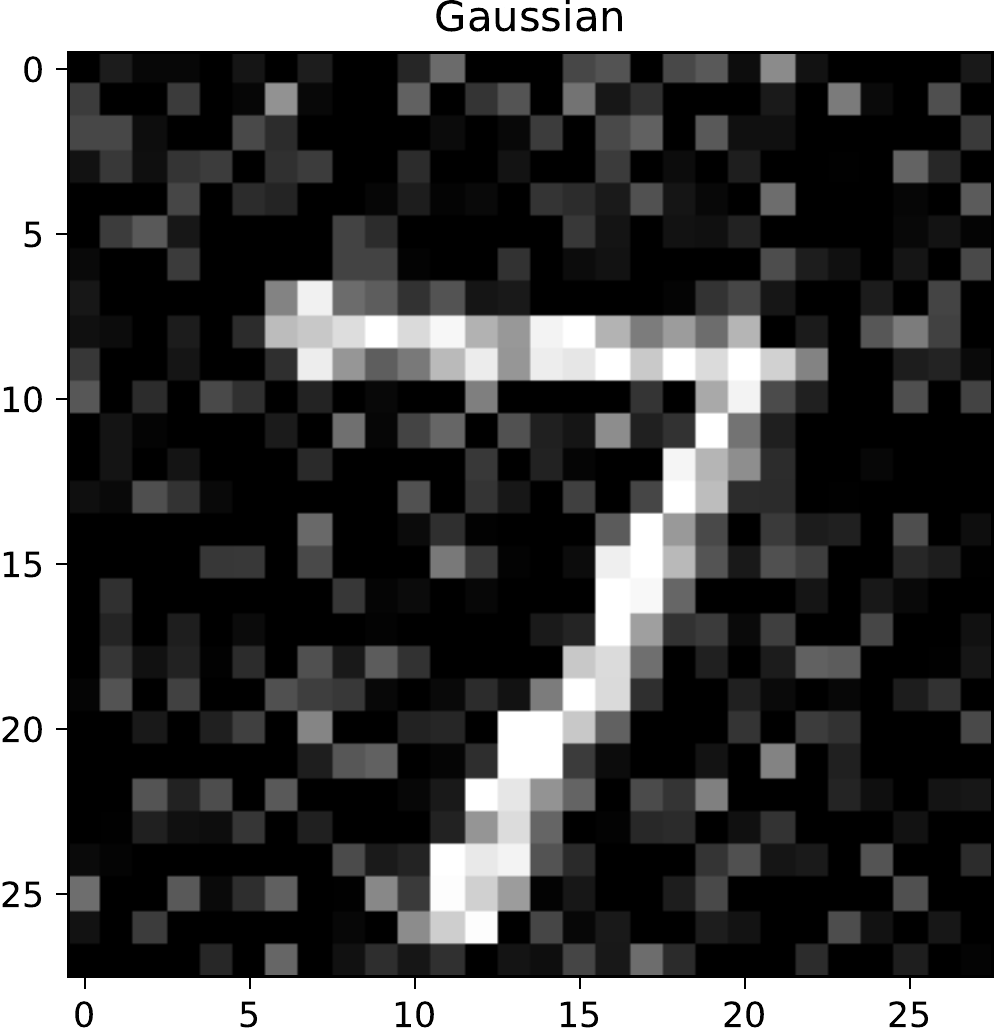}
     \caption{$\lambda=0.05$; $l^2 \approx 4.4$}
   \label{fig:gaussian_5} 
  \end{subfigure}
\hspace{0.1em}
  \begin{subfigure}[]{0.15\textwidth}
    \includegraphics[width=\textwidth]{gaussian_6.pdf}
     \caption{$\lambda=0.06$; $l^2 \approx 4.6$}
   \label{fig:gaussian_6}
  \end{subfigure}  
\caption{The number $7$ as a function of $\lambda$ for the Gaussian noise model.}
\label{fig:gaussian}
\end{figure}

\begin{figure}[h]
  \begin{subfigure}[]{0.15\textwidth}
    \includegraphics[width=\textwidth]{sp_10.pdf}
     \caption{$\lambda=0.01$; $l^2 \approx 2.4$}
   \label{fig:sp_10} 
  \end{subfigure}
\hspace{0.1em}
  \begin{subfigure}[]{0.15\textwidth}
    \includegraphics[width=\textwidth]{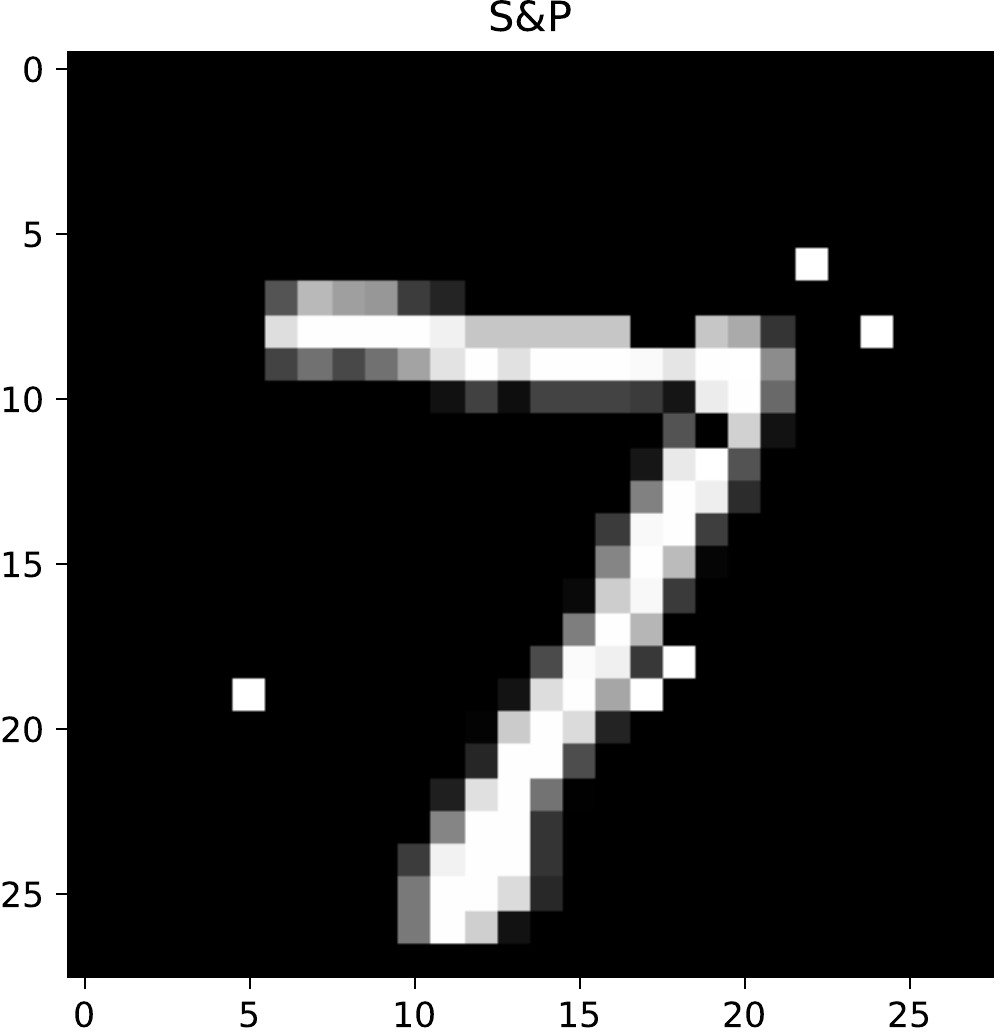}
     \caption{$\lambda=0.02$; $l^2 \approx 2.7$}
   \label{fig:sp_20} 
  \end{subfigure}
\hspace{0.1em}
  \begin{subfigure}[]{0.15\textwidth}
    \includegraphics[width=\textwidth]{sp_30.pdf}
     \caption{$\lambda=0.03$; $l^2 \approx 3.6$}
   \label{fig:sp_30} 
  \end{subfigure}  
\hspace{0.1em}
  \begin{subfigure}[]{0.15\textwidth}
    \includegraphics[width=\textwidth]{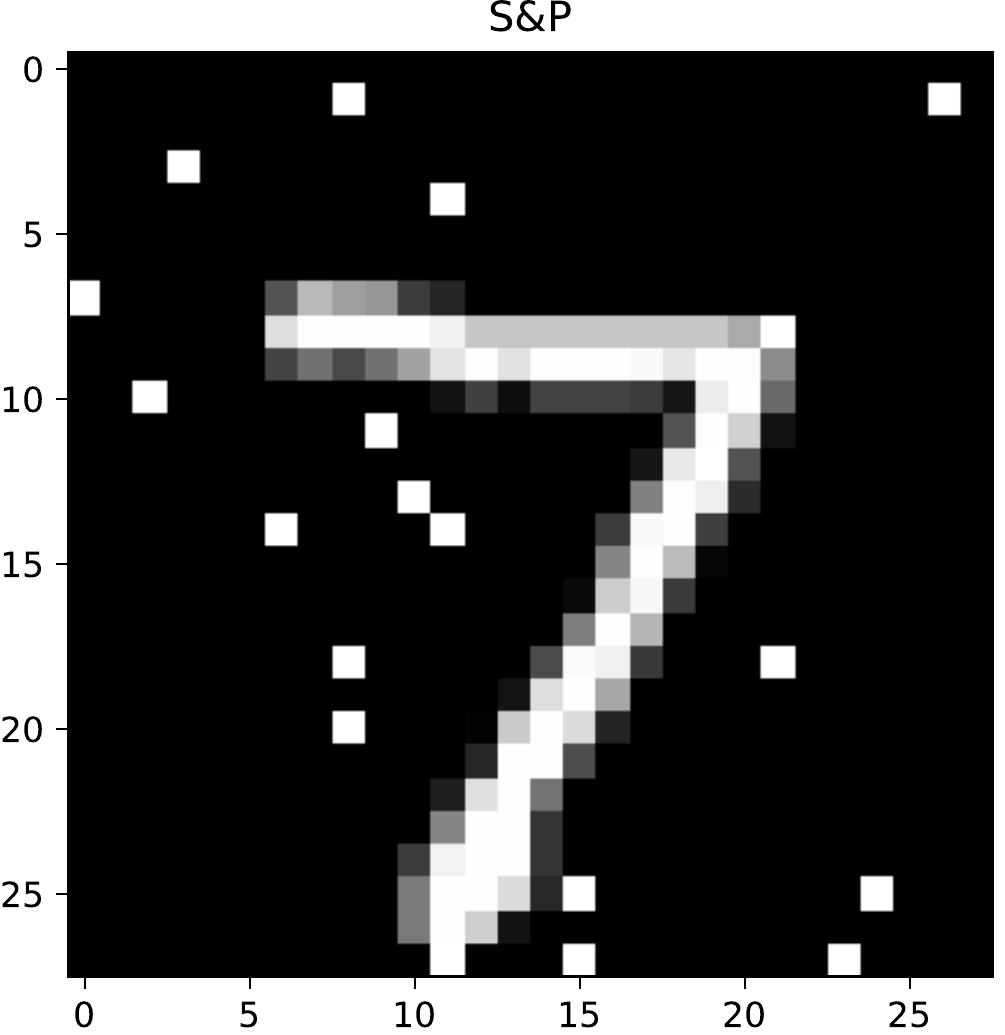}
     \caption{$\lambda=0.04$; $l^2 \approx 4.3$}
   \label{fig:sp_40} 
  \end{subfigure}  
\hspace{0.1em}
  \begin{subfigure}[]{0.15\textwidth}
    \includegraphics[width=\textwidth]{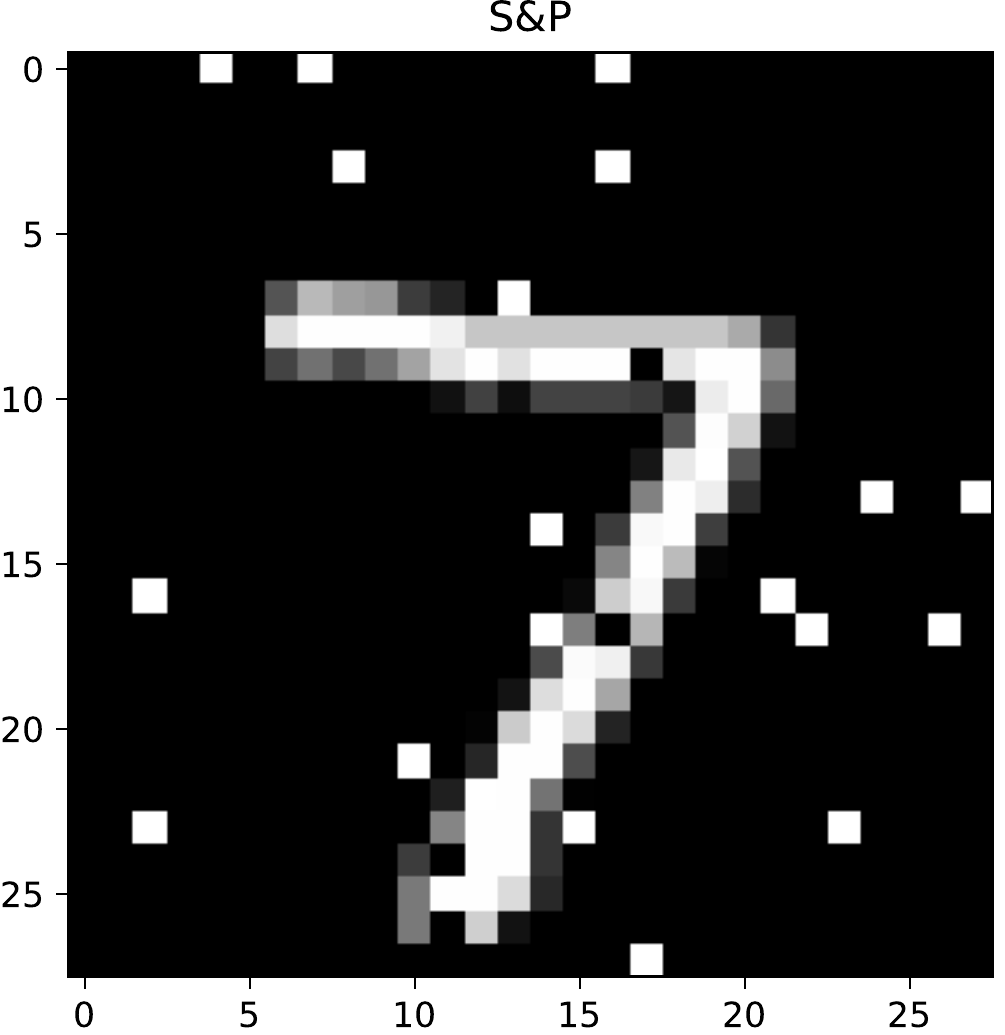}
     \caption{$\lambda=0.05$; $l^2 \approx 4.8$}
   \label{fig:sp_50} 
  \end{subfigure}
\hspace{0.1em}
  \begin{subfigure}[]{0.15\textwidth}
    \includegraphics[width=\textwidth]{sp_60.pdf}
     \caption{$\lambda=0.06$; $l^2 \approx 5.6$}
   \label{fig:sp_60} 
  \end{subfigure}  
\caption{The number $7$ as a function of $\lambda$ for the s\&p noise model.  $\lambda=0.01$, for instance, corresponds to a $1\%$ chance of flipping a pixel.}
\label{fig:sp}
\end{figure}

\begin{figure}[h]
  \begin{subfigure}[]{0.15\textwidth}
    \includegraphics[width=\textwidth]{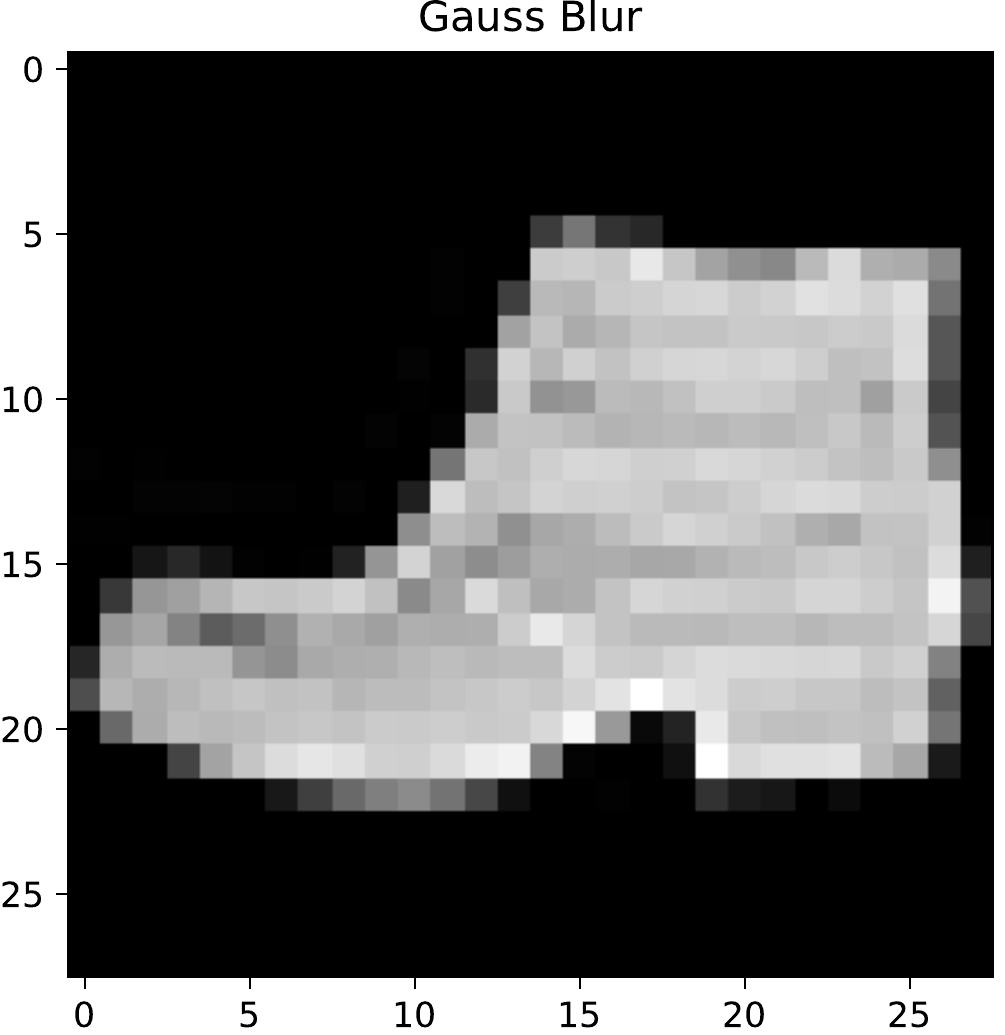}
     \caption{$\lambda=0.01$; $l^2 \approx 0.02$}
   \label{fig:gb_1f} 
  \end{subfigure}
\hspace{0.1em}
  \begin{subfigure}[]{0.15\textwidth}
    \includegraphics[width=\textwidth]{gauss_blur_2f.pdf}
     \caption{$\lambda=0.02$; $l^2 \approx 1.4$}
   \label{fig:gb_2f} 
  \end{subfigure}
\hspace{0.1em}
  \begin{subfigure}[]{0.15\textwidth}
    \includegraphics[width=\textwidth]{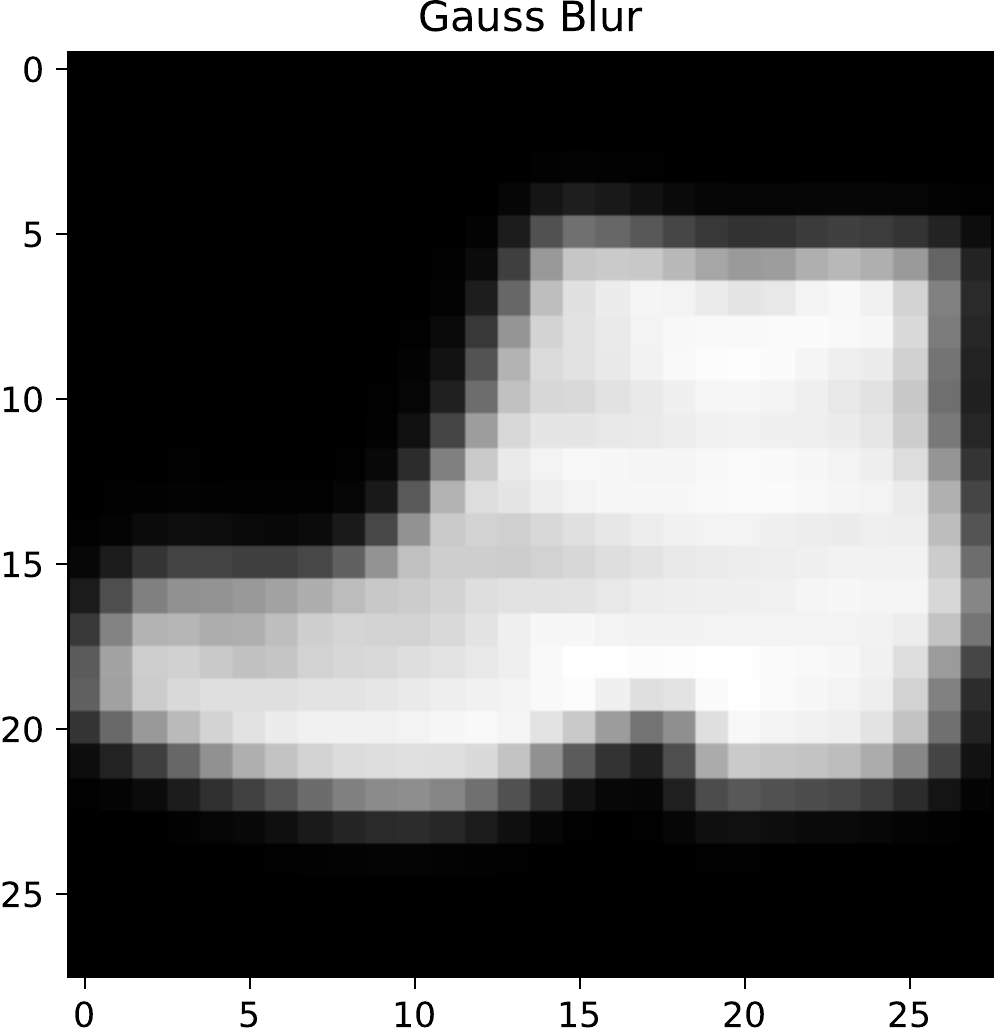}
     \caption{$\lambda=0.03$; $l^2 \approx 2.4$}
   \label{fig:gb_3f} 
  \end{subfigure}  
\hspace{0.1em}
  \begin{subfigure}[]{0.15\textwidth}
    \includegraphics[width=\textwidth]{gauss_blur_4f.pdf}
     \caption{$\lambda=0.04$; $l^2 \approx 3.1$}
   \label{fig:gb_4f} 
  \end{subfigure}  
\hspace{0.1em}
  \begin{subfigure}[]{0.15\textwidth}
    \includegraphics[width=\textwidth]{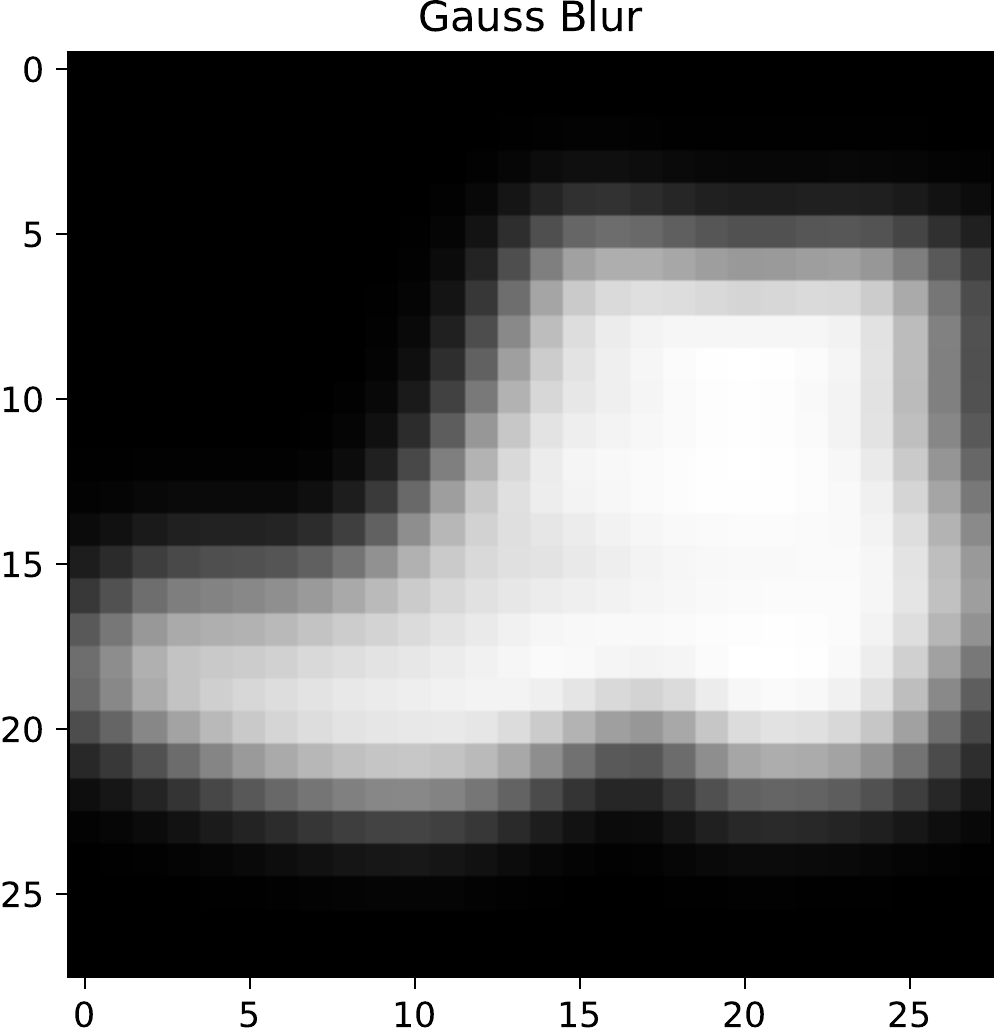}
     \caption{$\lambda=0.05$; $l^2 \approx 3.6$}
   \label{fig:gb_5f} 
  \end{subfigure}
\hspace{0.1em}
  \begin{subfigure}[]{0.15\textwidth}
    \includegraphics[width=\textwidth]{gauss_blur_6f.pdf}
     \caption{$\lambda=0.06$; $l^2 \approx 4.0$}
   \label{fig:gb_6f} 
  \end{subfigure}  
\caption{An image of a boot as a function of $\lambda$ for the Gauss blur noise model.  $\lambda$ can be thought of as a percentage of $28$, a fundamental length scale in this data.}
\label{fig:gbf}
\end{figure}

\begin{figure}[h]
  \begin{subfigure}[]{0.15\textwidth}
    \includegraphics[width=\textwidth]{gaussian_1f.pdf}
     \caption{$\lambda=0.01$; $l^2 \approx 2.4$}
   \label{fig:gaussian_1f} 
  \end{subfigure}
\hspace{0.1em}
  \begin{subfigure}[]{0.15\textwidth}
    \includegraphics[width=\textwidth]{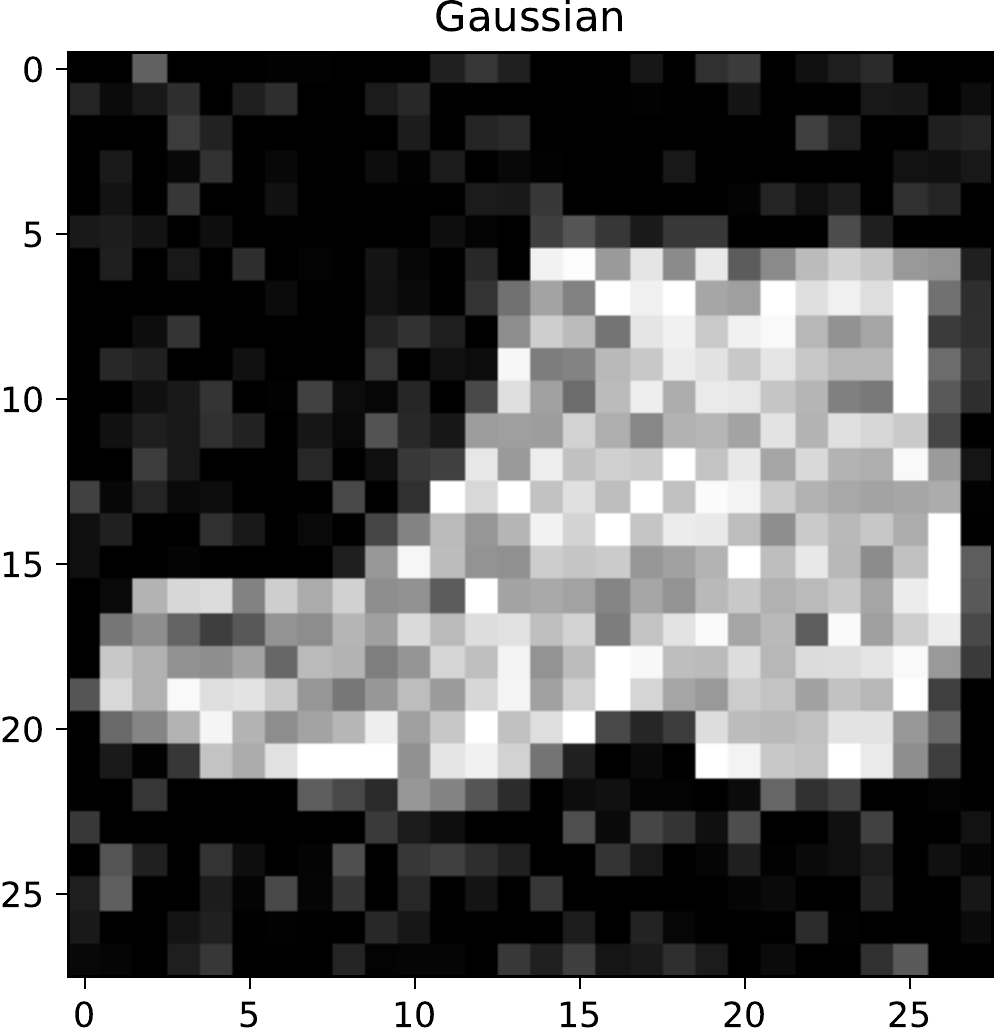}
     \caption{$\lambda=0.02$; $l^2 \approx 3.3$}
   \label{fig:gaussian_2f} 
  \end{subfigure}
\hspace{0.1em}
  \begin{subfigure}[]{0.15\textwidth}
    \includegraphics[width=\textwidth]{gaussian_3f.pdf}
     \caption{$\lambda=0.03$; $l^2 \approx 3.9$}
   \label{fig:gaussian_3f} 
  \end{subfigure}  
\hspace{0.1em}
  \begin{subfigure}[]{0.15\textwidth}
    \includegraphics[width=\textwidth]{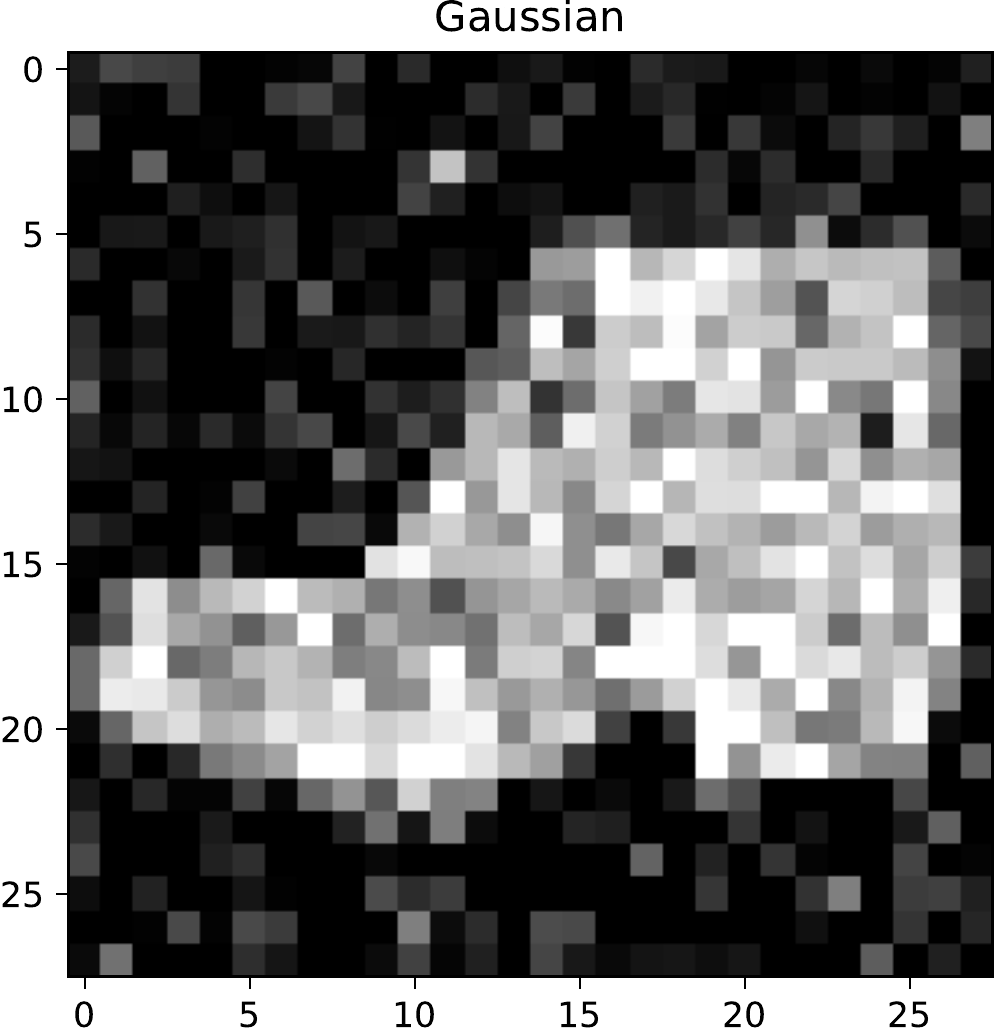}
     \caption{$\lambda=0.04$; $l^2 \approx 4.4$}
   \label{fig:gaussian_4f} 
  \end{subfigure}  
\hspace{0.1em}
  \begin{subfigure}[]{0.15\textwidth}
    \includegraphics[width=\textwidth]{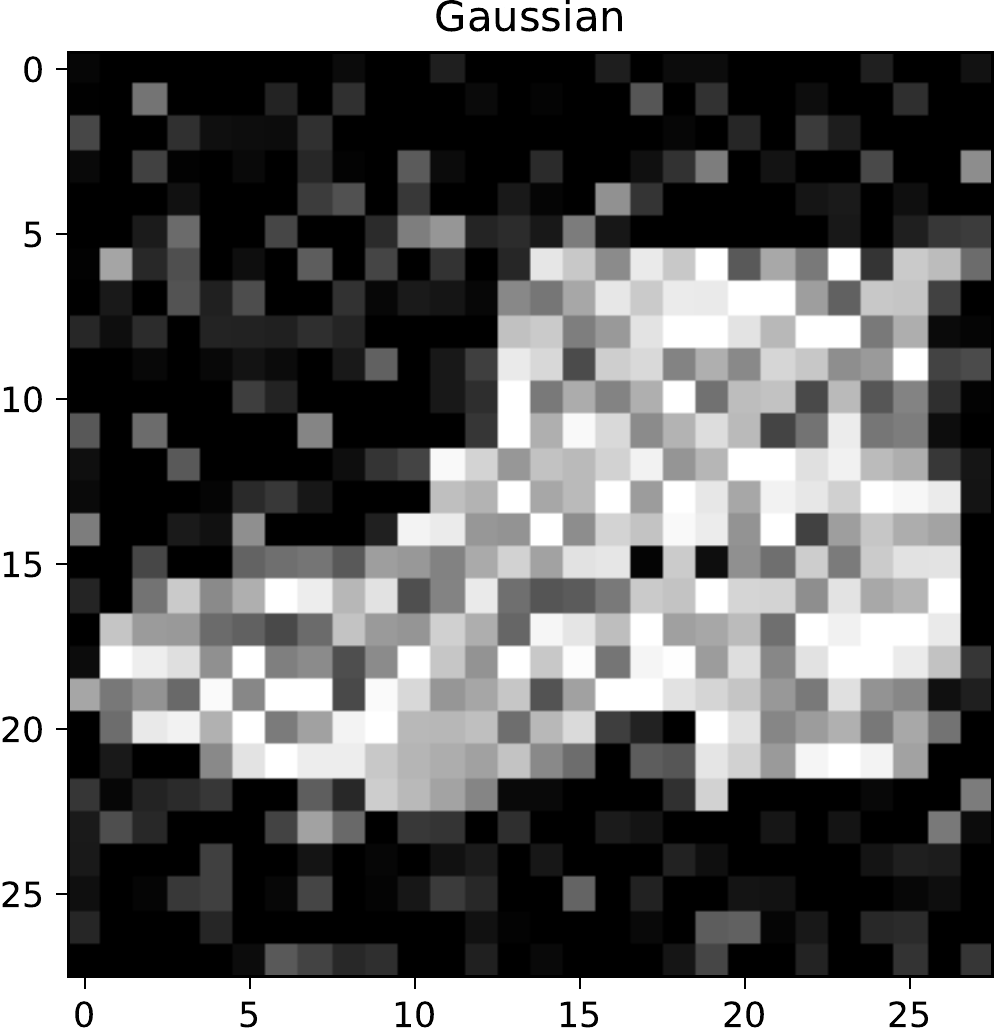}
     \caption{$\lambda=0.05$; $l^2 \approx 4.8$}
   \label{fig:gaussian_5f} 
  \end{subfigure}
\hspace{0.1em}
  \begin{subfigure}[]{0.15\textwidth}
    \includegraphics[width=\textwidth]{gaussian_6f.pdf}
     \caption{$\lambda=0.06$; $l^2 \approx 5.5$}
   \label{fig:gaussian_6f}
  \end{subfigure}  
\caption{ An image of a boot as a function of $\lambda$ for the Gaussian noise model.}
\label{fig:gaussianf}
\end{figure}

\begin{figure}[h]
  \begin{subfigure}[]{0.15\textwidth}
    \includegraphics[width=\textwidth]{sp_10f.pdf}
     \caption{$\lambda=0.01$; $l^2 \approx 2.0$}
   \label{fig:sp_10f} 
  \end{subfigure}
\hspace{0.1em}
  \begin{subfigure}[]{0.15\textwidth}
    \includegraphics[width=\textwidth]{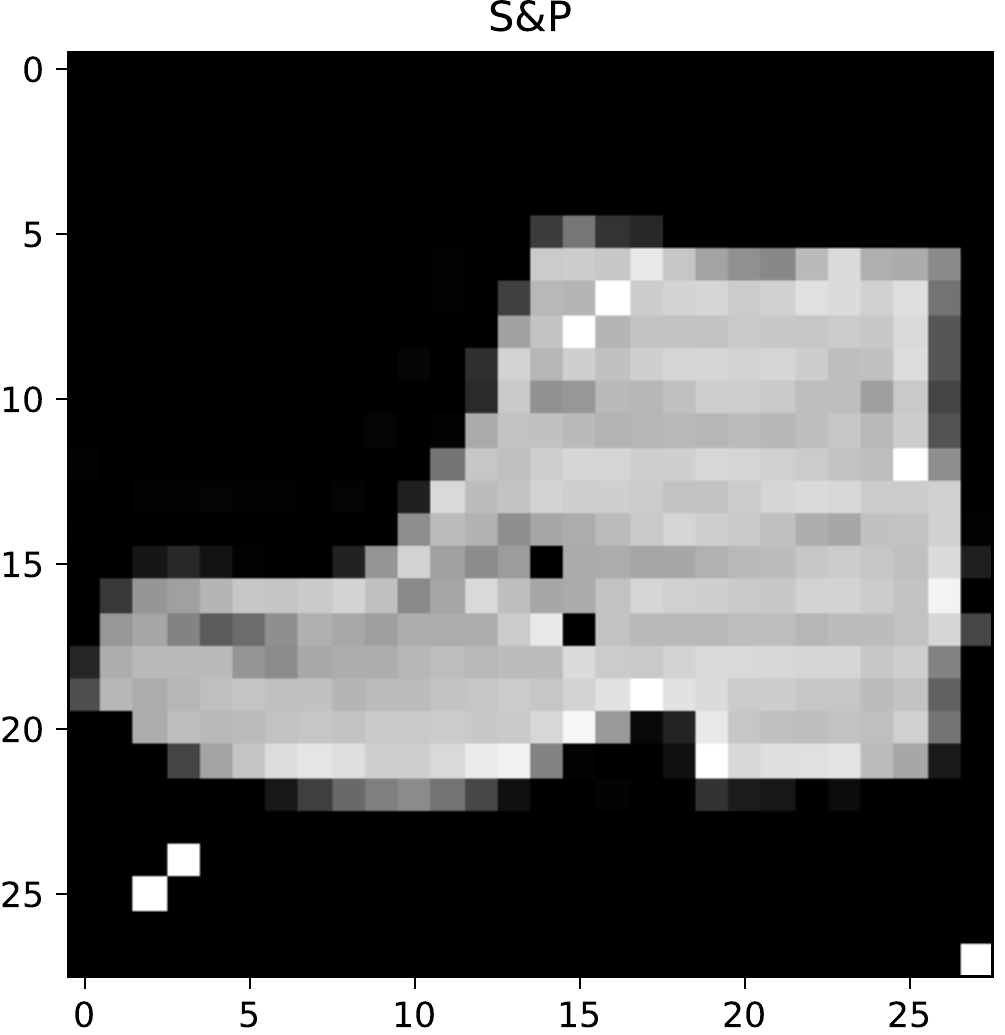}
     \caption{$\lambda=0.02$; $l^2 \approx 2.2$}
   \label{fig:sp_20f} 
  \end{subfigure}
\hspace{0.1em}
  \begin{subfigure}[]{0.15\textwidth}
    \includegraphics[width=\textwidth]{sp_30f.pdf}
     \caption{$\lambda=0.03$; $l^2 \approx 2.8$}
   \label{fig:sp_30f} 
  \end{subfigure}  
\hspace{0.1em}
  \begin{subfigure}[]{0.15\textwidth}
    \includegraphics[width=\textwidth]{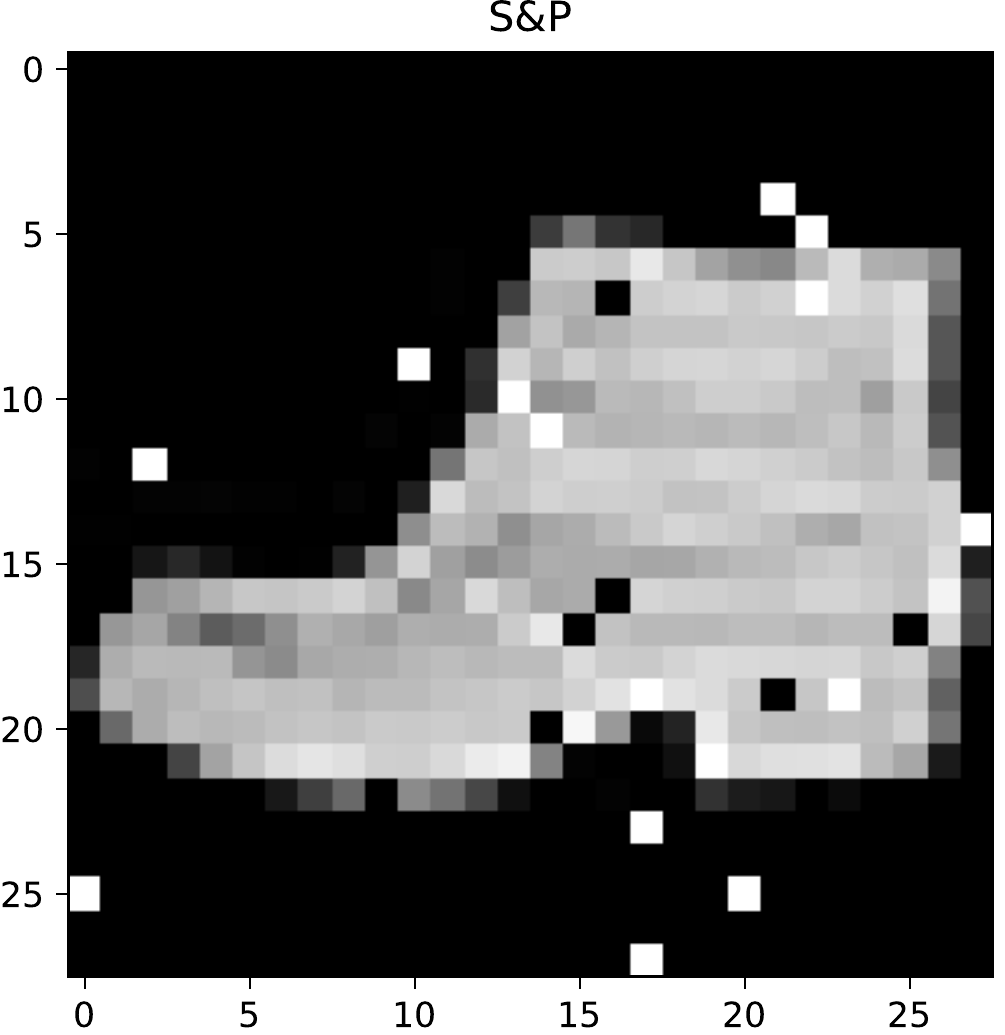}
     \caption{$\lambda=0.04$; $l^2 \approx 3.6$}
   \label{fig:sp_40f} 
  \end{subfigure}  
\hspace{0.1em}
  \begin{subfigure}[]{0.15\textwidth}
    \includegraphics[width=\textwidth]{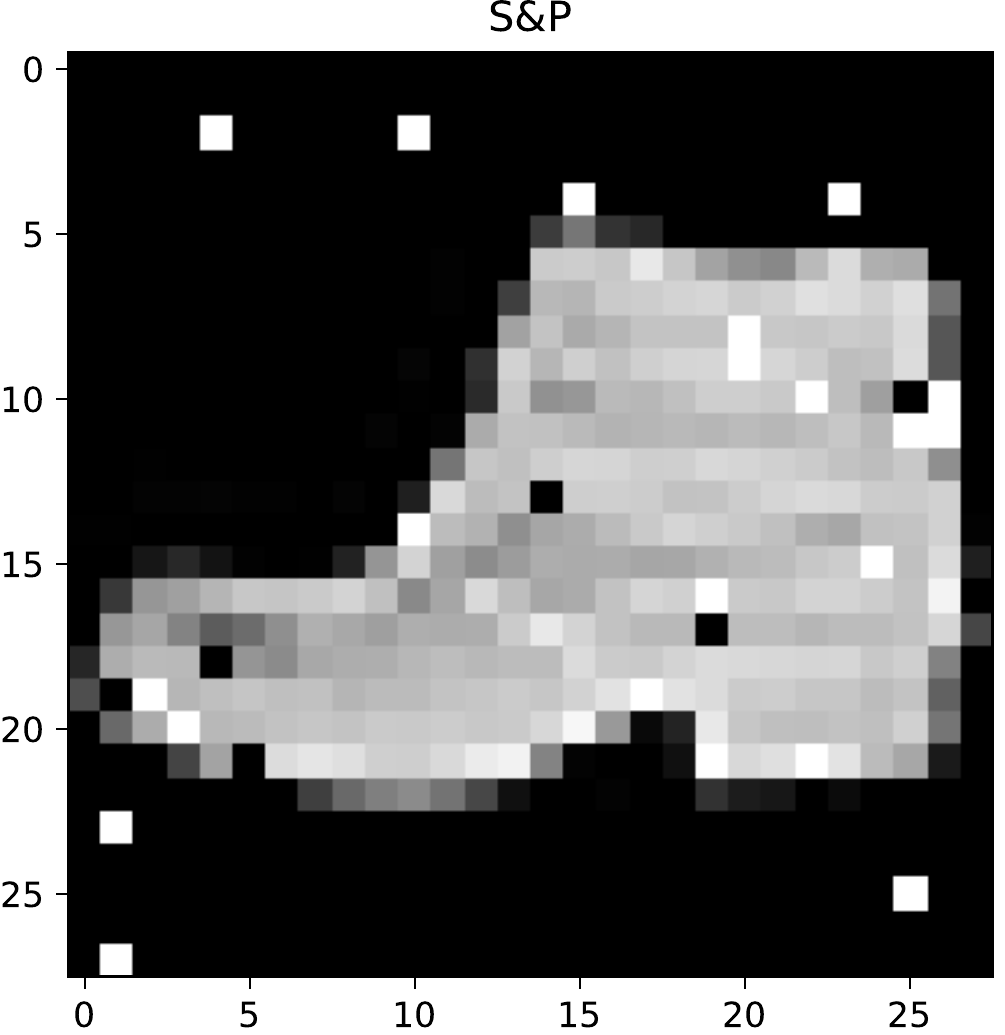}
     \caption{$\lambda=0.05$; $l^2 \approx 3.5$}
   \label{fig:sp_50f} 
  \end{subfigure}
\hspace{0.1em}
  \begin{subfigure}[]{0.15\textwidth}
    \includegraphics[width=\textwidth]{sp_60f.pdf}
     \caption{$\lambda=0.06$; $l^2 \approx 4.2$}
   \label{fig:sp_60f} 
  \end{subfigure}  
\caption{ An image of a boot as a function of $\lambda$ for the s\&p noise model.  $\lambda=0.01$, for instance, corresponds to a $1\%$ chance of flipping a pixel.}
\label{fig:spf}
\end{figure}

\section{Mapper Committee Dimensions}
In this section, we summarize the overall mapper dimensions (i.e. total number of nodes in a committee) that data points are sent to.  In our analysis, the mapper committee dimension depends on the latent space we project to.  In general, this number is highly dependent on $n_{bins}$ and $n_{int}$, but we keep these fixed to $n_{bins}=n_{int}=10$.

\begin{table}[!h]
    \caption{The total number of nodes in the mapper committee, with respect to choice in latent space projections.  For $60$k, we choose to use just PCA and VAE since they are the highest performing.  Additionally, for $60$k MNIST, we use the splitting procedure presented in Sec. \ref{secmapping} which will increase the dimensionality of the Mapper committee.}
\begin{center}
 \begin{tabular}{| c | c | c |} 
    \hline
\textbf{\bf Latent Space} & {\bf $10$k} & {\bf $60$k}  \\ [0.7ex] 
 \hline\hline
    PCA & 215 & 1326 \\ 
    CAE & 201 & NA \\
    DAE & 201  & NA \\
    VAE & 207 & 1337 \\
  \hline
  \end{tabular}
\end{center}
\label{tab:mapperdimensions}
\end{table}

\section{Architecture of the classifier used in MC method}
\label{secclassifier}
We used the following neural network architecture as the classifier on top of the Mapper method (see Fig.~\ref{fig:tdann}):
\begin{enumerate}
    \item ReLU(committee $C$ dim $\sum_{j=1}^{N_C}{N_{M_j}}$, $4000$),
    \item Dropout with $p=0.25$,
    \item ReLU($4000$, $2000$),
    \item Dropout with $p=0.25$,
    \item ReLU($2000$, $10$),
    \item LogSoftmax normalization,
\end{enumerate}
\begin{itemize}
    \item The negative log likelihood loss,
    \item batch size $100$,
    \item SGD optimizer with learning rate=$0.01$, momentum=$0.9$. 
\end{itemize}

\section{Additional Figures}\label{sec:figures}
\begin{figure}[h!]
	\centering
	\begin{subfigure}{0.4\textwidth} 
	    \label{fig:10bin}
		\includegraphics[width=\textwidth]{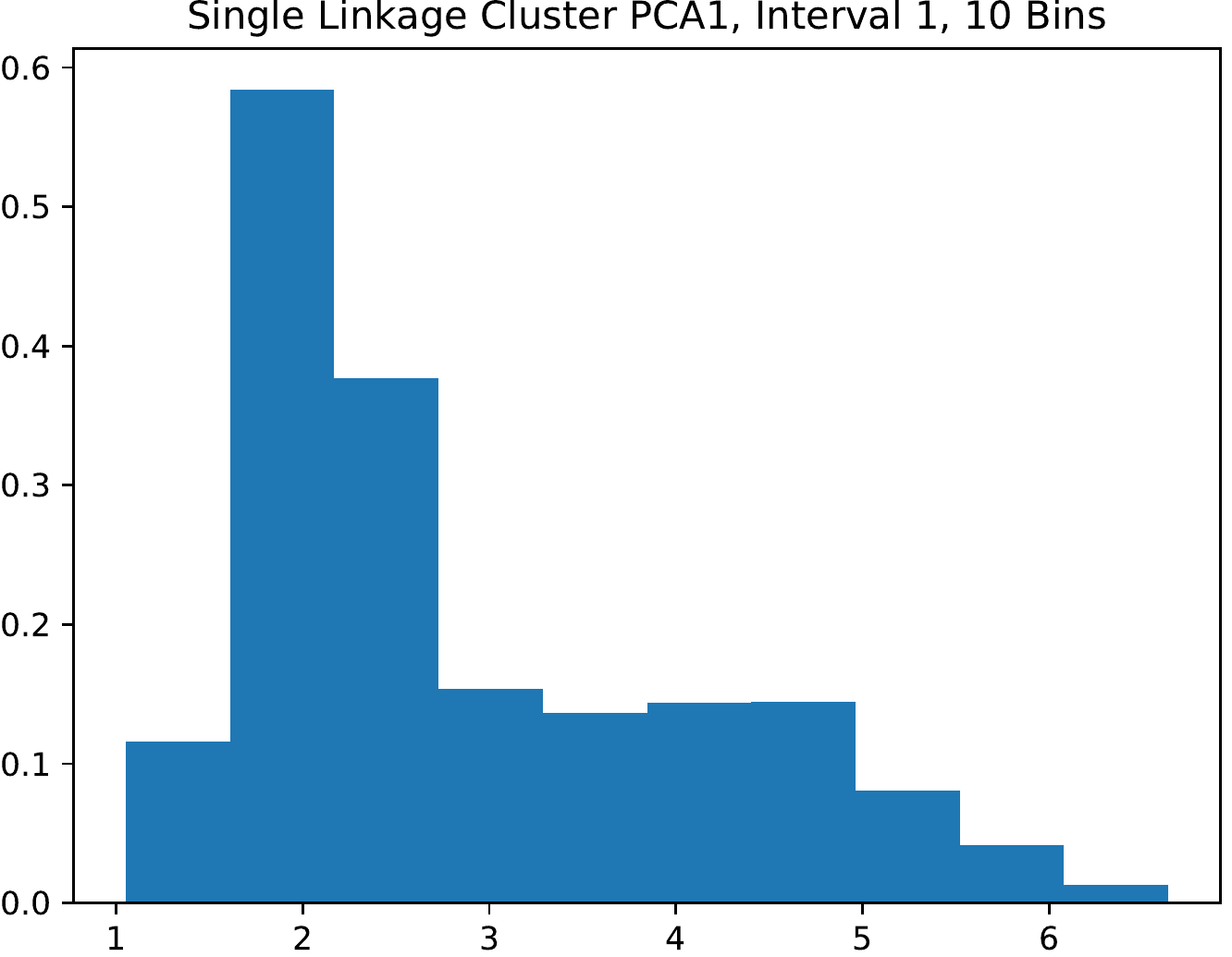}
		\caption{Produced with $n_{\rm int} = 10, n_{\rm bins}=10$.} 
	\end{subfigure}
	\vspace{1em} 
	\begin{subfigure}{0.4\textwidth}
	    \label{fig:20bin}
		\includegraphics[width=\textwidth]{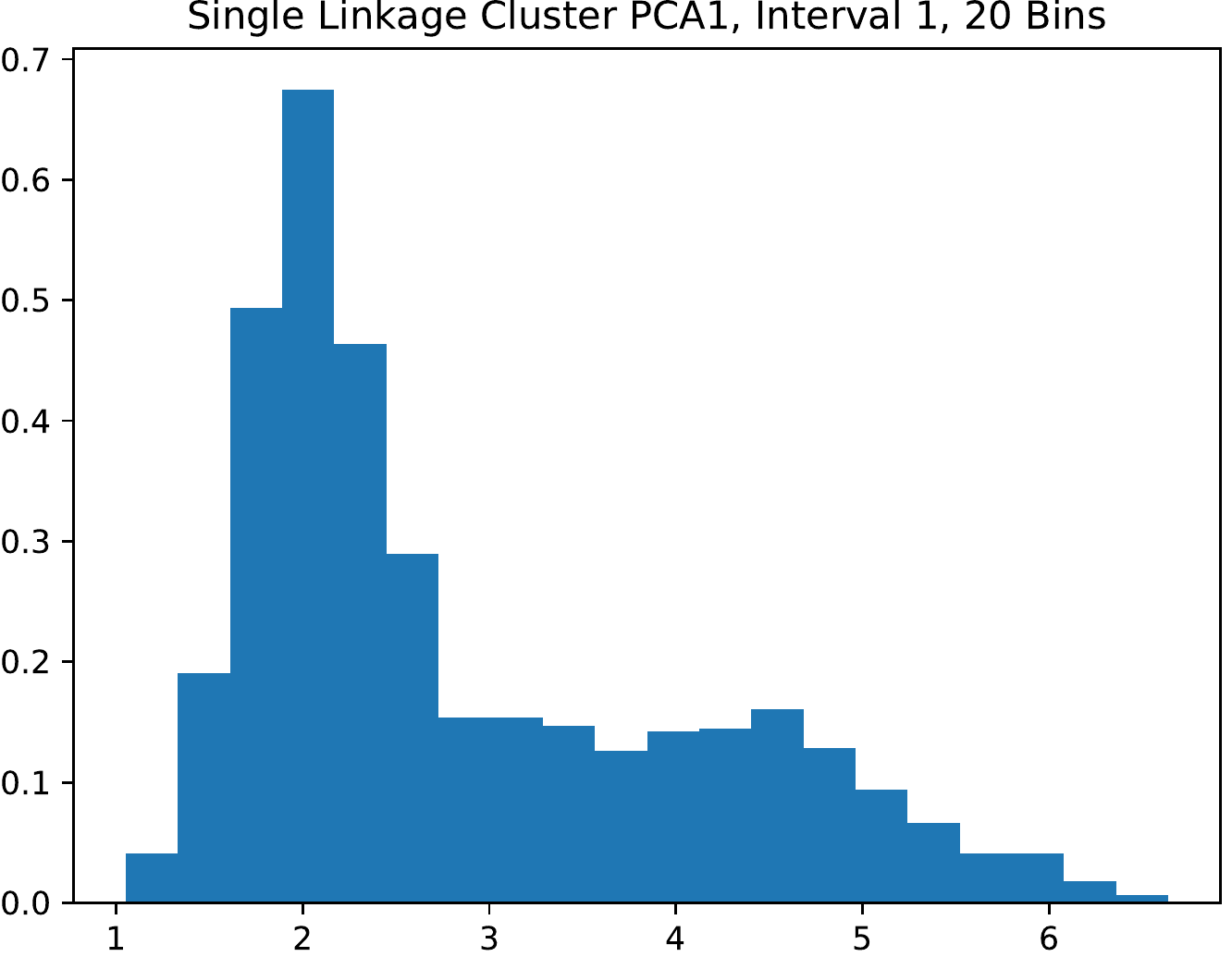}
		\caption{Produced with $n_{\rm int} = 10, n_{\rm bins}=20$.} 
	\end{subfigure}
	\caption{Histograms of length scales at which clusters are grouped via single-linkage clustering.  These plots represent varied $n_{bins}$ for the first interval in the open cover of $U_\alpha$, using the filter $f=PCA_1$.  In this case, increasing $n_{bins}$ from $10$ to $20$ has no effect on the cutoff value, which is set to $6.6$.} 
\label{fig:hists}
\end{figure}

\end{document}